\documentclass{article}

\usepackage{microtype}
\usepackage{graphicx}
\usepackage{subfigure}
\usepackage{scalefnt}
\usepackage{diagbox}
\usepackage{wrapfig}
\usepackage{booktabs} 
\usepackage{tabularx}
\usepackage[hyphens]{url}
\usepackage{hyperref}
\usepackage[linesnumbered,ruled,vlined]{algorithm2e}
\usepackage{tikz}
\usepackage{pgfplots}
\pgfplotsset{compat=1.18}

\usepackage[accepted]{icml2024}

\usepackage{amsmath}
\usepackage{amssymb}
\usepackage{mathtools}
\usepackage{amsthm}
\usepackage{xspace,lipsum}

\usepackage[capitalize,noabbrev]{cleveref}
\usepackage{thmtools}
\theoremstyle{plain}
\newtheorem{theorem}{Theorem}[section]

\newtheorem{property}[theorem]{Property}
\newtheorem{lemma}[theorem]{Lemma}

\theoremstyle{definition}
\newtheorem{definition}[theorem]{Definition}
\newtheorem{assumption}[theorem]{Assumption}
\theoremstyle{remark}

\usepackage[textsize=tiny]{todonotes}

\definecolor{mydarkred}{rgb}{0,0,0}

\newcommand{\userind}{i}
\newcommand{\FL}{\texttt{FL}\xspace}

\newcommand{\DP}{\texttt{DP}\xspace}

\newcommand{\PDP}{\texttt{PDP}\xspace}
\newcommand{\DPFL}{\texttt{DPFL}\xspace}

\newcommand{\algname}[1]{{\sf\color{mydarkred}\scalefont{0.90}{#1}}\xspace}

\usepackage{color}
\usepackage{enumerate}
\usepackage{graphicx}
\usepackage{amsfonts, amsmath, bm, amssymb}
\usepackage{dsfont}

\usepackage{xspace}
\makeatletter
\DeclareRobustCommand\onedot{\futurelet\@let@token\@onedot}
\def\@onedot{\ifx\@let@token.\else.\null\fi\xspace}

\makeatother

\newcommand{\floor}[1]{\left\lfloor #1 \right\rfloor}

\newcommand{\Sc}{\mathcal{S}}

\newcommand{\Eb}{\mathbb{E}}

\newcommand{\Rb}{\mathbb{R}}

\newcommand{\wv}{\mathbf{w}}

\newcommand{\thetav     }{\boldsymbol \theta     }

\newcommand{\lambdav    }{\boldsymbol \lambda    }

\icmltitlerunning{Noise-Aware Algorithm for Heterogeneous Differentially Private
Federated Learning}

\begin{document}

\twocolumn[
\icmltitle{Noise-Aware Algorithm for \\ Heterogeneous Differentially Private
Federated Learning}



\icmlsetsymbol{equal}{*}

\begin{icmlauthorlist}
\icmlauthor{Saber Malekmohammadi}{yyy,zzz}
\icmlauthor{Yaoliang Yu}{yyy,zzz}
\icmlauthor{Yang Cao}{comp}
\end{icmlauthorlist}

\icmlaffiliation{yyy}{School of Computer Science, University of Waterloo, Waterloo, Canada}
\icmlaffiliation{zzz}{Vector Institute, Toronto, Canada}
\icmlaffiliation{comp}{Department of Computer Science, Tokyo Institute of Technology, Tokyo, Japan}

\icmlcorrespondingauthor{Saber Malekmohammadi}{saber.malekmohammadi@uwaterloo.ca}

\icmlkeywords{Machine Learning, ICML}

]



\printAffiliationsAndNotice{} 
\begin{abstract}
High utility and rigorous data privacy are of the main goals of a federated learning (\FL) system, which learns a model from the data distributed among some clients. The latter has been tried to achieve by using differential privacy in \FL (\DPFL). There is often heterogeneity in clients' privacy requirements, and existing \DPFL works either assume uniform privacy requirements for clients or are not applicable when server is not fully trusted (our setting). Furthermore, there is often heterogeneity in batch and/or dataset size of clients, which as shown, results in extra variation in the \DP noise level across clients' model updates. With these sources of heterogeneity, straightforward aggregation strategies, e.g., assigning clients' aggregation weights proportional to their privacy parameters ($\epsilon$) will lead to lower utility. We propose \algname{Robust-HDP}, which efficiently estimates the true noise level in clients' model updates and reduces the noise-level in the aggregated model updates considerably. \algname{Robust-HDP} improves utility and convergence speed, while being safe to the clients that may maliciously send falsified privacy parameter $\epsilon$ to server. Extensive experimental results on multiple datasets and our theoretical analysis confirm the effectiveness of \algname{Robust-HDP}. Our code can be found \href{https://github.com/Saber-mm/HDPFL.git}{here}. 
\end{abstract}

\section{Introduction} \label{sec:related_work}
In the presence of sensitive information in the train data, \FL algorithms must be able to provide rigorous data privacy guarantees against a potentially curious server or any third party \cite{Hitaj2017DeepMU, Rigaki2020ASO, Wang2018BeyondIC, Zhu2019DeepLF, Geiping2020InvertingG}. Differential Privacy \cite{Dwork2006, Dwork2006OurDO,Dwork2011AFF, Dwork2014TheAF} has been used in \DPFL systems to achieve such formal privacy guarantees. When there is a trusted server in the system, \DP is provided by the server adding controlled noise to the aggregation of clients' updates \cite{McMahan2018LearningDP, Geyer2017DPFedAvg}. When there is no trusted server, local perturbations, where each client randomizes its updates locally before sending them to the server, is also a solution \cite{Zhao2020LocalDP}. However, this method is limited in the sense that achieving privacy while preserving model utility is challenging, due to clients' independent local noise additions. Some solutions have been proposed for improving utility of this method, e.g., using a trusted shuffler system \cite{Liu2021FLAMEDP, Girgis2021}, which may be difficult to establish if the server itself is not trusted.

Clients often have heterogeneous privacy preferences coming from their varying privacy policies. Furthermore, dataset size usually varies a lot across clients. Additionally, depending on their computational budgets, some clients may use relatively smaller batch sizes locally for running \algname{DPSGD} algorithm \cite{Abadi2016}. As we will show, a small privacy parameter ($\epsilon$) and/or a small batch size lead to a fast increment of the noise level in a client's model update. Existing heterogeneous \DPFL works mostly either depend on a trusted server \cite{Chathoth2022cohortDP, PDPFLglobecomm2022}, or suffer from suboptimal and vulnerable aggregation strategies on an untrusted server (based on clients' privacy parameters \cite{Liu2021ProjectedFA}). We consider a heterogeneous \DPFL systems with an \emph{untrusted} server and propose an efficient aggregation strategy for the server, which is aware of the noise level in each client's model update: we propose to employ Robust PCA (RPCA) algorithm \cite{Candes2009RobustPC} by the untrusted server to estimate the amount of noise in clients' model updates, which we show depends strongly on multiple factors (e.g., their privacy parameter and their batch size ratio), and assign their aggregation weights accordingly. The use of this efficient strategy on the server, which is independent of clients sending any privacy parameters to the server or not, improves model utility and  convergence speed while being robust to potential falsifying clients. The highlights of our contributions are the followings:

\begin{itemize}
    \item We show the effect of privacy parameter and batch/dataset size on the noise level in clients' updates.
    \item We propose ``\algname{Robust-HDP}'', a noise-aware robust algorithm for heterogeneous \DPFL.
    \item As the first work assuming heterogeneous dataset sizes, heterogeneous batch sizes, non-uniform and varying aggregation weights and partial participation of clients simultaneously, we prove convergence of our proposed algorithm under mild assumptions on loss functions.
    \item In various heterogeneity scenarios across clients, we show that \algname{Robust-HDP} improves utility and convergence speed while respecting clients' privacy.  
\end{itemize}

\section{Related work}
\paragraph{Differential privacy.} In this work, we use the following definition of differential privacy:

\begin{definition}[($\epsilon,\delta$)-\DP \cite{Dwork2006OurDO}]
\label{def:epsilondeltadp}
A randomized mechanism $\mathcal{M}:\mathcal{D}\to \mathcal{R}$ with domain $\mathcal{D}$ and range $\mathcal{R}$ satisfies $(\epsilon,\delta)$-\DP if for any two adjacent inputs $d$, $d'\in \mathcal{D}$, which differ only by a single record, and for any measurable subset of outputs $\mathcal{S} \subseteq \mathcal{R}$ it holds that
\begin{align*}
    \texttt{Pr}[\mathcal{M}(d)\in \mathcal{S}] \leq e^{\epsilon} \texttt{Pr}[\mathcal{M}(d')\in \mathcal{S}]+\delta.
\end{align*}
\end{definition}

Gaussian mechanism, which randomizes the output of a non-private computation $f$ on a dataset $d$ as $\mathbf{G_{\sigma}}f(d) \triangleq f(d)+\mathcal{N}(\mathbf{0},\sigma^2)$, provides ($\epsilon,\delta$)-\DP. 
The variance of the noise, $\sigma^2$, is calibrated to the sensitivity of $f$, i.e., the maximum amount of change in its output (measured in $\ell_2$ norm) on two neighboring datasets $d$ and $d'$. Gaussian mechanism has been used in \algname{DPSGD} algorithm \citep{Abadi2016} for private ML to randomize intermediate data-dependent computations, e.g., gradients. Some prior works \cite{GurAri2018GradientDH} found that stochastic gradients stay in a low-dimensional space during training with Stochastic Gradient Descent (\algname{SGD}). Inspired by this, \citet{projecteddpsgd} proposed projection-based variant of the \algname{DPSGD} \cite{Abadi2016} algorithm (projected \algname{DPSGD}), which improves utility by removing the unnecessary noise from noisy batch gradients by projecting them on a linear subspace obtained from a public dataset. Personalized \DP (\PDP), which specifies a separate privacy parameter $\epsilon$ for each data sample in a dataset, was used for centralized settings in \cite{Alaggan_Gambs_Kermarrec_2017, Jorgensen2015ConservativeOL, Huang2020ImprovingLM, Doudalis2017OnesidedDP, yu2023individual}, followed by some recent works in \cite{Boenisch_NEURIPS2023, Heo2023PersonalizedDU}. Another similar work in \cite{Niu2020UtilityawareEM} proposed ``Utility Aware Exponential Mechanism'' (UPEM) to pursue higher utility while achieving \PDP. In the same direction of improving utility, \citet{Shi2021SelectiveDP} proposed ``Selective \DP'' for improving utility by leveraging the fact that private information in natural language is sparse.

\begin{figure}[t]
    \centering
    \setlength{\columnsep}{0pt}\includegraphics[width=0.8\columnwidth,height=3.5cm]{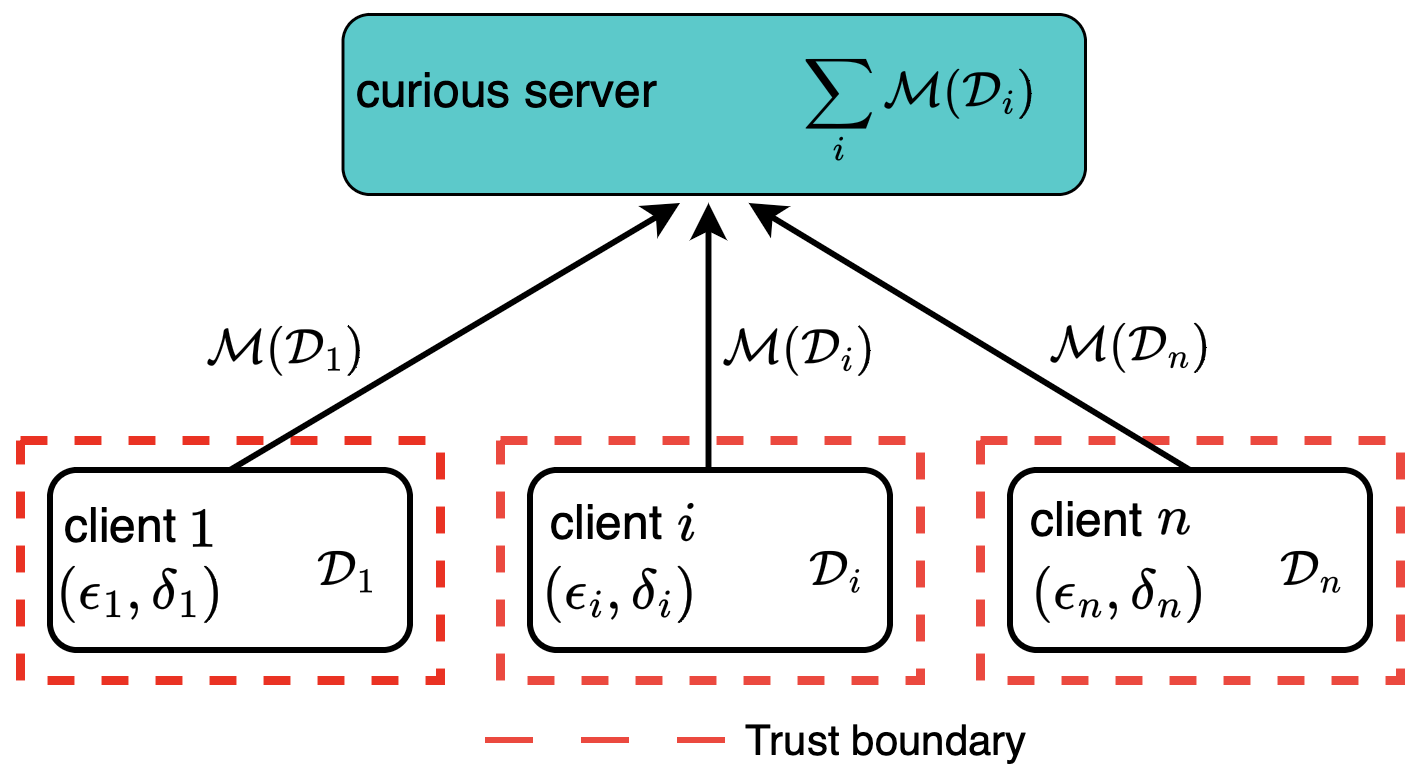}
    \vspace{-1em}
    \caption{Security model in heterogeneous \DPFL, where client $i$ has local train data $\mathcal{D}_i$ and privacy parameters $(\epsilon_i, \delta_i)$, and does not trust any external parties.}
    \label{fig:security_model}
    \vspace{-1.2em}
\end{figure}

\vspace{-1em}
\paragraph{Heterogeneous \DPFL.} Assuming the existence of a \emph{trusted} server, \citet{Chathoth2022cohortDP} proposed cohort-level privacy with privacy and data heterogeneity across cohorts using $\epsilon$-\DP definition (\cref{def:epsilondeltadp} with $\delta=0$). Also, the work in \cite{PDPFLglobecomm2022}, adapted the non-uniform sampling idea of \cite{Jorgensen2015ConservativeOL} to the \FL settings with a \emph{trusted} server to get client-level \DPFL (i.e., $d$ and $d'$ differ by one client's whole data) against membership inference attacks \cite{Rigaki2020ASO, Wang2018BeyondIC}. In contrast, we consider \emph{untrusted} servers. 

The output of an algorithm $\mathcal{M}$, in the sense of \cref{def:epsilondeltadp}, is all the information that the untrusted server, which we want to protect against, observes. We consider heterogeneous \DPFL (\cref{fig:security_model}), where each client $i$ has its own privacy parameters ($\epsilon_i, \delta_i$), and sends data-dependent computation results $\mathcal{M}(\mathcal{D}_i)$ (model updates) to the server. Also, in the context of \cref{def:epsilondeltadp}, the notion of neighboring datasets that we consider in this work, refers to pair of federated datasets $d={\{\mathcal{D}_1, \cdots, \mathcal{D}_n\}}$ and $d'={\{\mathcal{D}_1, \cdots, \mathcal{D}_n\}}$, differing by one data point of one client (i.e., \emph{record-level} \DPFL). 
\begin{table*}[t!]
    \caption{Features of different heterogeneous \DPFL algorithms. $\times$: needed at server, $\checkmark$: not needed.}
    \centering
    \begin{tabular}{lcccc}
    \toprule 
   \bf algorithm & \bf aggregation strategy  & \bf $\{\epsilon_i\}_{i=1}^n$ & \bf clients clustering & \bf PCA on clients updates\\
    \midrule
    \bf \algname{WeiAvg} \citep[Alg. \ref{alg:WeiAvg}]{Liu2021ProjectedFA} & $w_i \propto \epsilon_i$ & $\times$ & $\times$ & $\checkmark$\\
    \bf \algname{PFA} \cite{Liu2021ProjectedFA} & $w_i \propto \epsilon_i$ & $\times$ & $\times$ & $\times$\\
    \bf \algname{DPFedAvg} \cite{DPSCAFFOLD2022} & $w_i \propto N_i$ & $\checkmark$ & $\checkmark$ & $\checkmark$\\
    \bf \algname{minimum} $\epsilon$ \cite{Liu2021ProjectedFA} & $w_i \propto N_i$ & $\times$ & $\checkmark$ & $\checkmark$\\
    \midrule
    \bf \algname{Robust-HDP} (Alg. \ref{alg:Robusthdp}) & $w_i \propto \frac{1}{\sigma_i^2}$ & $\checkmark$ & $\checkmark$ & $\checkmark$\\
    \bottomrule
    \end{tabular}
    \label{tab:hetdpfl_algs}
    \vspace{-1em}
\end{table*}
 \citet{Liu2021ProjectedFA} adapted a projection-based approach, similar to that of projected \algname{DPSGD} \cite{projecteddpsgd}, to the heterogeneous \DPFL setting to propose \algname{PFA} and improve utility. Although assuming an untrusted server, their proposed algorithm relies on the assumption that the server knows the clients' \emph{``true"} privacy parameters $\{(\epsilon_i, \delta_i)\}$ and uses them to cluster clients to ``public'' (those with larger privacy parameters) and ``private''. As such, as we show, \algname{PFA} is extremely vulnerable to when clients share a falsified value of their privacy parameters (often larger than their true values) with server. Also, they used aggregation strategy $w_i \propto \epsilon_i$ on server for \algname{PFA} and another algorithm called \algname{WeiAvg} (see \Cref{tab:hetdpfl_algs}). As we will show, \emph{even if the server knows clients' true privacy parameters, this information is not a perfect indication of the ``true noise level" in their model updates, especially with heterogeneous privacy parameters  and batch/dataset sizes}.


The current state of the art in heterogeneous \DPFL calls for a robust algorithm that takes all the mentioned potential sources of heterogeneity across clients into account and achieves high utility and data privacy simultaneously.

\section{The \algname{Robust-HDP} algorithm for heterogeneous \DPFL}\label{sec: batchsize_analysis} 
In this section, we will devise a new heterogeneous \DPFL algorithm, and we first explain the intuitions behind it. Used notations are explained in \cref{tab:notations} and \cref{app:notations} in details. At the $t$-th gradient update step on a current model $\thetav$, client $i$ computes the following noisy batch gradient:
\begin{align}
    \Tilde{g}_i(\thetav) = \frac{1}{b_i}\bigg[ \Big(\sum_{j \in \mathcal{B}_i^t} \Bar{g}_{ij}(\thetav)\Big) + \mathcal{N}(\mathbf{0}, \sigma_{i, \texttt{\DP}}^2 \mathbb{I}_p)\bigg],
    \label{eq:noisy_sg}
\end{align}
where $\Bar{g}_{ij}(\thetav) = \texttt{clip}(\nabla \ell(h(x_{ij},\thetav), y_{ij}), c)$, and $c$ is a clipping threshold. For a given vector $\mathbf{v}$, $\texttt{clip}(\mathbf{v}, c) =  \min\{\|\mathbf{v}\|, c\} \cdot \frac{\mathbf{v}}{\|\mathbf{v}\|}$. Also,  $\sigma_{i, \texttt{\DP}}=c\cdot z(\epsilon_i, \delta_i, q_i, K_i, E)$: knowing $E$ (global communication rounds), client $i$ can compute $z(\epsilon_i, \delta_i, q_i, K_i, E)$ locally, which is the noise scale that it should use locally for \algname{DPSGD} in order to achieve $(\epsilon_i, \delta_i)-$\DP with respect to $\mathcal{D}_i$ at the end of $E$ global rounds. This can be done by client $i$ using a privacy accountant, e.g., the moments accountant \cite{Abadi2016}. Therefore, \emph{depending on its privacy preference $(\epsilon_i, \delta_i)$}, each client $i$ computes its required noise scale $z$, runs \algname{DPSGD} locally and sends its noisy model updates to the server at the end of each round. Now, an important question is that what is an efficient aggregation strategy for the server to aggregate the clients' noisy model updates? Intuitively, the server has to pay more attention to the less noisy updates. The challenge is that the server knows neither the noise added by each client $i$ nor its amount. To answer the above question, we first analyze the behavior of the noise level in clients' batch gradients in \cref{sec:batch_grad_noise}, which is used in \cref{sec:noisy_updates}, for a similar analysis of clients' uploaded model updates. The result of this analysis is an idea we propose for the server to estimate the noise amount in each model update, which leads to an efficient aggregation strategy in \cref{sec:robust_hdp}.

\begin{figure*}[hbt!]
    \centering
    \setlength{\columnsep}{0pt}
    \includegraphics[width=0.9\columnwidth,height=6cm]{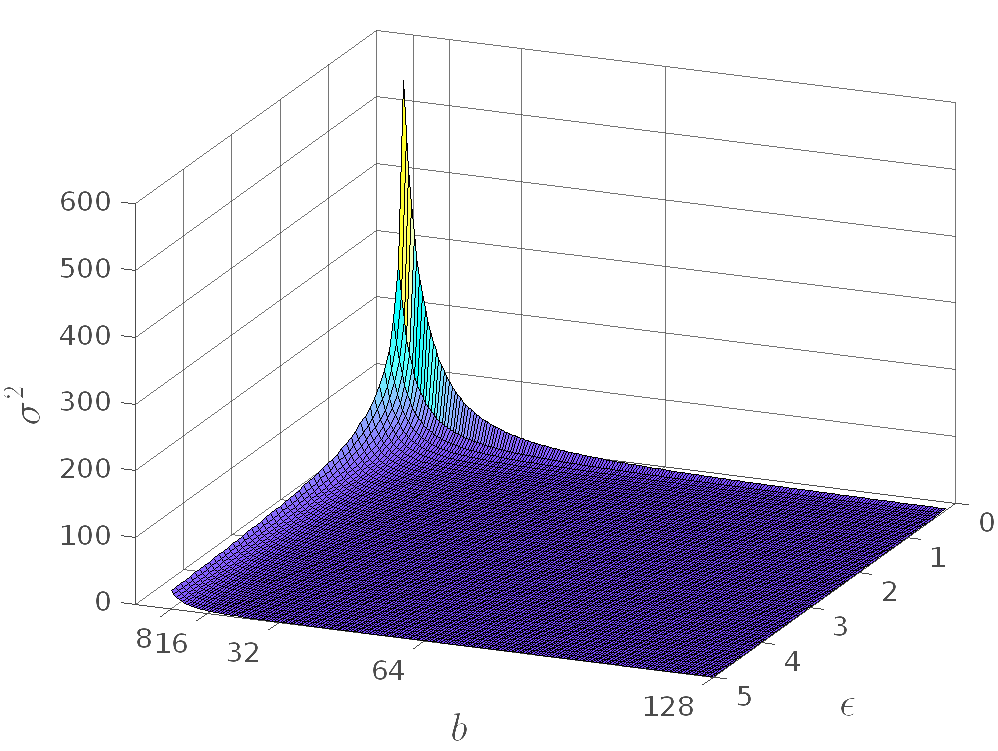}
    \quad \quad  \includegraphics[width=0.9\columnwidth,height=5.1cm]{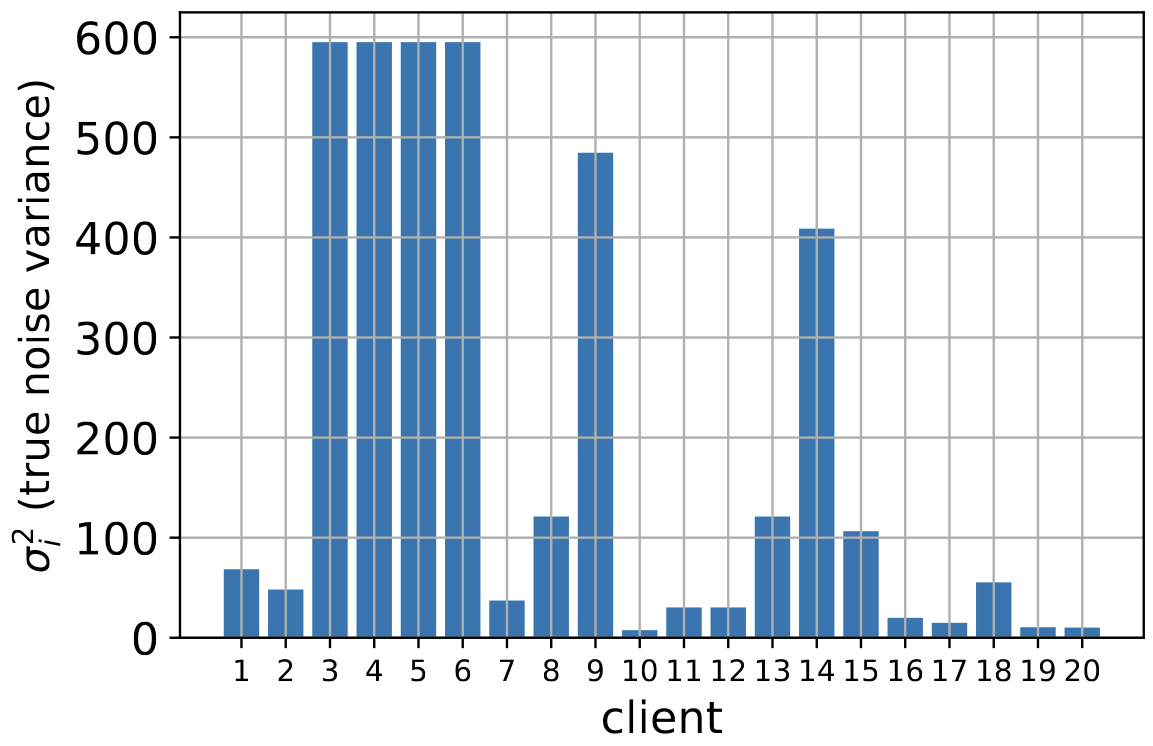}
    \caption{Left: 3D plot of noise variance $\sigma_i^2$ of a client $i$ (\cref{eq:sigma_i^2} with $K_i=1, N_i=2400, \eta_l = 0.01, c=3 , p=28939$) based on $b_i$ and the privacy budget $\epsilon_i$. Right: the noise variances $\{\sigma_i^2\}_{i=1}^n$ in a \DPFL system with $n=20$ clients, where $\{(\epsilon_i, b_i)\}_{i=1}^n$ are randomly selected for each client. It clearly shows an approximately \emph{sparse pattern} (14 of the clients have much smaller noise variance than the other 6). Each bar plot in the right figure corresponds to a point in the left figure. 
    }
    \label{fig:var_epsilon_b}
\end{figure*}

\vspace{-0.5em}

\begin{table}[hbt!]
    \caption{Used notations (also see \cref{app:notations})}
    \vspace{0.2em}
    \begin{tabularx}{\columnwidth}{p{0.15\columnwidth}X}
    \toprule 
      $n$ & number of clients, which  are indexed by $i$\\
      $x_{ij}, y_{ij}$ & $j$-th data point of client $i$ and its label \\ 
      $\mathcal{D}_i, N_i$ & local train set of client $i$ and its size \\
      $\mathcal{B}_i^{t}$& the train data batch used by client $i$ at the $t$-th gradient update \\
      $b_i$ & batch size of client $i$: $|\mathcal{B}_i^{t}|=b_i$\\
      $q_i$ & batch size ratio of client $i$: $q_i = \frac{b_i}{N_i}$\\
      $\epsilon_i, \delta_i$ & client $i$'s desired \DP privacy parameters \\
      $E$ & total number of global communication rounds in the \DPFL system, indexed by $e$\\
      $\mathcal{S}^e$ & set of participating clients in round $e$\\
      $\thetav^e$ & global model parameter, which has size $p$, at the beginning of global round $e$ \\
      $K_i$ & number of local train epochs performed by client $i$ during each global round $e$\\
      $E_i$ & number of batch gradient updates of client $i$ during each global round $e$: $E_i = K_i \cdot \lceil \frac{N_i}{b_i} \rceil$\\
      $h$ & predictor function, e.g., CNN model, with parameter $\thetav$ \\
      $\ell$ & cross entropy loss\\
      $\sigma_{i, \Tilde{g}}^2$ & variance of the noisy stochastic batch gradient $\Tilde{g}_i(\thetav)$ of client $i$ \\
      $\sigma_{i}^2$ & conditional variance of the noisy model update $\Delta \Tilde{\thetav}_i^e$ of client $i$: $\texttt{Var}(\Delta \Tilde{\thetav}_i^e|\thetav^e)$ \\
      
      \bottomrule
     \end{tabularx}
     \label{tab:notations}
\end{table}

\subsection{Noise level in clients' \DP batch gradients}\label{sec:batch_grad_noise}
We consider two cases, which are easier to analyze. Our analysis gives us an understanding of the parameters affecting the noise level in clients' batch gradients. Depending on the value of the used clipping threshold $c$ at the $t$-th gradient update step, we consider two general indicative cases:

\vspace{-0.5em}
\paragraph{1. Effective clipping threshold for all samples:}
in this case, from \cref{eq:noisy_sg}, we have:
\begin{align}
    \mathbb E[\Tilde{g}_i(\thetav)] = \frac{1}{b_i}\sum_{j \in \mathcal{B}_i^t} \mathbb E[\bar{g}_{ij}(\thetav)] = \frac{1}{b_i}\sum_{j \in \mathcal{B}_i^t} G_i(\thetav) = G_i(\thetav),
\label{expectation_gtilde}
\end{align}
where the expectation is with respect to the stochasticity of gradients and we have assumed that $E[\bar{g}_{ij}(\thetav)]$ is the same for all $j$ and is denoted by $G_i(\thetav)$. Now, for an arbitrary random variable $\mathbf{v} = (v_1, \ldots, v_p)^\top \in \mathbb{R}^{p\times 1}$, we define $\texttt{Var}(\mathbf{v}):= \sum_{j=1}^p \mathbb E[(v_j - \mathbb E[v_j])^2]$, i.e., variance of $\mathbf{v}$ is the sum of the variances of its elements. Then, the variance of the noisy stochastic gradient in \Cref{eq:noisy_sg}, which is also random, can be computed as (see \Cref{app:variance_derivation}):
\begin{align}
    \sigma_{i, \Tilde{g}}^2 &:= \texttt{Var}(\Tilde{g}_i(\thetav)) \nonumber \\
    & =  \frac{c^2 - \big\| G_i(\thetav)\big\|^2}{b_i} + \frac{p c^2 z^2(\epsilon_i, \delta_i, q_i, K_i, E)}{b_i^2} \nonumber \\
    & \approx \frac{p c^2 z^2(\epsilon_i, \delta_i, q_i, K_i, E)}{b_i^2},
\label{eq:var_g_effective}
\end{align}
where, the estimation is valid because $p\gg1$. For instance, $p\approx 2\times 10^7$ for ResNet-34 for CIFAR100, and $c=3$.

\paragraph{2. Ineffective clipping threshold for all samples:}
in this case, we have a noisy version of the batch gradient $g_i(\thetav) = \frac{1}{b_i} \sum_{j \in \mathcal{B}_i^t} g_{ij}(\thetav)$, which is unbiased with variance bounded by $\sigma_{i, g}^2$ (see Assumption \ref{assump:lipschitz_smooth_bounded}). Hence:
\begin{align}
    \mathbb E[\Tilde{g}_i(\thetav)] = \mathbb E[g_i(\thetav)] = \nabla f_i(\thetav),
\end{align}
\begin{align}
    \sigma_{i, \Tilde{g}}^2 &= \texttt{Var}(\Tilde{g}_i(\thetav)) = \texttt{Var}(g_i(\thetav)) + \frac{p \sigma_{i, \texttt{\DP}}^2}{b_i^2} \nonumber \\
    &\leq \sigma_{i, g}^2 + \frac{p c^2 z^2(\epsilon_i, \delta_i, q_i, K_i, E)}{b_i^2}.
\label{var_g_ineffective}
\end{align}
$z$ is a sub-linearly increasing function of $q_i$ (and equivalently $b_i$: see \Cref{thm:localdp} and \Cref{fig:zvsq} in the appendix). It is also clear that $z$ is a decreasing function of $\epsilon_i$ and $\delta_i$. Hence, \emph{$\sigma_{i, \Tilde{g}}^2$ is a decreasing function of $b_i$ (batch size), $N_i$ (dataset size) and $\epsilon_i$, and also an increasing function of $q_i$ (batch size ratio)}.

\vspace{-0.5em}
\subsection{Noise level in clients' \DP model updates}\label{sec:noisy_updates}
Having found the parameters affecting $\sigma_{i, \Tilde{g}}^2$, we now investigate the parameters affecting the noise level in clients' model updates. During each global communication round $e$, a participating client $i$ performs $E_i = K_i \cdot \lceil \frac{N_i}{b_i} \rceil = K_i \cdot \lceil \frac{1}{q_i} \rceil$ batch gradient updates locally with step size $\eta_l$:
\begin{align}
    &\Delta \Tilde{\thetav}_i^e =  \thetav_{i, E_i}^e - \thetav_{i, 0}^e,\nonumber \\
    &\thetav_{i, k}^e = \thetav_{i, k-1}^e -\eta_l \Tilde{g}_i(\thetav_{i, k-1}^e), ~ k=1, \ldots, E_i,
\end{align}
where $\thetav_{i, 0}^e = \thetav^e$. In each update, it adds a Gaussian noise from $\mathcal{N}(\mathbf{0}, \frac{c^2 z^2(\epsilon_i, \delta_i, q_i, K_i, E)}{b_i^2}\mathbb{I}_p)$ to its batch gradients independently (see \Cref{eq:noisy_sg}). Hence:
\begin{align}\label{eq:sigma_i^2}
    \sigma_i^2 := \texttt{Var}(\Delta \Tilde{\thetav}_i^e|\thetav^e)
    & = K_i \cdot \lceil \frac{1}{q_i} \rceil \cdot \eta_l^2 \cdot \sigma_{i, \Tilde{g}}^2,
\end{align}
where $\sigma_{i, \Tilde{g}}^2$ was computed in \Cref{eq:var_g_effective} and \Cref{var_g_ineffective} for two general indicative cases. This means that $\sigma_i^2$ heavily depends on $b_i$ (e.g., when clipping is effective, $b_i$ appears with power 3 in denominator. Recall $\frac{1}{q_i} = \frac{N_i}{b_i}$). Hence, $\sigma_i^2$ decreases quickly when $b_i$ increases. Similarly, $\sigma_i^2$ is a non-linearly decreasing function of $\epsilon_i$ (see \Cref{fig:var_epsilon_b}, left). However, note that $N_i$ and $q_i$ appear twice in \cref{eq:sigma_i^2} with opposing effects. This makes the variation of $\sigma_i^2$ with $N_i$ and $q_i$ small (explained in details in \cref{app:uuh}). An important message of these important understandings is that \emph{$\epsilon_i$ is not the only parameter of client $i$ that determines $\sigma_i^2$}.

\subsection{Optimum aggregation strategy} 
Assuming the set of participating clients $\mathcal{S}^e$ in round $e$, we have to solve the following problem to minimize the total noise after the aggregation at the end of this round:

\begin{align}
    \min_{w_i \geq 0} \quad &\texttt{Var}\big(\sum_{i \in \mathcal{S}^e} w_i \Delta \Tilde{\thetav}_i^e\big | \thetav^e \big)  = \sum_{i \in \mathcal{S}^e} {w_i}^2 \sigma_i^2,\nonumber \\
    &\texttt{s.t.} \sum_{i \in \mathcal{S}^e} w_i =1,
    \label{eq:w_opt}
\end{align}
which has a unique solution $w_i^*\propto \frac{1}{\sigma_i^2}$. Hence, the optimum aggregation strategy weights clients directly based on $\{\sigma_i^2\}_{i=1}^n$, which as shown,\emph{ not only depends on $\{\epsilon_i\}_{i=1}^n$ non-linearly, but it also depends on $\{b_i\}_{i=1}^n$ and $\{N_i\}_{i=1}^n$}. This point makes the aggregation strategy $w_i \propto \epsilon_i$ of \algname{PFA} and \algname{WeiAvg} algorithms \cite{Liu2021ProjectedFA} suboptimal, let alone its vulnerability to a client $i$ sharing a falsified $\epsilon'_i>\epsilon_i$ with the server to either attack the system, to get a larger aggregation weight, or to get a larger payment from a  server which incentivizes participation by payment to clients \cite{modelsharinggames, karimreddy_datasharing, fallah2023optimal, fairvalueofdata} (as a larger $\epsilon_i$ means a more exploitable data from client $i$). The same vulnerability discussion applies to the clustering of clients based on their shared privacy parameter $\epsilon$ (used in \algname{PFA}). Having these shortcomings of the existing algorithms as a motivation, how can we implement the optimum aggregation strategy when the untrusted server does not have any idea of the clients noise addition mechanisms and $\{\sigma_i^2\}_{i=1}^n$? We next propose our idea for estimating $\{\sigma_i^2\}_{i=1}^n$ and $\{w_i^*\}_{i=1}^n$.

\vspace{-0.5em}
\subsection{Description of \algname{Robust-HDP} algorithm}\label{sec:robust_hdp}
Assuming a \DPFL system with $n$ clients and full participation of clients for simplicity, at the end of each global round $e$, the server gets the matrix $\mathbf{M} := [\Delta \Tilde{\thetav}_1^e|\ldots|\Delta \Tilde{\thetav}_n^e]$. Assuming an \emph{i.i.d} or moderately heterogeneous data split, and based on the findings in \cite{GurAri2018GradientDH, projecteddpsgd}, we would expect $\mathbf{M}$ to have a low rank if there was no \DP/stochastic noise in $\{\Delta \Tilde{\thetav}_i^e\}_{i=1}^n$. So we can think about writing $\mathbf{M}$ as the summation of an underlying low-rank matrix $\mathbf{L}$ and a noise matrix $\mathbf{S}$:
\begin{align*}
    \vspace{-0.5em}
    \mathbf{M} = [\Delta \Tilde{\thetav}_1^e|\ldots|\Delta \Tilde{\thetav}_n^e] = \mathbf{L} + \mathbf{S}.
    \vspace{-1em}
\end{align*}
If the matrix $\mathbf{S}$ is sparse, not only can such a decomposition problem be solved using RPCA, it can be solved by a very convenient convex optimization program called \emph{Principal Component Pursuit} (Algorithm \ref{alg:RPCA} in the appendix) without imposing much computational overhead to the server \cite{Candes2009RobustPC}. Surprisingly, the entries in $\mathbf{S}$ can have arbitrarily large magnitudes. Theoretically, this is guaranteed to work even if $\textit{rank}(\mathbf{L})\in \mathcal{O}(n/(\log p)^2)$, i.e., the rank of $\mathbf{L}$ grows almost linearly in $n$ (see Theorem 1.1 in \cite{Candes2009RobustPC}). Hence, we expect to be able to do such a decomposition as long as we have a moderately heterogeneous data distribution across a large enough number of clients (also, see \cref{app:future_dirs} for detailed discussion on data heterogeneity, further experiments, and future directions). Hence, $\mathbf{L}$ will be a low rank matrix, estimating the ``true'' values of clients' updates and $\mathbf{S}$ will capture the noises in clients model updates $\{\Delta \Tilde{\thetav}_i^e\}_{i=1}^n$ induced by two sources: \DP additive Gaussian noise and batch gradients stochastic noise. Therefore, we can use $\hat{\sigma}_i^2:=\|\mathbf{S}_{:,i}\|_2^2$ ($\mathbf{S}_{:,i}$ is the $i$-th column of $\mathbf{S}$, corresponding to client $i$) as an estimate of $\sigma_i^2$ (\cref{eq:sigma_i^2}). Indeed, we observed such approximately sparse pattern for $\mathbf{S}$ in \cref{fig:var_epsilon_b} (right), where each barplot corresponds to the $\ell_2$ norm of one column of $\mathbf{S}$. Thus, according to \cref{eq:w_opt}, we assign the aggregation weights as $w_i^e = \frac{{1}/{\hat{\sigma}_i^2}}{\sum_{j\in \mathcal{S}^e} {1}/{\hat{\sigma}_j^2}}$, where $\hat{\sigma}_i^2=\|\mathbf{S}_{:,i}\|^2$ (see \cref{alg:Robusthdp}). \emph{Interestingly, this estimation is independent of clients' shared $\epsilon$ parameter values, which makes our \algname{Robust-HDP} optimal, robust and vastly applicable.}

\begin{algorithm}[t]
\caption{\algname{Robust-HDP}}
\label{alg:Robusthdp}
\KwIn{Initial parameter $\thetav^0$, batch sizes $\{b_1, \ldots, b_n\}$, dataset sizes $\{N_1, \ldots, N_n\}$, noise scales $\{z_1, \ldots, z_n\}$, gradient norm bound $c$, local epochs $\{K_1, \ldots, K_n\}$, global round $E$, number of model parameters $p$, privacy accountant \textbf{\algname{PA}}.}

\KwOut{$\thetav^E, \{\epsilon_1^E, \ldots, \epsilon_n^E\}$}


\textbf{Initialize} $\thetav_0$ randomly

\For{$e\in [E]$}
{
sample a set of clients $\mathcal{S}^e \subseteq \{1, \ldots, n\} $

\For{each client $\userind \in \mathcal{S}^e$ \textbf{in parallel}}{
  $\Delta \Tilde{\thetav}_i^e \gets$\textbf{\algname{DPSGD}($\thetav^e, b_i, N_i, K_i, z_i, c$)}
  
  $\epsilon_i^e \gets \textbf{\algname{PA}}(\frac{b_i}{N_i}, z_i, K_i, e)$
  }
  $\mathbf{M} = [\Delta \Tilde{\thetav}_1^e|\ldots|\Delta \Tilde{\thetav}_{|\mathcal{S}^e|}^e] \in \mathbb R^{p\times|\mathcal{S}^e|}$
  
  $\mathbf{L}, \mathbf{S}$ = \textbf{\algname{RPCA}}($\mathbf{M}$)
  
\For{$i \in \mathcal{S}^e$}
{
    $w_i^e \gets \frac{{1}/{\|\mathbf{S}_{:,i}\|_2^2}}{\sum_{j \in \mathcal{S}^e} {1}/{\|\mathbf{S}_{:,j}\|_2^2}}$
}
{$\thetav^{e+1} \gets \thetav^e + \sum_{i \in \mathcal{S}^e} w_i^e \Delta \Tilde{\thetav}_i^e$}
}
\end{algorithm}
\setlength{\textfloatsep}{10pt}

\subsection{Reliability of \algname{Robust-HDP}}
\label{sec:reliability_rpdp} In order for \algname{Robust-HDP} to assign the optimum aggregation weights $\{w_i^*\}$, \emph{it suffices to estimate the set $\{\sigma_i^2\}$ up to a multiplicative factor}. Assuming participants $\mathcal{S}^e$ in round $e$, let $s_{i,j}$ in matrix $\bf S$ represent the true value of noise in the $i$-th element of $\Delta \Tilde{\thetav}_j^e$ ($j \in \mathcal{S}^e$). Then, assume that $\bf S'$ is the matrix computed by \algname{Robust-HDP} at the server with bounded elements $s'^2_{i,j} \leq U$, where $\mathbb E[s'_{i,j}] = r s_{i,j}$, for some constant $r>0$, and $\mathbb E[|s'_{i,j} - r s_{i,j}|^2] \leq \alpha_j^2$ (i.e., on average, \algname{Robust-HDP} is able to estimate the true noise values $s_{i,j}$ up to a multiplicative factor $r$ by using RPCA). Then, from Hoeffding's inequality, we have:
\begin{align}
    \texttt{Pr}(|\hat{\sigma}_j^2 - (r^2 \sigma_j^2 + \alpha_j^2)|>\epsilon) \leq 2e^{\frac{-2p\epsilon^2}{U^2}},
\end{align}

meaning that estimating the entries of $\bf S$ up to a multiplicative factor $r$ with a small variance is enough for \algname{Robust-HDP} to estimate $\{\sigma_i^2\}$ up to a multiplicative factor $r^2$ with high probability. This probability increases with the number of model parameters $p$ exponentially: the $p$ noise elements of $\mathbf{S}_{:,i}$ are \emph{i.i.d}, and larger $p$ means having more samples from the same distribution to estimate its variance (see also Theorem 1.1 in \cite{Candes2009RobustPC}). Also, $w_j \propto \frac{1}{\hat{\sigma}_j^2} \approx \frac{1}{r^2\sigma_j^2 + \alpha_j^2}$. Hence, as $\sigma_j^2 \gg 1$ (it is the noise variance in the whole model update $\Delta \Tilde{\thetav}_j^e$. See the values in \cref{fig:var_epsilon_b}, right), a small deviation $\alpha_j^2$ from $r^2 \sigma_j^2$ still results in aggregation weights close to the optimum weights $\{w_i^*\}$.

\subsection{Scalability of \algname{Robust-HDP} with the number of model parameters \texorpdfstring{$p$}{Lg}}\label{sec:scalability}
The computation time (precision) of RPCA algorithm increases (decreases) when the number of model parameters $p$ grows. As such, in order to make the \algname{Robust-HDP} scalable for large models, we perform the noise estimation of \texorpdfstring{\algname{Robust-HDP}}{Lg} on sub-matrices of $\mathbf{M}$ with smaller rows:
\begin{align}
    &\mathbf{M}_1 = \mathbf{M}[0:p'-1,:]=\mathbf{L}_1 + \mathbf{S}_1 \nonumber\\
    &\mathbf{M}_2=\mathbf{M}[p':2p'-1,:]=\mathbf{L}_2 + \mathbf{S}_2 \nonumber\\
    &\dots \nonumber\\
    &\mathbf{M}_Q=\mathbf{M}[p-p':p-1,:]=\mathbf{L}_Q + \mathbf{S}_Q, \nonumber
\end{align}
where $Q=\floor{\frac{p}{p'}}$. Then, we get a set of noise variance estimates $\{Q \cdot \hat{\sigma}_i^2\}_{i=1}^n$ from each $\mathbf{S}_j, j \in \{1, \ldots, Q\}$. Finally, we use the sets' element-wise average for weight assignment. For instance, for CIFAR10 and CIFAR100, we perform RPCA on sub-matrcies of $\mathbf{M}$ with $p'=200,000$ rows, and average their noise variance estimates. Our experimental results show that this approach, even with $Q=1$ (i.e., using just $\textbf{M}_1$), still results in assigning aggregation weights close to the optimum weights $\{w_i^*\}$. This idea makes \algname{Robust-HDP} scalable to large models with large $p$.

\subsection{Privacy analysis of \algname{Robust-HDP}}
We have the following theorem about \DP guarantees of our proposed \algname{Robust-HDP} algorithm.

\begin{restatable}{theorem}{localdp}
For each client $i$ , there exist constants $c_1$ and $c_2$ such that given its number of steps $E \cdot E_i$, for any $\epsilon < c_1 q_i^2 E \cdot E_i$, the output model of \algname{Robust-HDP} satisfies  $(\epsilon_i, \delta_i)-$\DP with respect to $\mathcal{D}_i$ for any $\delta_i>0$ if $z_i > c_2 \frac{q_i \sqrt{E \cdot E_i \cdot \log\frac{1}{\delta_i}}}{\epsilon_i}$, where $z_i$ is the noise scale used by the client $i$ for \algname{DPSGD}. The algorithm also satisfies $(\epsilon_{\texttt{max}}, \delta_{\texttt{max}})$-\DP, where $(\epsilon_{\texttt{max}}, \delta_{\texttt{max}}) = \big(\max(\{\epsilon_i\}_{i=1}^n), \max(\{\delta_i\}_{i=1}^n)\big)$. 
\label{thm:localdp}
\end{restatable}

Therefore, the model returned by \algname{Robust-HDP} is $(\epsilon_i, \delta_i)$-\DP
with respect to $\mathcal{D}_i$, satisfying heterogeneous \DPFL.

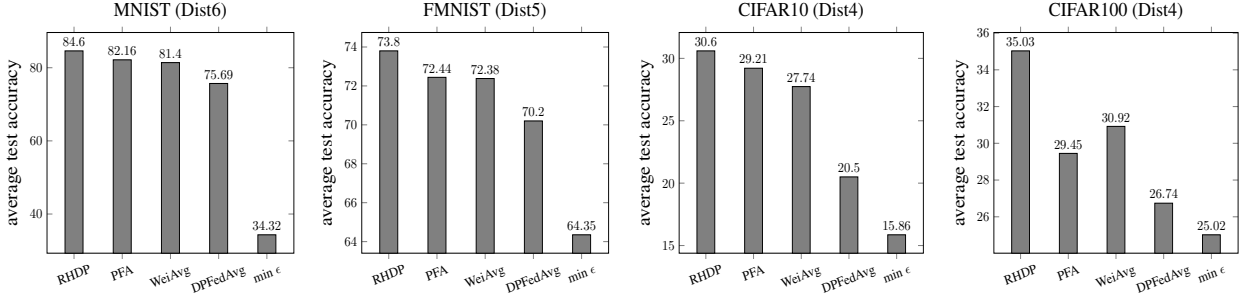
\begin{figure*}[t]
\subfigure{
\resizebox{3.9cm}{4cm}{%
\begin{tikzpicture}
\begin{axis}
[
title=\Large MNIST (Dist6),
ylabel={\Large average test accuracy},  
ybar,
nodes near coords, bar width=0.5cm,
symbolic x coords={\rotatebox{20}{RHDP}, \rotatebox{20}{PFA}, \rotatebox{20}{WeiAvg}, \rotatebox{20}{DPFedAvg}, \rotatebox{20}{min $\epsilon$}},
enlarge x limits=.14]
\addplot [style={black,fill=gray}]coordinates{ (\rotatebox{20}{RHDP},84.60) (\rotatebox{20}{PFA},82.16) (\rotatebox{20}{WeiAvg}, 81.40) (\rotatebox{20}{DPFedAvg},75.69) (\rotatebox{20}{min $\epsilon$},34.32)};
\end{axis}
\end{tikzpicture}
}
}
\subfigure{
\resizebox{3.9cm}{4cm}{%
\begin{tikzpicture}
\begin{axis}
[
title=\Large FMNIST (Dist5),
ylabel={\Large average test accuracy},  
ybar,
nodes near coords, bar width=0.5cm,
symbolic x coords={\rotatebox{20}{RHDP}, \rotatebox{20}{PFA}, \rotatebox{20}{WeiAvg}, \rotatebox{20}{DPFedAvg}, \rotatebox{20}{min $\epsilon$}},
enlarge x limits=.14]
\addplot [style={black,fill=gray}] coordinates{ (\rotatebox{20}{RHDP},73.80) (\rotatebox{20}{PFA},72.44) (\rotatebox{20}{WeiAvg},72.38) (\rotatebox{20}{DPFedAvg},70.20) (\rotatebox{20}{min $\epsilon$},64.35)};
\end{axis}
\end{tikzpicture}
}
}
\subfigure{
\resizebox{3.9cm}{4cm}{
\begin{tikzpicture}
\begin{axis}
[
title=\Large CIFAR10 (Dist4),
ylabel={\Large average test accuracy},  
ybar,
nodes near coords, bar width=0.5cm,
symbolic x coords={\rotatebox{20}{RHDP}, \rotatebox{20}{PFA}, \rotatebox{20}{WeiAvg}, \rotatebox{20}{DPFedAvg}, \rotatebox{20}{min $\epsilon$}},
enlarge x limits=.14]
\addplot [style={black,fill=gray}] coordinates{ (\rotatebox{20}{RHDP},30.60) (\rotatebox{20}{PFA}, 29.21) (\rotatebox{20}{WeiAvg},27.74) (\rotatebox{20}{DPFedAvg},20.50) (\rotatebox{20}{min $\epsilon$},15.86)};
\end{axis}
\end{tikzpicture}
}
}
\subfigure{
\resizebox{3.9cm}{4cm}{
\begin{tikzpicture}
\begin{axis}
[
title=\Large CIFAR100 (Dist4),
ylabel={\Large average test accuracy},  
ybar,
nodes near coords, bar width=0.5cm,
symbolic x coords={\rotatebox{20}{RHDP}, \rotatebox{20}{PFA}, \rotatebox{20}{WeiAvg}, \rotatebox{20}{DPFedAvg}, \rotatebox{20}{min $\epsilon$}},
enlarge x limits=.14]
\addplot [style={black,fill=gray}] coordinates{ (\rotatebox{20}{RHDP},35.03) (\rotatebox{20}{PFA}, 29.45) (\rotatebox{20}{WeiAvg},30.92) (\rotatebox{20}{DPFedAvg},26.74) (\rotatebox{20}{min $\epsilon$},25.02)};
\end{axis}
\end{tikzpicture}
}
}

\vspace{-1.5em}
\caption{Comparison of average test accuracy between studied algorithms. See Tables \ref{table:mnist} to \ref{table:cifar100} in the appendix for detailed results.}
\vspace{-1em}
\label{fig:comparison}
\end{figure*}

\begin{figure*}[t]
\centering
\includegraphics[width=0.95\columnwidth,height=4.5cm]{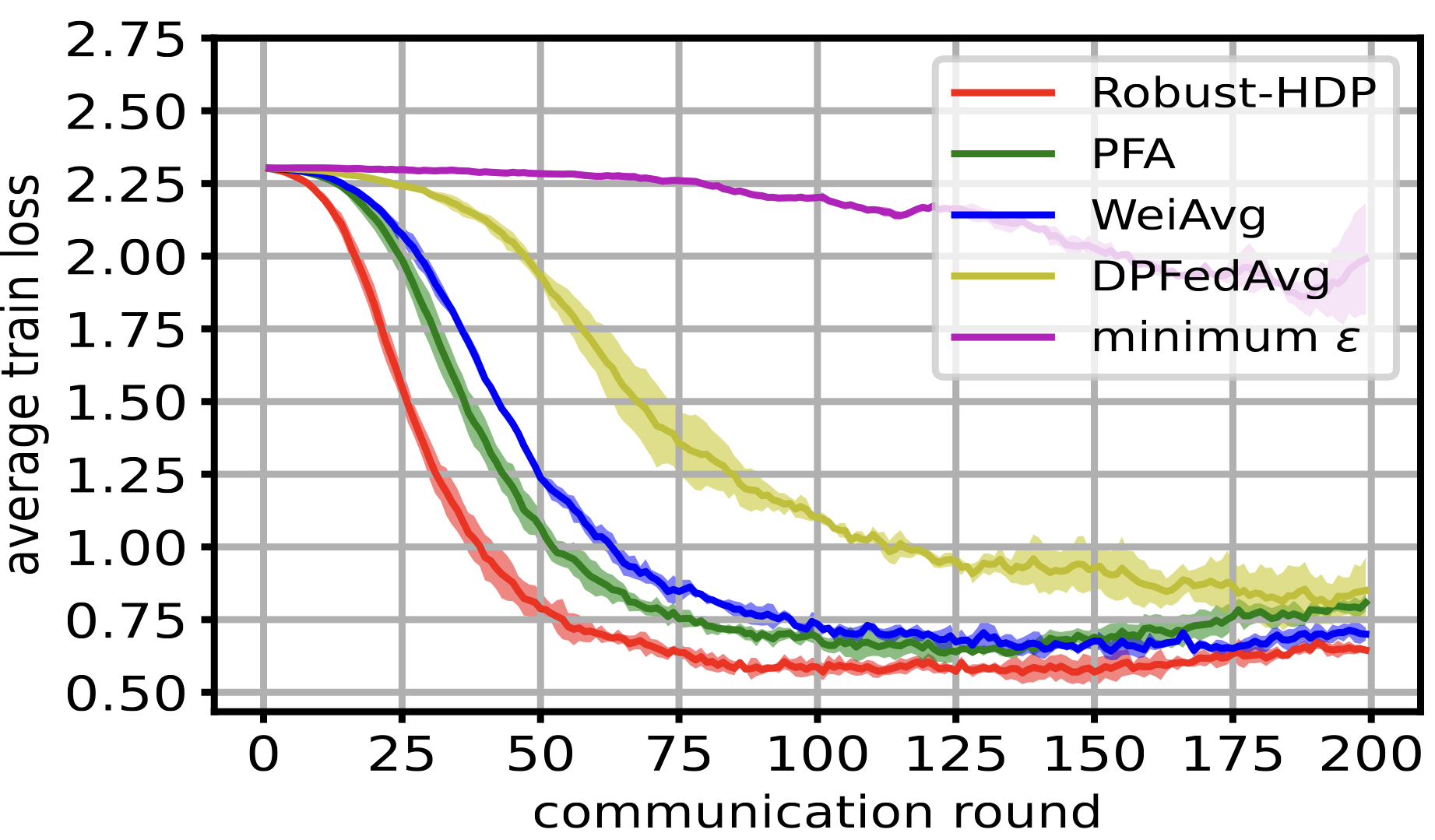}
\includegraphics[width=0.95\columnwidth, height=4.5cm]{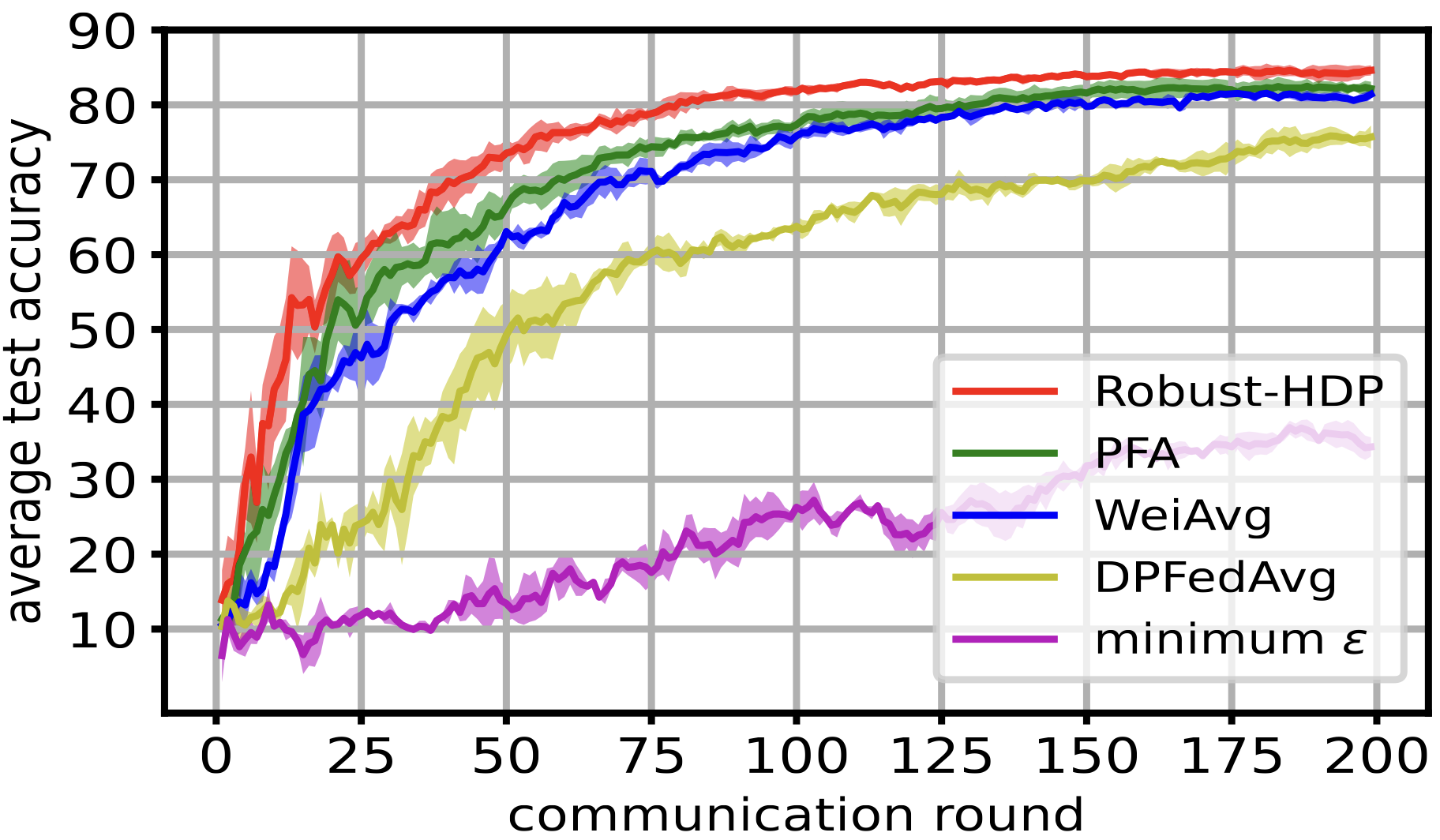}
\vspace{-1.2em}
\caption{Convergence speed comparison on MNIST and Dist6. Minimum $\epsilon$ algorithm diverged in 1 out of 3 trials.}
\vspace{-1em}
\label{fig:conv_speed_mnist_dist6}
\end{figure*}

\subsection{The optimization side of \algname{Robust-HDP}}

We assume that $f(\thetav)= \sum_{i \in [n]} \lambda_i f_i(\thetav)$, where $\lambda_i = \frac{N_i}{\sum_i N_i}$, has minimum value $f^*$ and minimizer $\thetav^*$. We also make some mild assumptions about the loss functions $f_i$ (see Assumptions \ref{assump:lipschitz_smooth_bounded} and \ref{assump:bounded_sample_grad} in the Appendix). We now analyze the convergence of the \algname{Robust-HDP} algorithm.

\begin{restatable}[\textbf{\algname{Robust-HDP}}]{theorem}{Robusthdp}\label{thm:Robusthdp}
Assume that Assumptions \ref{assump:lipschitz_smooth_bounded} and \ref{assump:bounded_sample_grad} hold, and for every $i$, learning rate $\eta_l$ satisfies: $\eta_l \leq \frac{1}{6 \beta E_i}$ and $\eta_l \leq \frac{1}{12 \beta \sqrt{(1+\sum_{i=1}^n E_i)\big(\sum_{i=1}^n E_i^4\big)}}$. Then, we have:
\begin{align}
    & \min_{0\leq e \leq E-1} \mathbb E[\|\nabla f(\thetav^e)\|^2] \nonumber \\
    &\leq \frac{12}{11E_l^{\texttt{min}} - 7} \bigg( \frac{f(\thetav^0)-f^*}{E \eta_l} +  \Psi_{\sigma} + \Psi_{\textit{p}} \bigg),
\end{align}
where $E_l^{\texttt{min}} = \min_{i} E_i$, i.e., the minimum number of local SGD steps across clients. Also, $\Psi_p$ and $\Psi_{\sigma}$ are two constants controlling the quality of the final model parameter returned by \algname{Robust-HDP}, which are explained in the following. 
\end{restatable}

\vspace{-1em}
\paragraph{Discussion.} Our convergence guarantees are quite general: we allow for partial participation, heterogeneous number of local steps $\{E_i\}$, non-uniform batch sizes $\{b_i\}$, varying and nonuniform aggregation weights $\{w_i^e\}$. When $\{f_i\}$ are convex, \algname{Robsut-HDP} solution converges to a neighborhood of the optimal solution. The term $\Psi_{\sigma}$ decreases when data split across clients is more \emph{i.i.d}, and variance of mini-batch gradients $\{\sigma_{i, \Tilde{g}}^2\}$ decrease (e.g., when clients are less privacy sensitive). Similarly, $\Psi_{p}$ decreases when clients participate more often, and the set of local steps $\{E_i\}$ is more uniform (e.g., clients have similar dataset sizes and batch sizes). Also, smaller local steps $\{E_i\}$, which can be achieved by having smaller local epochs $\{K_i\}$ and larger batch sizes $\{b_i\}$, result in reduction of both $\Psi_p$ and $\Psi_{\sigma}$, and higher quality solutions \cite{malekmohammadi2021operator}. Compared to the results in previous \DPFL works, we have the most general results with more realistic assumptions. For instance, \cite{Liu2021ProjectedFA} (\algname{WeiAvg} and \algname{PFA}) assumes uniform number of local \algname{SGD} updates for all clients, or \cite{DPSCAFFOLD2022} (\algname{DPFedAvg}) assumes uniform aggregation weights and uniform number of local updates. These assumptions may not be practical in real systems. In a more general view, when we have no \DP guarantees, we recover the results for the simple FedAvg algorithm \cite{zhang2023proportional}. When we additionally have $\sigma = 0$ (i.e., FedAvg on \emph{i.i.d} data), our results are the same as those of \algname{SGD} \cite{ghadimi2013stochastic}:
\begin{align}
    \min_{e} \mathbb E[\|\nabla f(\thetav^e)\|^2] &\leq \frac{12}{11E_l^{\texttt{min}} - 7} \frac{f(\thetav^0)-f^*}{E \eta_l} + \mathcal{O}(\eta_l)\nonumber,
\end{align}
which shows convergence rate $\frac{1}{\sqrt{E}}$ with $\eta_l = \mathcal{O}(\frac{1}{\sqrt{E}})$.

\section{Experiments}\label{sec:exps}
See \cref{app:exp_setup} for details of experimental setup and hyperparameter tuning used for evaluation of algorithms.

\subsection{Experimental Setup}
\paragraph{Datasets, models and baseline algorithms:}
We evaluate our proposed method on four benchamrk datasets: MNIST \cite{mnist}, FMNIST \cite{fmnist} and CIFAR10/100 \cite{cifar10} using CNN-based models. Also, we compare four baseline algorithms: 1. \algname{WeiAvg}\cite{Liu2021ProjectedFA} 2. \algname{PFA}\cite{Liu2021ProjectedFA} 3. \algname{DPFedAvg} \cite{DPSCAFFOLD2022} 4. \algname{minimum}  $\epsilon$.


\vspace{-1em}
\paragraph{Privacy preference and batch size heterogeneity:}
We consider an \FL setting with $20$ clients as explained in \cref{appendix:datasets}, which results in homogeneous $\{N_i\}_{i=1}^n$. We also assume full participation and one local epoch for each client ($K_i=1$ for all $i$). Batch size heterogeneity leads to heterogeneity in the number of local steps $\{E_i\}_{i=1}^n$. We sample $\{\epsilon_i\}_{i=1}^n$ from a set of distributions, as shown in \Cref{table:mixture_dists} in the Appendix. We also sample batch sizes $\{b_i\}_{i=1}^n$ uniformly from $\{\texttt{16}, \texttt{32}, \texttt{64}, \texttt{128}\}$. Therefore, we consider heterogeneous $\{\epsilon_i\}_{i=1}^n$, heterogeneous $\{b_i\}_{i=1}^n$ and uniform $\{N_i\}_{i=1}^n$ in this section. We have also considered various other heterogeneity scenarios for clients and more experimental results are reported in \cref{app:additional_exps} and \ref{app:future_dirs}.

\begin{table*}[t]
\centering
\caption{The average \emph{per parameter} noise variance (\cref{eq:sigma_i^2} and \cref{eq:var_g_effective}) normalized by used learning rate ($\frac{\sum_{i=1}^n {w_i^e}^2 \sigma_i^2}{p \eta_l^2}$) in the aggregated model update ($\sum_{i=1}^n w_i^e\Delta \Tilde{\thetav}_i^e$) at the end of first round $(e=1)$ on FMNIST with $E=200$. Due to the projection used in \algname{PFA}, computation of its noise variance was not possible. Results for \algname{Robust-HDP} are shown with std variation across three experiments.}
\begin{tabular}{l|r*{7}{r}r}\toprule
\diagbox{alg}{dist}
&\makebox[2em]{Dist1}
&\makebox[2em]{Dist2}
&\makebox[2em]{Dist3}
&\makebox[2em]{Dist4}
&\makebox[2em]{Dist5}
&\makebox[2em]{Dist6}
&\makebox[2em]{Dist7}
&\makebox[2em]{Dist8}
&\makebox[2em]{Dist9}\\
\midrule \midrule
\algname{WeiAvg} & \footnotesize \tt 1.02 & \footnotesize \tt 1.89 & \footnotesize\tt 0.92& \footnotesize\tt 3.22& \footnotesize\tt 4.58& \footnotesize\tt 28.29& \footnotesize\tt 9.85& \footnotesize\tt 48.15& \footnotesize\tt 34.91\\\midrule

\algname{DPFedAvg} & \footnotesize\tt 1.27 & \footnotesize\tt 16.94& \footnotesize\tt 16.28& \footnotesize\tt 26.87& \footnotesize\tt 25.64& \footnotesize\tt 70.71& \footnotesize\tt 18.50& \footnotesize\tt 85.70&\footnotesize\tt 43.20\\\midrule

\algname{minimum} $\epsilon$ & \footnotesize\tt 4.68 & \footnotesize\tt 103.91& \footnotesize\tt 103.91& \footnotesize\tt 127.18& \footnotesize\tt 103.91& \footnotesize\tt 1868.45& \footnotesize\tt 74.41 & \footnotesize\tt 241.37& 
\footnotesize\tt 87.15\\\midrule
\algname{Robust-HDP} & 
\begin{tabular}{@{}c@{}}
\bf \footnotesize\tt \textbf{0.27} \\ 
{\scriptsize $\pm$ \tt 1.9e-5} 
\end{tabular} & 
\begin{tabular}{@{}c@{}}
\bf \footnotesize\tt \textbf{0.47} \\ 
{\scriptsize $\pm$ \tt 5.7e-5} 
\end{tabular} & 
\begin{tabular}{@{}c@{}}
\bf \footnotesize\tt \textbf{0.07} \\ 
{\scriptsize $\pm$ \tt 2.9e-6} 
\end{tabular} & 
\begin{tabular}{@{}c@{}}
\bf \footnotesize\tt \textbf{0.64} \\ 
{\scriptsize $\pm$ \tt 1.0e-4} 
\end{tabular} & 
\begin{tabular}{@{}c@{}}
\bf \footnotesize\tt \textbf{0.39} \\ 
{\scriptsize $\pm$ \tt 9.3e-6} 
\end{tabular} & 
\begin{tabular}{@{}c@{}}
\bf \footnotesize\tt \textbf{7.62} \\ 
{\scriptsize $\pm$ \tt 1.3e-3} 
\end{tabular} & 
\begin{tabular}{@{}c@{}}
\bf \footnotesize\tt \textbf{2.25} \\ 
{\scriptsize $\pm$ \tt 5.5e-5} 
\end{tabular} & 
\begin{tabular}{@{}c@{}}
\bf \footnotesize\tt \textbf{13.86} \\ 
{\scriptsize $\pm$ \tt 9.9e-4} 
\end{tabular} & 
\begin{tabular}{@{}c@{}}
\bf \footnotesize\tt \textbf{5.95} \\ 
{\scriptsize $\pm$ \tt 2.7e-4} 
\end{tabular}\\\midrule\midrule

Oracle (Eq. \ref{eq:w_opt}) & \footnotesize\tt 0.27 & \footnotesize\tt 0.47& \footnotesize\tt 0.07& \footnotesize\tt 0.64& \footnotesize\tt 0.39 & \footnotesize\tt 7.60& \footnotesize\tt 2.25& \footnotesize\tt 13.81 & \footnotesize\tt 5.93\\\midrule
\end{tabular}
\vspace{-1em}
\label{table:fmnist_noise}
\end{table*}

\begin{figure*}[ht!]
\centering
\includegraphics[width=0.665\columnwidth]{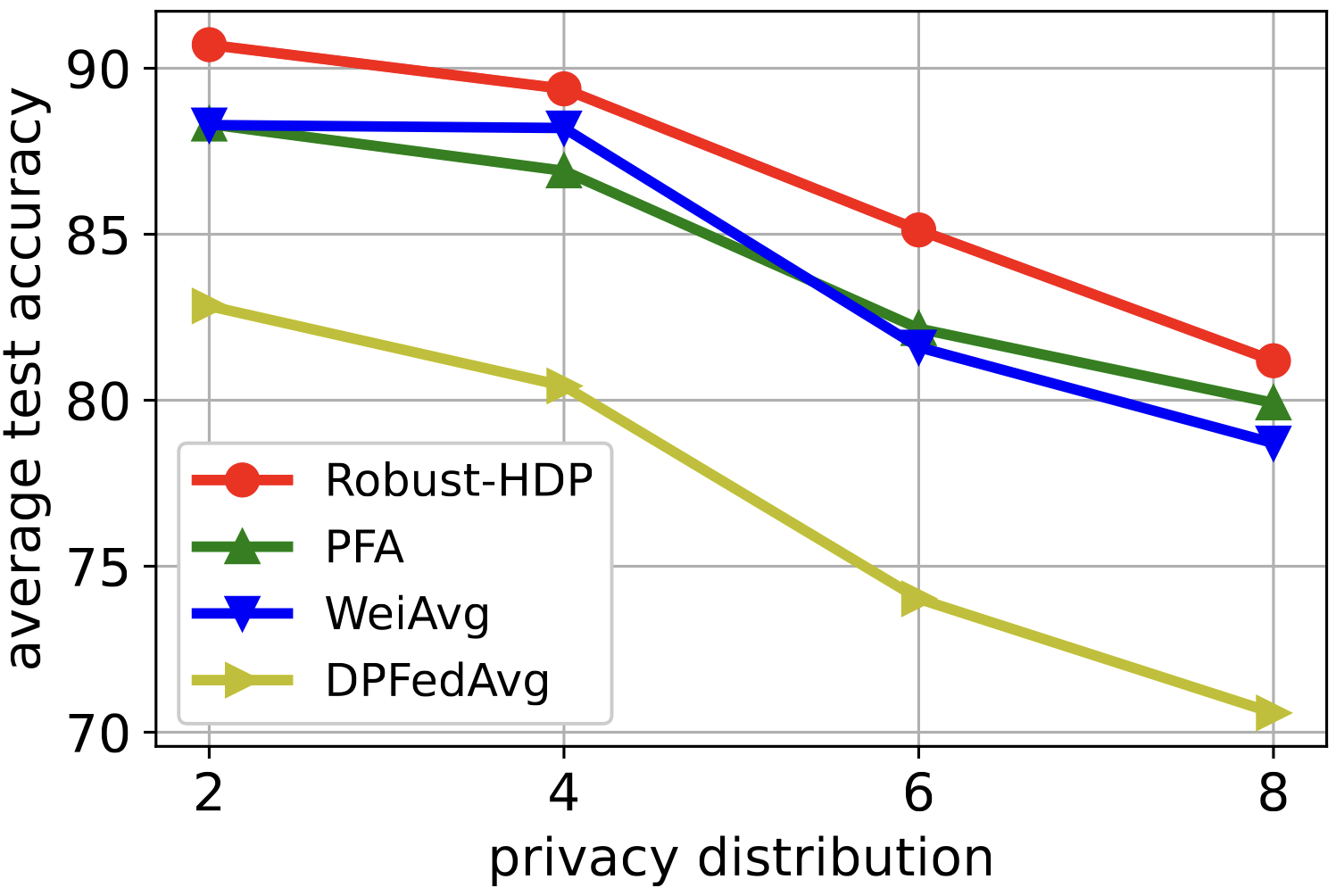}
\includegraphics[width=0.665\columnwidth]{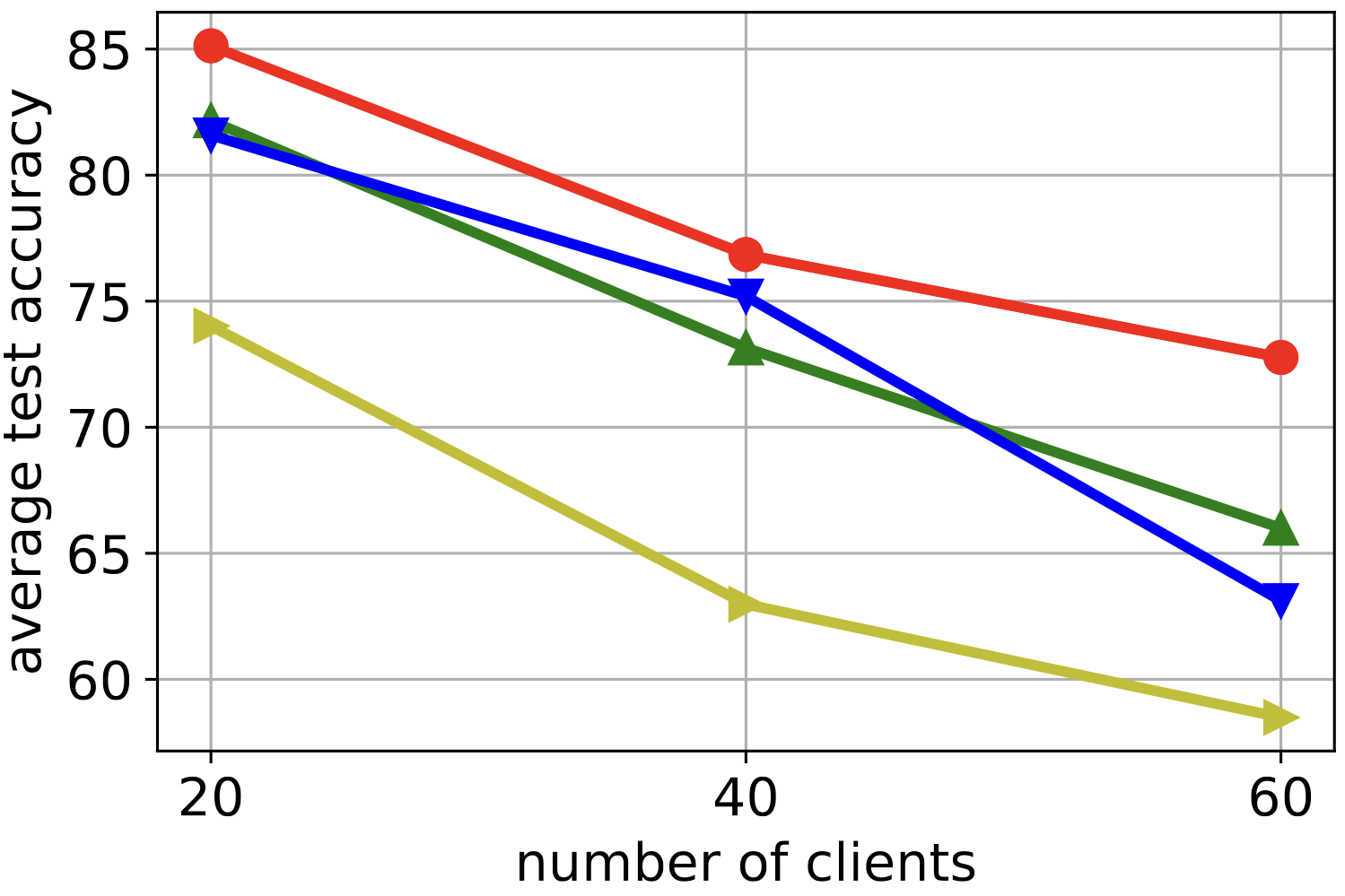}
\includegraphics[width=0.72\columnwidth, height=3.7cm]{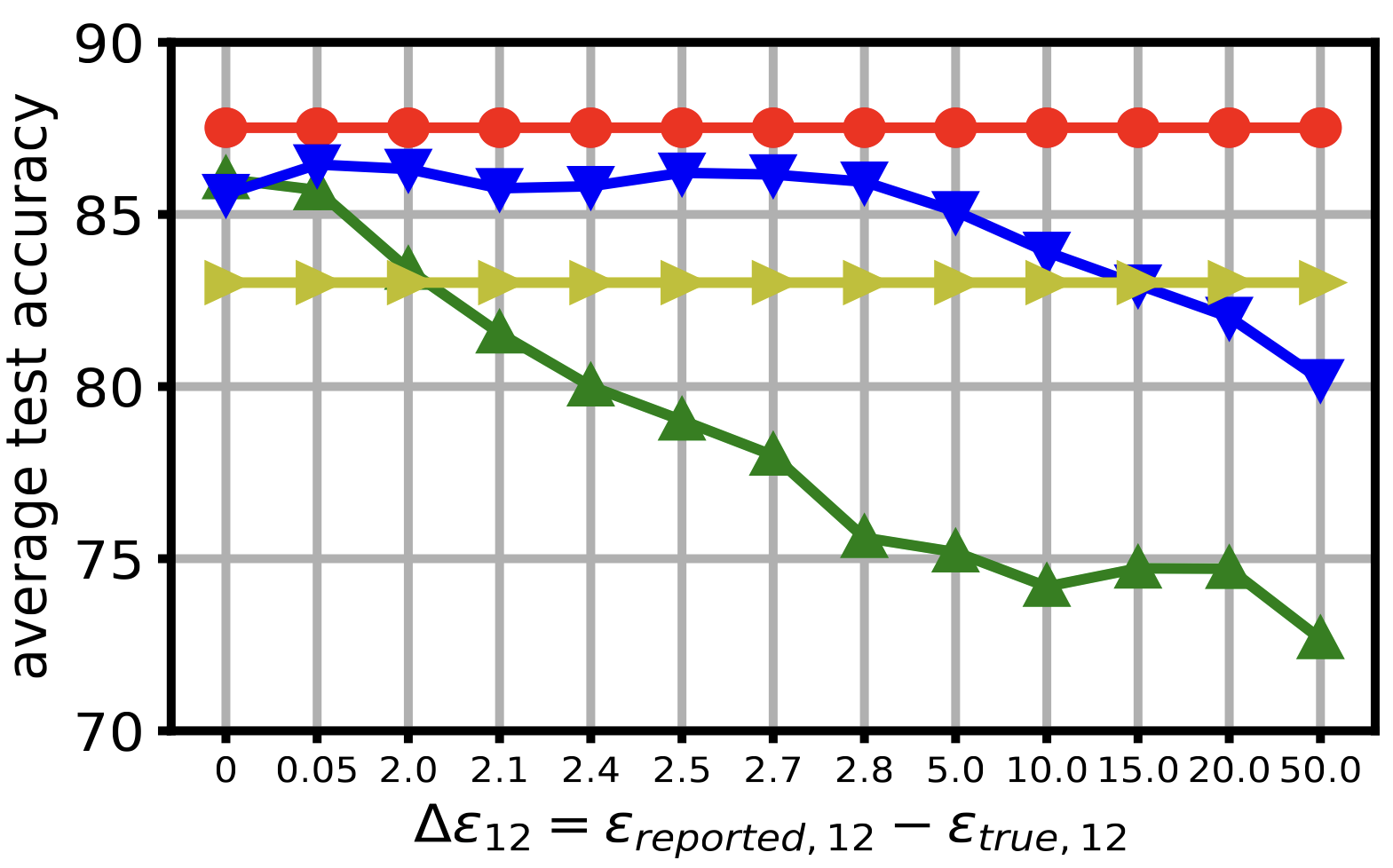}
\vspace{-2em}
\caption{Performance comparison on MNIST. \textbf{Left:} effect of clients desired privacy on utility (detailed results in \cref{table:ablation_privacy_preference}) \textbf{Middle:} effect of number of existing clients (privacy parameters of clients are sampled from Dist6) on utility (detailed results in \cref{table:ablation_num_clients}) \textbf{Right:} Robustness of \algname{}{Robust-HDP} when a random client (client 12 with a moderate $\epsilon$ value of $0.95$) sends falsified version of its $\epsilon$ to the server for \emph{ aggregation} (privacy parameters of other clients are sampled from Dist5). \algname{WeiAvg} and \algname{PFA} are much vulnerable to this falsification.}
\vspace{-1em}
\label{fig:ablation_privacy_preference}
\end{figure*}


\subsection{Experimental Results}
In this section, we investigate five main research questions about \algname{Robust-HDP}, as follows.

\noindent \textbf{RQ1: How do various heterogeneous \DPFL algorithms affect the system utility?} In Fig. \ref{fig:comparison}, we have done a comparison in terms of the average test accuracy across clients. We observe that \algname{Robust-HDP} outperforms the baselines (see tables \ref{table:mnist} to \ref{table:cifar100} in the appendix for detailed results). It achieves higher system utility by using an efficient aggregation strategy, where it assigns smaller weights to the model updates that are indeed more noisy and minimizes the noise level in the aggregation of clients' model updates. The aggregation strategy of \algname{PFA} and \algname{WeiAvg} is sub-optimal, as it can not take the batch size heterogeneity and privacy parameter heterogeneity into account simultaneously.

\noindent \textbf{RQ2: How does \algname{Robust-HDP} improve convergence speed during training?} 
We have also compared different algorithms based on their convergence speed in \Cref{fig:conv_speed_mnist_dist6}. While the baseline algorithms suffer from high levels of noise in the aggregated model update $\sum_{i \in \mathcal{S}^e} w_i^e \Delta \Tilde{\thetav}_i^e$ (see \Cref{table:fmnist_noise}), \algname{Robust-HDP} enjoys its efficient noise minimization, which performs very close to the optimum aggregation strategy, and not only results in faster convergence but also improves utility. In contrast, based on our experiments, the baseline algorithms have to use smaller learning rates to avoid divergence of their training optimization. Note that fast convergence of \DPFL algorithms is indeed important, as the privacy budgets of participating clients does not let the server to run the federated training for more rounds.

\noindent \textbf{RQ3: Is \algname{Robust-HDP} indeed Robust?} In Fig. \ref{fig:ablation_privacy_preference}, we compare \algname{Robust-HDP} with others based on clients' desired privacy level and number of clients. As clients become more privacy sensitive, they send more noisy updates to the server, making convergence to better solutions harder. \algname{Robust-HDP} shows the highest robustness to the larger noise in clients' updates and achieves the highest utility, especially in more privacy sensitive scenarios, e.g., Dist8. Also, we observe that it achieves the highest system utility when the number of clients in the system increases. Furthermore, it is completely safe in scenarios that some clients report a falsified privacy parameter to the server (\Cref{fig:ablation_privacy_preference}, right).

\begin{figure}[hbt!]
\centering
\includegraphics[width=0.95\columnwidth]{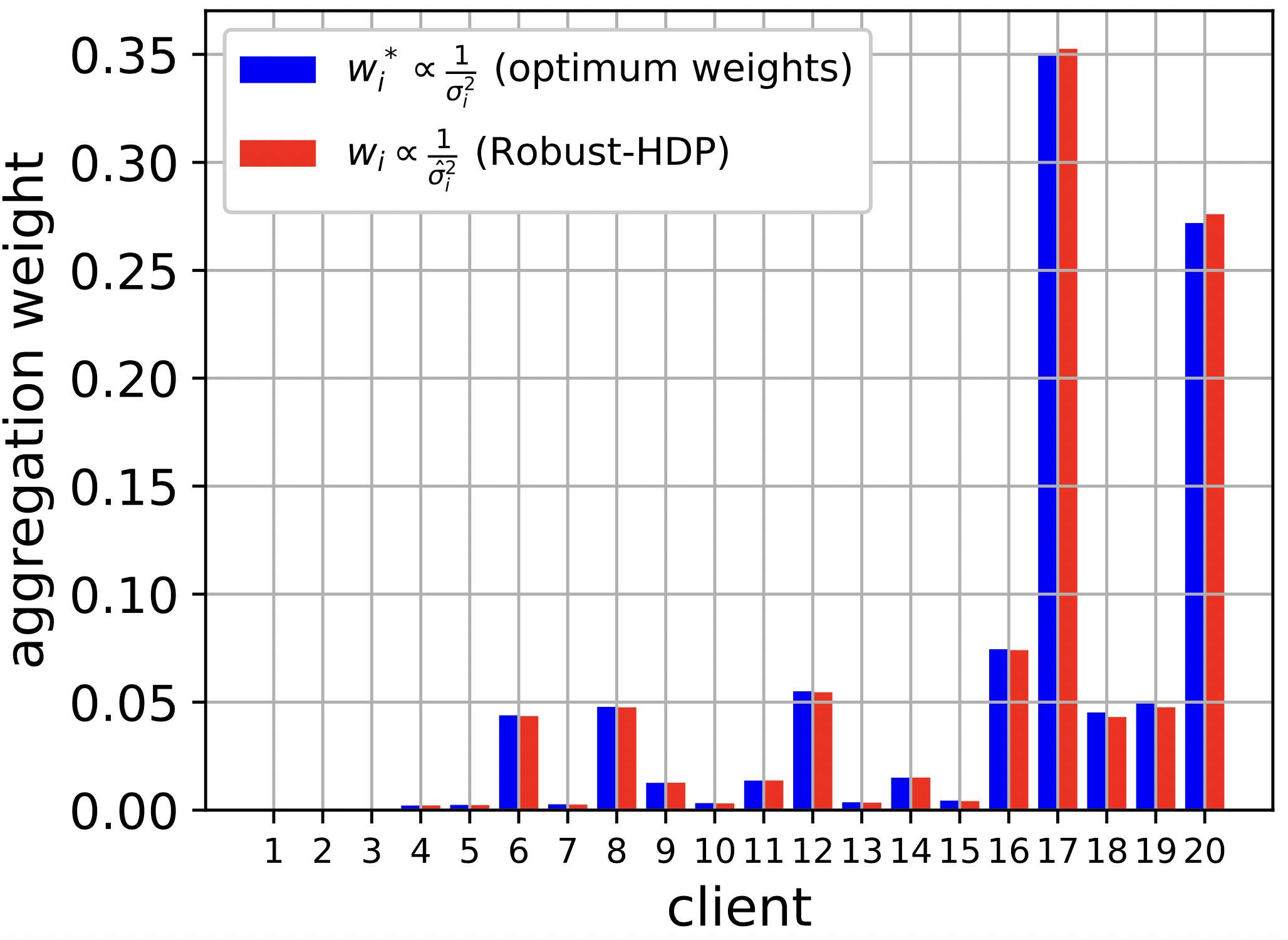}
    \vspace{-1em}
    \caption{Precision of \algname{Robust-HDP} (red) compared to oracle optimum strategy (blue) for CIFAR10 and Dist2, when using the approximation method in \Cref{sec:scalability} with $Q=1$ and $p'=2\times10^5$.}\label{fig:weights_comparison_cifar10}
\end{figure}

\noindent \textbf{RQ4: How accurate \algname{Robust-HDP} is in estimating $\{w_i^*\}$?} \Cref{fig:weights_comparison_cifar10} compares the weight assignment of \algname{Robust-HDP} with the optimum assignment (computed from Equations \ref{eq:w_opt}) for CIFAR10 dataset and Dist2. As the model used for CIFAR10 is relatively large (with $p\approx 11\times 10^6$), we have used the approximation method in \cref{sec:scalability} (with $Q=1$ and $p'=2\times10^5$). \Cref{fig:weights_comparison_cifar10} has sorted clients based on their privacy parameter $\epsilon$ in ascending order. \algname{WeiAvg} and \algname{PFA} assign smaller weights to more privacy sensitive clients, while \algname{Robust-HDP} assigns smaller weights to the clients with less noisy model updates.

\begin{table*}[t]
\centering
\caption{Comparison of different algorithms (on MNIST, $E=200$) with \textbf{heterogeneous} data split (maximum 8 labels per client) and \textbf{60 clients} in the system all using \textbf{uniform batch size} $\mathbf{128}$. }
\label{table:mnist_8labels_60clients_uniformbatch}
\begin{tabular}{l|*{8}{c}c}\toprule
\diagbox{alg}{distr}
&\makebox[2.5em]{\footnotesize Dist1}
&\makebox[2.5em]{\footnotesize Dist2}
&\makebox[2.5em]{\footnotesize Dist3}
&\makebox[2.5em]{\footnotesize Dist4}
&\makebox[2.5em]{\footnotesize Dist5}
&\makebox[2.5em]{\footnotesize Dist6}
&\makebox[2.5em]{\footnotesize Dist7}
&\makebox[2.5em]{\footnotesize Dist8}
&\makebox[2.5em]{\footnotesize Dist9}\\
\midrule\midrule
\algname{WeiAvg} \cite{Liu2021ProjectedFA} & 81.14 & 81.21& \bf 84.44 &  71.64 & 81.45 & 72.27 & 80.55& 71.28 & 72.07\\\midrule
\footnotesize \algname{PFA}\cite{Liu2021ProjectedFA} & 81.29  & 77.40 & 80.65 & 74.45 & 81.89 & 64.48 & \bf 81.40 & \bf 73.39 & 72.97 \\\midrule
\algname{DPFedAvg} \cite{DPSCAFFOLD2022} & 82.61 & 74.32 &83.12 & 70.8& 81.34& 66.51 & 76.10 & 65.16& 73.03\\\midrule
\algname{Robust-HDP} & \bf 84.84 & \bf 82.78& 80.78& \bf 78.91& \bf 81.66& \bf 72.30& 79.32& 70.68 & \bf 73.82\\
\bottomrule
\end{tabular}
\end{table*}

We have also studied the effect of parameter $p'$, on the precision of the aggregation weights returned by \algname{Robust-HDP}. In \Cref{fig:ablation_p_prime} and for CIFAR10, we have shown the increasing precision of the weights returned by \algname{Robust-HDP} when $p'$ grows. The larger $p'$ gets, the more samples we have for estimating the noise variance in clients' model updates, hence more precise weight assignments. As explained in \Cref{sec:scalability}, when $p$ is already large, we also avoid using too large values for $p'$, as the main point of \Cref{sec:scalability} was to feed a matrix with smaller number of rows to RPCA to avoid its low precision and high computation time when the number of rows ($p$) in the original input matrix $\mathbf{M}$ is large.

\begin{figure}[t]
\centering
\includegraphics[width=0.95\columnwidth]{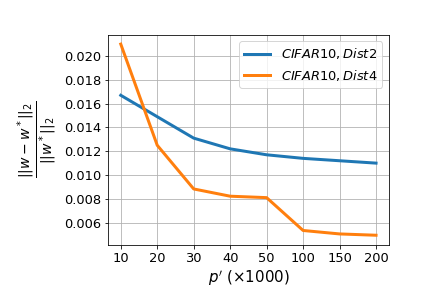}
\caption{Effect of the parameter $p'$, used in the approximation method explained in \cref{sec:scalability}, on the precision of the weights returned by \algname{Robust-HDP} for CIFAR10 (with $p\approx 11\times 10^6$).}
\label{fig:ablation_p_prime}
\end{figure}

\noindent \textbf{RQ5: What is the effect of data heterogeneity across clients on the performance of \algname{Robust-HDP}?} 

So far we assumed an \emph{i.i.d} data distribution across clients. What if the data distribution is moderately/highly heterogeneous? Assuming full participation of clients in round $e$, in order to have a useful RPCA decomposition $\mathbf{M} = [\Delta \Tilde{\thetav}_1^e|\ldots|\Delta \Tilde{\thetav}_n^e] = \mathbf{L} + \mathbf{S}$ at the end of the round, two conditions should be met \cite{Candes2009RobustPC}: 1. There should be an underlying low-rank matrix $\mathbf{L}$ in $\mathbf{M}$ 2. The difference between the matrix $\mathbf{L}$ and $\mathbf{M}$, i.e., the noise matrix $\mathbf{S}$, should be (approximately) sparse. 

Whether the first condition is met or not mainly depends on how much heterogeneous the data split across clients is. Note that $\textit{rank}(\mathbf{L})$ should be low, and not necessarily close to 1. If we assume that the second condition is met, it was shown in Theorem 1.1 in \cite{Candes2009RobustPC} that the decomposition is guaranteed to work even if $\textit{rank}(\mathbf{L})\in \mathcal{O}(n/(\log p)^2)$, i.e., the rank of $\mathbf{L}$ grows almost linearly in $n$. \emph{Therefore, even if the data split across clients is moderately heterogeneous, we expect \algname{Robust-HDP} to be successful in at least the decomposition task and the following noise estimation, given that the noise matrix $\mathbf{S}$ is sparse, and there are large enough number of clients}.

Whether the second condition is met or not, mainly depends on how much variation exists in the amount of noise in clients' model updates, i.e., how (approximately) sparse the set $\{\sigma_1^2, \cdots, \sigma_n^2\}$ is. As shown in Equations \ref{eq:var_g_effective}, \ref{var_g_ineffective} and \ref{eq:sigma_i^2}, this mainly depends on clients' privacy parameters ($\{(\epsilon_i, \delta_i)\}_{i=1}^n$), and batch sizes ($\{b_i\}_{i=1}^n$), \emph{and is independent of whether the data split is i.i.d or not.} The more the variation in clients' privacy parameters/batch sizes (similar to what we saw in \cref{fig:var_epsilon_b}), the better we can consider $\mathbf{S}$ as an approximately sparse matrix, which validates our RPCA decomposition.

So far, we assumed an \emph{i.i.d} data distribution across clients, which ensures that the underlying matrix $\mathbf{L}$ is indeed low-rank.  
Also, we assumed heterogeneity in batch size and privacy parameters of clients, which led to a sparse pattern in the noise matrix $\mathbf{S}$ (as shown in \cref{fig:var_epsilon_b}, right). In order to evaluate \algname{Robust-HDP} when the data split is moderately heterogeneous, we run experiments on MNIST with 60 clients in total (compared to the $20$ clients before) and uniform batch size $b=128$, and we split data such that each client holds data samples of at maximum 8 classes. The results obtained are reported in \cref{table:mnist_8labels_60clients_uniformbatch}. As observed, \algname{Robust-HDP} still outperforms the baselines in most of the cases. However, compared to the detailed results in \cref{table:mnist}, which were obtained for \emph{i.i.d} data split, its superiority to the baseline algorithms has decreased. Detailed discussion of these results along with scenarios with highly heterogeneous data splits are reported in \Cref{app:future_dirs}.

\vspace{-1em}
\section{Conclusion}
In heterogeneous \DPFL systems, heterogeneity in privacy preference, batch/dataset size results in large variations across the noise levels in clients' model updates, which existing algorithms can not fully take into account. To address this heterogeneity, we proposed a robust heterogeneous  \DPFL algorithm that performs  noise-aware aggregation on an untrusted server, and is independent of clients' privacy parameter values shared with the server. The proposed algorithm is optimal, robust, vastly applicable, scalable, and improves utility and convergence speed.

\section*{Impact Statement}
This paper presents work whose goal is to advance the field of Machine Learning. There are many potential societal consequences of our work, none which we feel must be specifically highlighted here.

\section*{Acknowledgements}
We thank the reviewers and the area chair for the critical comments that have largely improved the final version of this paper.
YY gratefully acknowledges funding support from NSERC, the Ontario early researcher program and the Canada CIFAR AI Chairs program. 
Resources used in preparing this research were provided, in part, by the Province of Ontario, the Government of Canada through CIFAR, and companies sponsoring the Vector Institute. Also, YC acknowledges the support by JSPS KAKENHI JP22H03595, JST PRESTO JPMJPR23P5, JST CREST JPMJCR21M2.


\bibliography{example_paper}

\begin{thebibliography}{53}
\providecommand{\natexlab}[1]{#1}
\providecommand{\url}[1]{\texttt{#1}}
\expandafter\ifx\csname urlstyle\endcsname\relax
  \providecommand{\doi}[1]{doi: #1}\else
  \providecommand{\doi}{doi: \begingroup \urlstyle{rm}\Url}\fi

\bibitem[Abadi et~al.(2016)Abadi, Chu, Goodfellow, McMahan, Mironov, Talwar, and Zhang]{Abadi2016}
Abadi, M., Chu, A., Goodfellow, I., McMahan, H.~B., Mironov, I., Talwar, K., and Zhang, L.
\newblock Deep learning with differential privacy.
\newblock In \emph{Proceedings of the 2016 ACM SIGSAC Conference on Computer and Communications Security}, 2016.
\newblock URL \url{https://doi.org/10.1145/2976749.2978318}.

\bibitem[Alaggan et~al.(2017)Alaggan, Gambs, and Kermarrec]{Alaggan_Gambs_Kermarrec_2017}
Alaggan, M., Gambs, S., and Kermarrec, A.-M.
\newblock Heterogeneous differential privacy.
\newblock \emph{Journal of Privacy and Confidentiality}, 2017.
\newblock URL \url{https://journalprivacyconfidentiality.org/index.php/jpc/article/view/652}.

\bibitem[Bagdasaryan \& Shmatikov(2019)Bagdasaryan and Shmatikov]{bagdasaryan2019differential}
Bagdasaryan, E. and Shmatikov, V.
\newblock Differential privacy has disparate impact on model accuracy.
\newblock In \emph{Advances in Neural Information Processing Systems}, 2019.
\newblock URL \url{https://proceedings.neurips.cc/paper/2019/file/fc0de4e0396fff257ea362983c2dda5a-Paper.pdf}.

\bibitem[Boenisch et~al.(2023)Boenisch, M\"{u}hl, Dziedzic, Rinberg, and Papernot]{Boenisch_NEURIPS2023}
Boenisch, F., M\"{u}hl, C., Dziedzic, A., Rinberg, R., and Papernot, N.
\newblock Have it your way: Individualized privacy assignment for {DP-SGD}.
\newblock In \emph{Advances in Neural Information Processing Systems}, 2023.
\newblock URL \url{https://neurips.cc/virtual/2023/poster/71354}.

\bibitem[Candes et~al.(2009)Candes, Li, Ma, and Wright]{Candes2009RobustPC}
Candes, E.~J., Li, X., Ma, Y., and Wright, J.
\newblock Robust principal component analysis?, 2009.
\newblock URL \url{https://arxiv.org/pdf/0912.3599}.

\bibitem[Chathoth et~al.(2022)Chathoth, Necciai, Jagannatha, and Lee]{Chathoth2022cohortDP}
Chathoth, A.~K., Necciai, C.~P., Jagannatha, A., and Lee, S.
\newblock Differentially private federated continual learning with heterogeneous cohort privacy.
\newblock In \emph{2022 IEEE International Conference on Big Data (Big Data)}, 2022.
\newblock URL \url{https://ieeexplore.ieee.org/document/10021082}.

\bibitem[Cummings et~al.(2019)Cummings, Gupta, Kimpara, and Morgenstern]{10.1145/3314183.3323847}
Cummings, R., Gupta, V., Kimpara, D., and Morgenstern, J.
\newblock On the compatibility of privacy and fairness.
\newblock In \emph{Adjunct Publication of the 27th Conference on User Modeling, Adaptation and Personalization}, 2019.
\newblock URL \url{https://doi.org/10.1145/3314183.3323847}.

\bibitem[Deng(2012)]{mnist}
Deng, L.
\newblock The {MNIST} database of handwritten digit images for machine learning research.
\newblock \emph{IEEE Signal Processing Magazine}, 2012.
\newblock URL \url{https://ieeexplore.ieee.org/document/6296535}.

\bibitem[Donahue \& Kleinberg(2021)Donahue and Kleinberg]{modelsharinggames}
Donahue, K. and Kleinberg, J.
\newblock Model-sharing games: Analyzing federated learning under voluntary participation.
\newblock \emph{Proceedings of the AAAI Conference on Artificial Intelligence}, 2021.
\newblock URL \url{https://ojs.aaai.org/index.php/AAAI/article/view/16669}.

\bibitem[Dwork(2011)]{Dwork2011AFF}
Dwork, C.
\newblock A firm foundation for private data analysis.
\newblock \emph{Commun. ACM}, 2011.
\newblock URL \url{https://doi.org/10.1145/1866739.1866758}.

\bibitem[Dwork \& Roth(2014)Dwork and Roth]{Dwork2014TheAF}
Dwork, C. and Roth, A.
\newblock The algorithmic foundations of differential privacy.
\newblock \emph{Found. Trends Theor. Comput. Sci.}, 2014.
\newblock URL \url{https://dl.acm.org/doi/10.1561/0400000042}.

\bibitem[Dwork et~al.(2006{\natexlab{a}})Dwork, Kenthapadi, McSherry, Mironov, and Naor]{Dwork2006OurDO}
Dwork, C., Kenthapadi, K., McSherry, F., Mironov, I., and Naor, M.
\newblock Our data, ourselves: Privacy via distributed noise generation.
\newblock In \emph{Proceedings of the 24th Annual International Conference on the Theory and Applications of Cryptographic Techniques}, 2006{\natexlab{a}}.
\newblock URL \url{https://doi.org/10.1007/11761679_29}.

\bibitem[Dwork et~al.(2006{\natexlab{b}})Dwork, McSherry, Nissim, and Smith]{Dwork2006}
Dwork, C., McSherry, F., Nissim, K., and Smith, A.
\newblock Calibrating noise to sensitivity in private data analysis.
\newblock In \emph{Proceedings of the Third Conference on Theory of Cryptography}. Springer-Verlag, 2006{\natexlab{b}}.
\newblock URL \url{https://doi.org/10.1007/11681878_14}.

\bibitem[Fallah et~al.(2023)Fallah, Makhdoumi, Malekian, and Ozdaglar]{fallah2023optimal}
Fallah, A., Makhdoumi, A., Malekian, A., and Ozdaglar, A.
\newblock Optimal and differentially private data acquisition: Central and local mechanisms.
\newblock \emph{Operations Research}, 2023.
\newblock URL \url{https://arxiv.org/pdf/2201.03968.pdf}.

\bibitem[Fioretto et~al.(2022)Fioretto, Tran, Hentenryck, and Zhu]{Fioretto_2022}
Fioretto, F., Tran, C., Hentenryck, P.~V., and Zhu, K.
\newblock Differential privacy and fairness in decisions and learning tasks: A survey.
\newblock In \emph{Proceedings of the Thirty-First International Joint Conference on Artificial Intelligence}, 2022.
\newblock URL \url{https://doi.org/10.24963%2Fijcai.2022%2F766}.

\bibitem[Geiping et~al.(2020)Geiping, Bauermeister, Dr{\"o}ge, and Moeller]{Geiping2020InvertingG}
Geiping, J., Bauermeister, H., Dr{\"o}ge, H., and Moeller, M.
\newblock Inverting gradients - how easy is it to break privacy in federated learning?
\newblock \emph{ArXiv}, 2020.
\newblock URL \url{https://proceedings.neurips.cc/paper/2020/file/c4ede56bbd98819ae6112b20ac6bf145-Paper.pdf}.

\bibitem[Geyer et~al.(2017)Geyer, Klein, and Nabi]{Geyer2017DPFedAvg}
Geyer, R.~C., Klein, T., and Nabi, M.
\newblock Differentially private federated learning: A client level perspective.
\newblock \emph{ArXiv}, 2017.
\newblock URL \url{https://arxiv.org/pdf/1712.07557.pdf}.

\bibitem[Ghadimi \& Lan(2013)Ghadimi and Lan]{ghadimi2013stochastic}
Ghadimi, S. and Lan, G.
\newblock Stochastic first- and zeroth-order methods for nonconvex stochastic programming.
\newblock \emph{SIAM Journal on Optimization}, 2013.
\newblock URL \url{https://doi.org/10.1137/120880811}.

\bibitem[Girgis et~al.(2021)Girgis, Data, Diggavi, Kairouz, and Suresh]{Girgis2021}
Girgis, A.~M., Data, D., Diggavi, S.~N., Kairouz, P., and Suresh, A.~T.
\newblock Shuffled model of differential privacy in federated learning.
\newblock In \emph{AISTATS}, 2021.
\newblock URL \url{https://proceedings.mlr.press/v130/girgis21a.html}.

\bibitem[Gur-Ari et~al.(2018)Gur-Ari, Roberts, and Dyer]{GurAri2018GradientDH}
Gur-Ari, G., Roberts, D.~A., and Dyer, E.
\newblock Gradient descent happens in a tiny subspace.
\newblock \emph{ArXiv}, 2018.
\newblock URL \url{https://arxiv.org/pdf/1812.04754}.

\bibitem[He et~al.(2015)He, Zhang, Ren, and Sun]{resnet}
He, K., Zhang, X., Ren, S., and Sun, J.
\newblock Deep residual learning for image recognition.
\newblock In \emph{IEEE Conference on Computer Vision and Pattern Recognition (CVPR)}, 2015.
\newblock URL \url{https://www.cv-foundation.org/openaccess/content_cvpr_2016/papers/He_Deep_Residual_Learning_CVPR_2016_paper.pdf}.

\bibitem[Heo et~al.(2023)Heo, Seo, and Whang]{Heo2023PersonalizedDU}
Heo, G., Seo, J., and Whang, S.~E.
\newblock Personalized {DP-SGD} using sampling mechanisms.
\newblock \emph{ArXiv}, 2023.
\newblock URL \url{https://arxiv.org/pdf/2305.15165}.

\bibitem[Hitaj et~al.(2017)Hitaj, Ateniese, and P{\'e}rez-Cruz]{Hitaj2017DeepMU}
Hitaj, B., Ateniese, G., and P{\'e}rez-Cruz, F.
\newblock Deep models under the {GAN}: Information leakage from collaborative deep learning.
\newblock \emph{Proceedings of the 2017 ACM SIGSAC Conference on Computer and Communications Security}, 2017.
\newblock URL \url{https://arxiv.org/pdf/1702.07464}.

\bibitem[Huang et~al.(2020)Huang, Zhou, Zhu, Liao, Wu, and Qiu]{Huang2020ImprovingLM}
Huang, W., Zhou, S., Zhu, T., Liao, Y., Wu, C., and Qiu, S.
\newblock Improving laplace mechanism of differential privacy by personalized sampling.
\newblock In \emph{2020 IEEE 19th International Conference on Trust, Security and Privacy in Computing and Communications (TrustCom)}, 2020.
\newblock URL \url{https://ieeexplore.ieee.org/document/9343130}.

\bibitem[Jorgensen et~al.(2015)Jorgensen, Yu, and Cormode]{Jorgensen2015ConservativeOL}
Jorgensen, Z., Yu, T., and Cormode, G.
\newblock Conservative or liberal? personalized differential privacy.
\newblock In \emph{2015 IEEE 31st International Conference on Data Engineering}, 2015.
\newblock URL \url{https://ieeexplore.ieee.org/document/7113353}.

\bibitem[Kang et~al.(2023)Kang, Pedarsani, and Ramchandran]{fairvalueofdata}
Kang, J., Pedarsani, R., and Ramchandran, K.
\newblock The fair value of data under heterogeneous privacy constraints.
\newblock \emph{ArXiv}, 2023.
\newblock URL \url{https://arxiv.org/pdf/2301.13336}.

\bibitem[Karimireddy et~al.(2022)Karimireddy, Guo, and Jordan]{karimreddy_datasharing}
Karimireddy, S.~P., Guo, W., and Jordan, M.~I.
\newblock Mechanisms that incentivize data sharing in federated learning.
\newblock \emph{ArXiv}, 2022.
\newblock URL \url{https://arxiv.org/abs/2207.04557}.

\bibitem[Kotsogiannis et~al.(2020)Kotsogiannis, Doudalis, Haney, Machanavajjhala, and Mehrotra]{Doudalis2017OnesidedDP}
Kotsogiannis, I., Doudalis, S., Haney, S., Machanavajjhala, A., and Mehrotra, S.
\newblock One-sided differential privacy.
\newblock In \emph{2020 IEEE 36th International Conference on Data Engineering (ICDE)}, 2020.
\newblock URL \url{https://ieeexplore.ieee.org/document/9101725}.

\bibitem[Krizhevsky(2009)]{cifar10}
Krizhevsky, A.
\newblock Learning multiple layers of features from tiny images, 2009.
\newblock URL \url{https://www.cs.toronto.edu/~kriz/learning-features-2009-TR.pdf}.

\bibitem[Li et~al.(2020)Li, Hu, Beirami, and Smith]{Li2020DittoFA}
Li, T., Hu, S., Beirami, A., and Smith, V.
\newblock Ditto: Fair and robust federated learning through personalization.
\newblock In \emph{International Conference on Machine Learning}, 2020.
\newblock URL \url{http://proceedings.mlr.press/v139/li21h/li21h.pdf}.

\bibitem[Liu et~al.(2021{\natexlab{a}})Liu, Lou, Xiong, Liu, and Meng]{Liu2021ProjectedFA}
Liu, J., Lou, J., Xiong, L., Liu, J., and Meng, X.
\newblock Projected federated averaging with heterogeneous differential privacy.
\newblock \emph{Proceedings of VLDB Endowment.}, 2021{\natexlab{a}}.
\newblock URL \url{https://www.vldb.org/pvldb/vol15/p828-liu.pdf}.

\bibitem[Liu et~al.(2021{\natexlab{b}})Liu, Cao, Chen, Guo, and Yoshikawa]{Liu2021FLAMEDP}
Liu, R., Cao, Y., Chen, H., Guo, R., and Yoshikawa, M.
\newblock {FLAME}: Differentially private federated learning in the shuffle model.
\newblock In \emph{AAAI}, 2021{\natexlab{b}}.
\newblock URL \url{https://www.aaai.org/AAAI21Papers/AAAI-4838.LiuR.pdf}.

\bibitem[Malekmohammadi et~al.(2023)Malekmohammadi, Shaloudegi, Hu, and Yu]{malekmohammadi2021operator}
Malekmohammadi, S., Shaloudegi, K., Hu, Z., and Yu, Y.
\newblock \emph{A Unifying Framework for Federated Learning}, pp.\  87--115.
\newblock Springer International Publishing, Cham, 2023.
\newblock ISBN 978-3-031-11748-0.
\newblock \doi{10.1007/978-3-031-11748-0_5}.
\newblock URL \url{https://doi.org/10.1007/978-3-031-11748-0_5}.

\bibitem[Malekmohammadi et~al.(2024)Malekmohammadi, Taik, and Farnadi]{malekmohammadi2024mitigating}
Malekmohammadi, S., Taik, A., and Farnadi, G.
\newblock Mitigating disparate impact of differential privacy in federated learning through robust clustering, 2024.
\newblock URL \url{https://arxiv.org/abs/2405.19272}.

\bibitem[Matzken et~al.(2023)Matzken, Eger, and Habernal]{matzken2023tradeoffs}
Matzken, C., Eger, S., and Habernal, I.
\newblock {Trade-Offs Between Fairness and Privacy in Language Modeling}.
\newblock In \emph{Findings of the Association for Computational Linguistics: ACL}, 2023.

\bibitem[McMahan et~al.(2017)McMahan, Moore, Ramage, Hampson, and y~Arcas]{McMahanMRHA17}
McMahan, B., Moore, E., Ramage, D., Hampson, S., and y~Arcas, B.~A.
\newblock Communication-efficient learning of deep networks from decentralized data.
\newblock In \emph{{AISTATS}}, 2017.
\newblock URL \url{http://proceedings.mlr.press/v54/mcmahan17a/mcmahan17a.pdf}.

\bibitem[McMahan et~al.(2018)McMahan, Ramage, Talwar, and Zhang]{McMahan2018LearningDP}
McMahan, H.~B., Ramage, D., Talwar, K., and Zhang, L.
\newblock Learning differentially private recurrent language models.
\newblock In \emph{ICLR}, 2018.
\newblock URL \url{https://arxiv.org/pdf/1710.06963.pdf}.

\bibitem[Niu et~al.(2020)Niu, Chen, Wang, Cao, and Li]{Niu2020UtilityawareEM}
Niu, B., Chen, Y., Wang, B., Cao, J., and Li, F.
\newblock Utility-aware exponential mechanism for personalized differential privacy.
\newblock In \emph{2020 IEEE Wireless Communications and Networking Conference (WCNC)}, 2020.
\newblock URL \url{https://ieeexplore.ieee.org/stamp/stamp.jsp?tp=&arnumber=9120532}.

\bibitem[Noble et~al.(2021)Noble, Bellet, and Dieuleveut]{DPSCAFFOLD2022}
Noble, M., Bellet, A., and Dieuleveut, A.
\newblock Differentially private federated learning on heterogeneous data.
\newblock In \emph{International Conference on Artificial Intelligence and Statistics}, 2021.
\newblock URL \url{https://proceedings.mlr.press/v151/noble22a/noble22a.pdf}.

\bibitem[Rigaki \& Garc{\'i}a(2020)Rigaki and Garc{\'i}a]{Rigaki2020ASO}
Rigaki, M. and Garc{\'i}a, S.
\newblock A survey of privacy attacks in machine learning.
\newblock \emph{ArXiv}, 2020.
\newblock URL \url{https://arxiv.org/pdf/2007.07646}.

\bibitem[Sattler et~al.(2019)Sattler, M{\"u}ller, and Samek]{Sattler2019ClusteredFL}
Sattler, F., M{\"u}ller, K.-R., and Samek, W.
\newblock Clustered federated learning: Model-agnostic distributed multitask optimization under privacy constraints.
\newblock \emph{IEEE Transactions on Neural Networks and Learning Systems}, 2019.
\newblock URL \url{https://ieeexplore.ieee.org/stamp/stamp.jsp?arnumber=9174890}.

\bibitem[Shi et~al.(2021)Shi, Cui, Li, Jia, and Yu]{Shi2021SelectiveDP}
Shi, W., Cui, A., Li, E., Jia, R., and Yu, Z.
\newblock Selective differential privacy for language modeling.
\newblock \emph{ArXiv}, 2021.
\newblock URL \url{https://arxiv.org/pdf/2108.12944.pdf}.

\bibitem[Wang et~al.(2019{\natexlab{a}})Wang, Yurochkin, Sun, Papailiopoulos, and Khazaeni]{wang2019federated}
Wang, H., Yurochkin, M., Sun, Y., Papailiopoulos, D., and Khazaeni, Y.
\newblock Federated learning with matched averaging.
\newblock In \emph{International Conference on Learning Representations}, 2019{\natexlab{a}}.
\newblock URL \url{https://arxiv.org/pdf/2002.06440}.

\bibitem[Wang et~al.(2019{\natexlab{b}})Wang, Song, Zhang, Song, Wang, and Qi]{Wang2018BeyondIC}
Wang, Z., Song, M., Zhang, Z., Song, Y., Wang, Q., and Qi, H.
\newblock Beyond inferring class representatives: User-level privacy leakage from federated learning.
\newblock \emph{IEEE INFOCOM}, 2019{\natexlab{b}}.
\newblock URL \url{https://arxiv.org/pdf/1812.00535}.

\bibitem[Werner et~al.(2023)Werner, He, Jordan, Jaggi, and Karimireddy]{Werner2023ProvablyPA}
Werner, M., He, L., Jordan, M., Jaggi, M., and Karimireddy, S.~P.
\newblock Provably personalized and robust federated learning.
\newblock \emph{Transactions on Machine Learning Research}, 2023.
\newblock URL \url{https://openreview.net/pdf?id=B0uBSSUy0G}.

\bibitem[Xiao et~al.(2017)Xiao, Rasul, and Vollgraf]{fmnist}
Xiao, H., Rasul, K., and Vollgraf, R.
\newblock {Fashion-MNIST}: a novel image dataset for benchmarking machine learning algorithms.
\newblock \emph{CoRR}, 2017.
\newblock URL \url{http://arxiv.org/abs/1708.07747}.

\bibitem[Yu et~al.(2023)Yu, Kamath, Kulkarni, Liu, Yin, and Zhang]{yu2023individual}
Yu, D., Kamath, G., Kulkarni, J., Liu, T.-Y., Yin, J., and Zhang, H.
\newblock Individual privacy accounting for differentially private stochastic gradient descent.
\newblock \emph{ArXiv}, 2023.
\newblock URL \url{https://arxiv.org/abs/2206.02617}.

\bibitem[Yu et~al.(2019)Yu, Liu, Pu, Gursoy, and Truex]{Yu2019DifferentiallyPM}
Yu, L., Liu, L., Pu, C., Gursoy, M.~E., and Truex, S.
\newblock Differentially private model publishing for deep learning.
\newblock \emph{2019 IEEE Symposium on Security and Privacy (SP)}, 2019.
\newblock URL \url{https://ieeexplore.ieee.org/stamp/stamp.jsp?arnumber=8835283}.

\bibitem[Zhang et~al.(2023)Zhang, Malekmohammadi, Chen, and Yu]{zhang2023proportional}
Zhang, G., Malekmohammadi, S., Chen, X., and Yu, Y.
\newblock Proportional fairness in federated learning.
\newblock \emph{Transactions on Machine Learning Research}, 2023.
\newblock URL \url{https://openreview.net/forum?id=ryUHgEdWCQ}.

\bibitem[Zhao et~al.(2021)Zhao, Zhao, Yang, Wang, Wang, Lyu, Niyato, and Lam]{Zhao2020LocalDP}
Zhao, Y., Zhao, J., Yang, M., Wang, T., Wang, N., Lyu, L., Niyato, D., and Lam, K.-Y.
\newblock Local differential privacy-based federated learning for internet of things.
\newblock \emph{IEEE Internet of Things Journal}, 2021.
\newblock URL \url{https://ieeexplore.ieee.org/stamp/stamp.jsp?tp=&arnumber=9253545}.

\bibitem[Zhou et~al.(2022)Zhou, Su, Ni, Wang, Pan, and Xing]{PDPFLglobecomm2022}
Zhou, J., Su, Z., Ni, J., Wang, Y., Pan, Y., and Xing, R.
\newblock Personalized privacy-preserving federated learning: Optimized trade-off between utility and privacy.
\newblock In \emph{GLOBECOM, IEEE Global Communications Conference}, 2022.
\newblock URL \url{https://ieeexplore.ieee.org/document/10000793}.

\bibitem[Zhou et~al.(2021)Zhou, Wu, and Banerjee]{projecteddpsgd}
Zhou, Y., Wu, S., and Banerjee, A.
\newblock Bypassing the ambient dimension: Private {SGD} with gradient subspace identification.
\newblock In \emph{9th International Conference on Learning Representations, {ICLR}}, 2021.
\newblock URL \url{https://arxiv.org/pdf/2007.03813}.

\bibitem[Zhu et~al.(2019)Zhu, Liu, and Han]{Zhu2019DeepLF}
Zhu, L., Liu, Z., and Han, S.
\newblock Deep leakage from gradients.
\newblock In Wallach, H., Larochelle, H., Beygelzimer, A., d\textquotesingle Alch\'{e}-Buc, F., Fox, E., and Garnett, R. (eds.), \emph{Advances in Neural Information Processing Systems}, 2019.
\newblock URL \url{https://proceedings.neurips.cc/paper_files/paper/2019/file/60a6c4002cc7b29142def8871531281a-Paper.pdf}.

\end{thebibliography}
\bibliographystyle{icml2024}



\appendix
\onecolumn
\newpage
\begin{center}
\Large
\bf
Appendix for \emph{Noise-Aware Algorithm for Heterogeneous differentially private federated learning}
\end{center}

\section{Notations}\label{app:notations}
We consider an \FL setting with $n$ clients. Let $x\in \mathcal{X}\subseteq\Rb^d$ and $y \in \mathcal{Y}=\left\{1, \ldots, C \right\}$ denote an input data point and its target label. Client $i$ holds dataset $\mathcal{D}_i = \{x_{ij}\}_{j=1}^{N_i}$ with $N_i$ samples from distribution $P_i(x,y)$. Let $h: \mathcal{X}\times\thetav\to\Rb^C$ be the predictor function, which is parameterized by $\thetav\in \Rb^p$ ($p$ is the number of model parameters) shared among all clients. Also, let $\ell:\Rb^C\times\mathcal{Y}\to \Rb_+$ be the loss function used (cross entropy loss). Following \cite{McMahanMRHA17}, many existing \FL algorithms fall into the natural formulation that minimizes the (arithmetic) average loss $f(\thetav) := \sum_{{\userind}=1}^{n} \lambda_{\userind} f_{\userind}(\thetav)$, where $f_i(\thetav)=\frac{1}{N_i}\sum_{(x,y)\in \mathcal{D}_i}[\ell(h(x,\thetav), y)]$, with minimum value $f_i^*$. The weights $\lambdav = (\lambda_1, \ldots, \lambda_n)$ are nonnegative and sum to 1. At gradient update $t$, client $i$ uses a data batch $\mathcal{B}_i^t$ with size $b_i = |\mathcal{B}_i^t|$. Let $q_i = \frac{b_i}{N_i}$ be batch size ratio of client $i$. There are $E$ global communication rounds indexed by $e$, and in each of them, client $i$ runs $K_i$ local epochs. We use boldface letters to denote vectors.

\section{Experimental setup}\label{app:exp_setup}
In this section, we provide more experimental details that were deferred to the appendix in the main paper.

\subsection{Datasets and models}\label{appendix:datasets}
\paragraph{MNIST and FMNIST datasets:}
\label{sec:mnist_exp_setup}
 We consider a distributed setting with $20$ clients. In order to create a heterogeneous dataset, we follow a similar procedure as in \cite{McMahanMRHA17}: first we split the data from each class into several shards. Then, each user is randomly assigned a number of shards of data.  For example, in some experiments, in order to guarantee that no user receives data from more than $8$ classes, we split each class of MNIST/FMNIST into $16$ shards (i.e., a total of $160$ shards for the whole dataset), and each user is randomly assigned $8$ shards of data. By considering $20$ clients, this procedure guarantees that no user receives data from more than $8$ classes and the data distribution of each user is different from each other. The local datasets are balanced--all clients have the same amount of training samples. In this way, each user has $2400$ data points for training,  and $600$ for testing. We use a simple 2-layer CNN model with ReLU activation, the details of which can be found in \Cref{table:mnist_fmnist_model}. To update the local models at each user using its local data, unless otherwise is stated, we apply stochastic gradient descent (\algname{SGD}).

\begin{table}[th]
\footnotesize	
\centering
\caption{CNN model for classification on MNIST/FMNIST datasets \label{table:mnist_fmnist_model}}
\begin{tabular}{lcccc} \toprule
          Layer &  Output Shape &  $\#$ of Trainable Parameters & Activation & Hyper-parameters  \\\midrule
           Input & $(1, 28, 28)$ & $0$ &  &  \\
           Conv2d & $(16, 28, 28)$ & $416$ & ReLU & kernel size =$5$; strides=$(1, 1)$ \\
           MaxPool2d & $(16, 14, 14)$ & $0$ &  & pool size=$(2, 2)$ \\
           Conv2d & $(32, 14, 14)$ & $12,\!832$ & ReLU & kernel size =$5$; strides=$(1, 1)$ \\
           MaxPool2d & $(32, 7, 7)$ & $0$ &  & pool size=$(2, 2)$ \\
           Flatten & $1568$ & $0$ & & \\
            Dense &  $10$ & $15,\!690$ & ReLU & \\
            \midrule
          Total & & $28,\!938$  & & \\ \bottomrule
\end{tabular}
\end{table}

\paragraph{CIFAR10/100 datasets:}
\label{sec:cifar_exp_setup}
We consider a distributed setting with $20$ clients, and split the 50,000 training samples and the 10,000 test samples in the datasets among them. In order to create a dataset, we follow a similar procedure as in \cite{McMahanMRHA17}: For instance for CIFAR10, first we sort all data points according to their classes. Then, each class is split into $20$ shards, and each user is randomly assigned $1$ shard of each class. We use the residual neural network (ResNet-18) defined in \cite{resnet}, which is a large model with $p=11,181,642$ parameters for CIFAR10. We also use ResNet-34 \cite{resnet}, which is a larger model with $p=21,272,778$ parameters for CIFAR100. To update the local models at each user using its local data, we apply stochastic gradient descent (\algname{SGD}). In the reported experimental results, all clients participate in each communication round.

\begin{table*}[ht]
\centering
\caption{Details of the experiments and the used datasets in the main body of the paper. ResNet-18/34 are the residual neural networks defined in \cite{resnet}. CNN: Convolutional Neural Network defined in \Cref{table:mnist_fmnist_model}. }
\label{tab:datasets}
\small
\setlength\tabcolsep{2pt}
\begin{tabular}{ccccccc}
\toprule
\bf{Datasets} & \bf{Train set size} & \bf{Test set size} & \bf{Data Partition method} & \bf{\# of clients} & \bf{Model} & \bf{\# of parameters} 
\\ 
\midrule
MNIST & 48000 & 12000 & sharding \cite{McMahanMRHA17} & 20/40/60 & CNN (\Cref{table:mnist_fmnist_model}) & 28,938\\

FMNIST & 50000 & 10000 & sharding \cite{McMahanMRHA17} & 20 & CNN (\Cref{table:mnist_fmnist_model}) & 28,938
\\

CIFAR10 & 50000 & 10000 & sharding \cite{McMahanMRHA17} &  20 & ResNet-18 \cite{resnet} & 11,181,642
\\

CIFAR100 & 50000 &  10000 & sharding \cite{McMahanMRHA17} & 20 & ResNet-34 \cite{resnet} & 21,272,778
\\
\bottomrule
\end{tabular}
\label{table:split_uniform}
\end{table*}

\begin{table}[t]
\centering
\begin{tabular}{l|*{1}{c}}\hline
\toprule
Distribution
&\makebox[9em]{Parameter setting}  \\
\hline

Dist1 &  Gaussian distribution $\mathcal{N}(2.0, 1.0)$\\\hline

Dist2 &  mixture of $\mathcal{N}(0.2, 0.01)$, $\mathcal{N}(1.0, 0.1)$ and $\mathcal{N}(5.0, 1.0)$ with weights $(0.2, 0.6, 0.2)$ \\\hline

Dist3 &  Uniform distribution $U[0.2, 5]$ \\\hline

Dist4 &  mixture of $\mathcal{N}(0.2, 0.01)$, $\mathcal{N}(0.5, 0.1)$ and $\mathcal{N}(2.0, 1.0)$ with weights $(0.2, 0.6, 0.2)$ \\\hline

Dist5 &  Uniform distribution $U[0.2, 2]$ \\\hline

Dist6 &  mixture of $\mathcal{N}(0.2, 0.01)$, $\mathcal{N}(0.5, 0.1)$ and $\mathcal{N}(1.0, 0.1)$ with weights $(0.3, 0.5, 0.2)$ \\\hline

Dist7 &  Uniform distribution $U[0.2, 1]$ \\\hline

Dist8 &  mixture of $\mathcal{N}(0.2, 0.01)$ and  $\mathcal{N}(0.5, 0.1)$ with weights $(0.6, 0.4)$ \\\hline

Dist9 &  Uniform distribution $U[0.2, 0.5]$\\
\bottomrule
\end{tabular}
\caption{Distributions of privacy parameters $(\epsilon)$, from which we sample clients' privacy parameters.}

\label{table:mixture_dists}
\end{table}

\subsection{\DP training parameters}
For each dataset, we sample the privacy parameter $\epsilon$ of clients from different distributions, as shown in \Cref{table:mixture_dists}. In order to get reasonable accuracy results for CIFAR100, which is a harder dataset compared to the other three datasets, we scale the values of $\epsilon$ sampled for clients from the distributions above by a factor 10. For instance, we have $\mathcal{N}(20.0, 10.0)$ as "Dist1" for CIFAR100. This is only for getting meaningful accuracy values for CIFAR100, otherwise the test accuracy values will be too low. We fix $\delta$ for all clients to $10^{-4}$. We also set the clipping threshold $c$ equal to $3$, as it results in better test accuracy, as reported in \cite{Abadi2016}.

\subsection{Algorithms to compare and tuning hyperparameters}

We compare our \algname{Robust-HDP}, which benefits from \algname{RPCA} (\Cref{alg:RPCA}), with four baseline algorithms, including WeiAvg \cite{Liu2021ProjectedFA} (\Cref{alg:WeiAvg}), PFA \cite{Liu2021ProjectedFA}, DPFedAvg \cite{DPSCAFFOLD2022} and minimum $\epsilon$ \cite{Liu2021ProjectedFA}. For PFA, we always use projection dimension 1, as in \cite{Liu2021ProjectedFA}. For each algorithm and each dataset, we find the best learning rate from a grid: \emph{the one which is small enough to avoid divergence of the federated optimization, and results in the lowest average train loss (across clients) at the end of \FL training}. Here are the grids we use for each dataset:

\begin{itemize}
\item MNIST: \texttt{\{1e-4, 2e-4, 5e-4, 1e-3, 2e-3, 5e-3, 1e-2\}};
\item FMNIST: \texttt{\{1e-4, 2e-4, 5e-4, 1e-3, 2e-3, 5e-3, 1e-2\}};
\item CIFAR10:  \texttt{\{1e-4, 2e-4, 5e-4, 1e-3, 2e-3, 5e-3, 1e-2\}};
\item CIFAR100: \texttt{\{1e-5, 2e-5, 5e-5, 1e-4, 2e-4, 5e-4, 1e-3\}}.
\end{itemize}

The best learning rates used for each dataset are reported in \Cref{table:mnistlr} to \Cref{table:cifar100lr}.

{~~~~~~\centering
\begin{minipage}{.9\linewidth}
\begin{algorithm}[H]
\caption{\algname{WeiAvg} \cite{Liu2021ProjectedFA}}
\label{alg:WeiAvg}
\KwIn{Initial parameter $\thetav^0$, Clients batch sizes $\{b_1, \ldots, b_n\}$, Clients dataset sizes $\{N_1, \ldots, N_n\}$, \\ Clients noise scales $\{z_1, \ldots, z_n\}$, gradient norm bound $c$, local epochs $\{K_1, \ldots, K_n\}$, global round $E$, \\ 
privacy parameter $\delta$, number of model parameters $p$, privacy accountant \textbf{PA}.}

\KwOut{$\thetav_E, \{\epsilon_1, \ldots, \epsilon_n\}$}

\textbf{Initialize} $\thetav_0$ randomly.

\For{$e\in [E]$}
{
sample a set of clients $\mathcal{S}^e \subseteq \{1, \ldots, n\} $

\For{each client $\userind \in \mathcal{S}^e$ \textbf{in parallel}}{
  $\Delta \Tilde{\thetav}_i^e \gets$\textbf{\algname{DPSGD}($\thetav^e, b_i, N_i, K_i, z_i, c$)}
  
  $\epsilon_i^e \gets \textbf{PA}(\frac{b_i}{N_i}, z_i, K_i, e)$
  }

\For{$i \in \mathcal{S}^e$}
{
    $w_i^e \gets \frac{\epsilon_i}{\sum_{j \in \mathcal{S}_e} \epsilon_j}$
}

{$\thetav^{e+1} \gets \thetav^e + \sum_{i \in \mathcal{S}_e} w_i^e \Delta \Tilde{\thetav}_i^e$}
}
\KwOut{$\thetav^E, \{\epsilon_1^E, \ldots, \epsilon_n^E\}$}
\end{algorithm}
\end{minipage}
}

{~~~~~~\centering
\begin{minipage}{.9\linewidth}
\begin{algorithm}[H]
\caption{Principal Component Pursuit by Alternating Directions \cite{Candes2009RobustPC}}
\label{alg:RPCA}

\KwIn{matrix $M$, shrinkage operator $\mathcal{S}_{\tau}[x] =$ sgn$(x) \max(|x|-\tau,0)$, singular value thresholding operator \\
$\mathcal{D}_{\tau}(U\Sigma V^*) = U\mathcal{S}_{\tau}(\Sigma)V^*$}

\textbf{Initialize} $S_0 = Y_0 = 0, \mu>0$. \\
\While{not converged}{
~~compute $L_{k+1} = \mathcal{D}_{\mu^{-1}}(M-S_k-\mu^{-1}Y_k)$\\
compute $S_{k+1} = \mathcal{S}_{\lambda \mu^{-1}}(M-L_{k+1}+\mu^{-1}Y_k)$\\
compute $Y_{k+1} = Y_k + \mu(M-L_{k+1}-S_{k+1})$\\
}
\KwOut{$L, S$}
\end{algorithm}
\end{minipage}
}

\vspace{-1em}
\begin{table}[hbt!]
\centering
\caption{The learning rates used for training with each algorithm on MNIST dataset}
\begin{tabular}{l|*{8}{c}c}\toprule
\diagbox{alg}{dist}
&\makebox[2.5em]{\footnotesize Dist1}
&\makebox[2.5em]{\footnotesize Dist2}
&\makebox[2.5em]{\footnotesize Dist3}
&\makebox[2.5em]{\footnotesize Dist4}
&\makebox[2.5em]{\footnotesize Dist5}
&\makebox[2.5em]{\footnotesize Dist6}
&\makebox[2.5em]{\footnotesize Dist7}
&\makebox[2.5em]{\footnotesize Dist8}
&\makebox[2.5em]{\footnotesize Dist9}\\
\midrule \midrule
\footnotesize \algname{WeiAvg} \cite{Liu2021ProjectedFA} & \tt 1e-2& \tt 5e-3& \tt 1e-2& \tt 5e-3 &\tt 5e-3 & \tt 1e-3 & \tt 1e-3& \tt 1e-3& \tt 1e-3\\\midrule
\footnotesize \algname{PFA}\cite{Liu2021ProjectedFA} & \tt 5e-3& \tt 5e-3& \tt 5e-3& \tt 5e-3& \tt 5e-3 & \tt 1e-3& \tt 1e-3& \tt 1e-3& \tt 5e-4\\\midrule
\footnotesize \algname{DPFedAvg} \cite{DPSCAFFOLD2022} & \tt 5e-3& \tt 1e-3& \tt 1e-3& \tt 1e-3& \tt 1e-3& \tt 5e-4& \tt 1e-3& \tt 1e-3& \tt 1e-3\\ \midrule
\footnotesize \algname{minimum} $\epsilon$ \cite{Liu2021ProjectedFA}& \tt 5e-4 & \tt 5e-4& \tt 5e-4 & \tt 5e-4 & \tt 1e-3& \tt 1e-4& \tt 1e-3& \tt 5e-4& \tt 1e-3\\\midrule
\footnotesize Robust-HDP & \tt 1e-2& \tt 1e-2& \tt 1e-2& \tt 1e-2 & \tt 5e-3 & \tt 2e-3& \tt 2e-3& \tt 2e-3& \tt 2e-3\\\midrule

\end{tabular}
\label{table:mnistlr}
\end{table}

\vspace{-1em}
\begin{table}[hbt!]
\centering
\caption{The learning rates used for training with each algorithm on FMNIST dataset}
\begin{tabular}{l|*{8}{c}c}\toprule
\diagbox{alg}{dist}
&\makebox[2.5em]{\footnotesize Dist1}
&\makebox[2.5em]{\footnotesize Dist2}
&\makebox[2.5em]{\footnotesize Dist3}
&\makebox[2.5em]{\footnotesize Dist4}
&\makebox[2.5em]{\footnotesize Dist5}
&\makebox[2.5em]{\footnotesize Dist6}
&\makebox[2.5em]{\footnotesize Dist7}
&\makebox[2.5em]{\footnotesize Dist8}
&\makebox[2.5em]{\footnotesize Dist9}\\
\midrule \midrule
\footnotesize \algname{WeiAvg} \cite{Liu2021ProjectedFA} & \tt 5e-3 & \tt 5e-3 & \tt 5e-3 & \tt 5e-3 & \tt 2e-3 & \tt 5e-4 & \tt 5e-4& \tt 5e-4& \tt 5e-4\\\midrule
\footnotesize \algname{PFA}\cite{Liu2021ProjectedFA} & \tt 2e-3 & \tt 2e-3& \tt 5e-3& \tt 5e-3& \tt 5e-3& \tt 5e-3& \tt 2e-3& \tt 1e-3& \tt 1e-3\\\midrule
\footnotesize \algname{DPFedAvg} \cite{DPSCAFFOLD2022} & \tt 2e-3 & \tt 1e-3& \tt 1e-3& \tt 1e-3& \tt 1e-3& \tt 5e-4& \tt 5e-4& \tt 5e-4 & \tt 5e-4\\\midrule
\footnotesize \algname{minimum} $\epsilon$ \cite{Liu2021ProjectedFA}& \tt 1e-3 & \tt 5e-4 & \tt 5e-4& \tt 5e-4 & \tt 5e-4 & \tt 1e-4& \tt 5e-4& \tt 5e-4& \tt 5e-4\\\midrule
\footnotesize Robust-HDP & \tt 5e-3 & \tt 5e-3& \tt 5e-3& \tt 5e-3 & \tt 5e-3& \tt 1e-3& \tt 1e-3& \tt 1e-3& \tt 1e-3\\\midrule

\end{tabular}
\label{table:fmnistlr}
\end{table}

\begin{table}[hbt!]
\centering
\caption{The learning rates used for training with each algorithm on CIFAR10 dataset}
\begin{tabular}{l|*{8}{c}c}\toprule
\diagbox{alg}{dist}
&\makebox[2.5em]{\footnotesize Dist1}
&\makebox[2.5em]{\footnotesize Dist2}
&\makebox[2.5em]{\footnotesize Dist3}
&\makebox[2.5em]{\footnotesize Dist4}
&\makebox[2.5em]{\footnotesize Dist5}
&\makebox[2.5em]{\footnotesize Dist6}
&\makebox[2.5em]{\footnotesize Dist7}
&\makebox[2.5em]{\footnotesize Dist8}
&\makebox[2.5em]{\footnotesize Dist9}\\
\hline\hline
\footnotesize \algname{WeiAvg} \cite{Liu2021ProjectedFA} & \tt 2e-3 & \tt 1e-3 & \tt 1e-3 & \tt 5e-4 & \tt 5e-4 & \tt 2e-4 & \tt 2e-4  & \tt 2e-4 & \tt 2e-4\\\midrule
\footnotesize PFA \cite{Liu2021ProjectedFA} & \tt 2e-3 & \tt 2e-3  & \tt 2e-3  & \tt 2e-3 & \tt 2e-3 & \tt 1e-3  & \tt 5e-4 & \tt 5e-4 & \tt 2e-4\\\midrule
\footnotesize DPFedAvg \cite{DPSCAFFOLD2022} & \tt 1e-3 & \tt 5e-4 & \tt 2e-4 & \tt 2e-4 & \tt 2e-4 & \tt 1e-4 & \tt 1e-4 & \tt 5e-5 & \tt 1e-4\\\midrule
\footnotesize \algname{minimum} $\epsilon$ \cite{Liu2021ProjectedFA}& \tt 2e-3 & \tt 1e-3 & \tt 1e-3 & \tt 1e-3 & \tt 1e-3 & \tt 1e-4 & \tt 5e-4& \tt 2e-4 & \tt 2e-4\\\midrule
\footnotesize Robust-HDP & \tt 2e-3 & \tt 2e-3 & \tt 2e-3 & \tt 2e-3 & \tt 2e-3  & \tt 5e-4 & \tt 1e-3 & \tt 2e-4 & \tt 2e-4\\\midrule

\end{tabular}
\label{table:cifar10lr}
\end{table}

\begin{table}[hbt!]
\centering
\caption{The learning rates used for training with each algorithm on CIFAR100 dataset.}
\begin{tabular}{l|*{8}{c}c}\toprule
\diagbox{alg}{dist}
&\makebox[2.5em]{\footnotesize Dist1}
&\makebox[2.5em]{\footnotesize Dist2}
&\makebox[2.5em]{\footnotesize Dist3}
&\makebox[2.5em]{\footnotesize Dist4}
&\makebox[2.5em]{\footnotesize Dist5}
&\makebox[2.5em]{\footnotesize Dist6}
&\makebox[2.5em]{\footnotesize Dist7}
&\makebox[2.5em]{\footnotesize Dist8}
&\makebox[2.5em]{\footnotesize Dist9}\\
\hline\hline
\footnotesize \algname{WeiAvg} \cite{Liu2021ProjectedFA} & \tt 1e-3 & \tt 1e-3 & \tt 1e-3 & \tt 5e-4 & \tt 5e-4 & \tt 2e-4 & \tt 2e-4  & \tt 2e-4 & \tt 2e-4\\\midrule
\footnotesize PFA \cite{Liu2021ProjectedFA} & \tt 2e-3 & \tt 2e-3 & \tt 2e-3 & \tt 1e-3 & \tt 1e-3  & \tt 5e-4 & \tt 2e-4 & \tt 2e-4 & \tt 1e-4 \\\midrule
\footnotesize DPFedAvg \cite{DPSCAFFOLD2022} & \tt 5e-4 & \tt 5e-4 & \tt 1e-4 & \tt 1e-4 & \tt 1e-4 & \tt 5e-5 & \tt 5e-5 & \tt 2e-5 & \tt 2e-5\\\midrule
\footnotesize \algname{minimum} $\epsilon$ \cite{Liu2021ProjectedFA}& \tt 2e-4 & \tt 2e-4 & \tt 1e-4 & \tt 1e-4 & \tt 1e-4 & \tt 5e-5 & \tt 5e-5& \tt 2e-5 & \tt 2e-5\\\midrule
\footnotesize Robust-HDP & \tt 2e-3 & \tt 2e-3 & \tt 2e-3 & \tt 2e-3 & \tt 2e-3  & \tt 1e-3 & \tt 1e-3 & \tt 1e-3 & \tt 1e-3\\\midrule

\end{tabular}
\label{table:cifar100lr}
\end{table}

\begin{figure*}[t]
\centering
    \includegraphics[width=0.5\columnwidth]{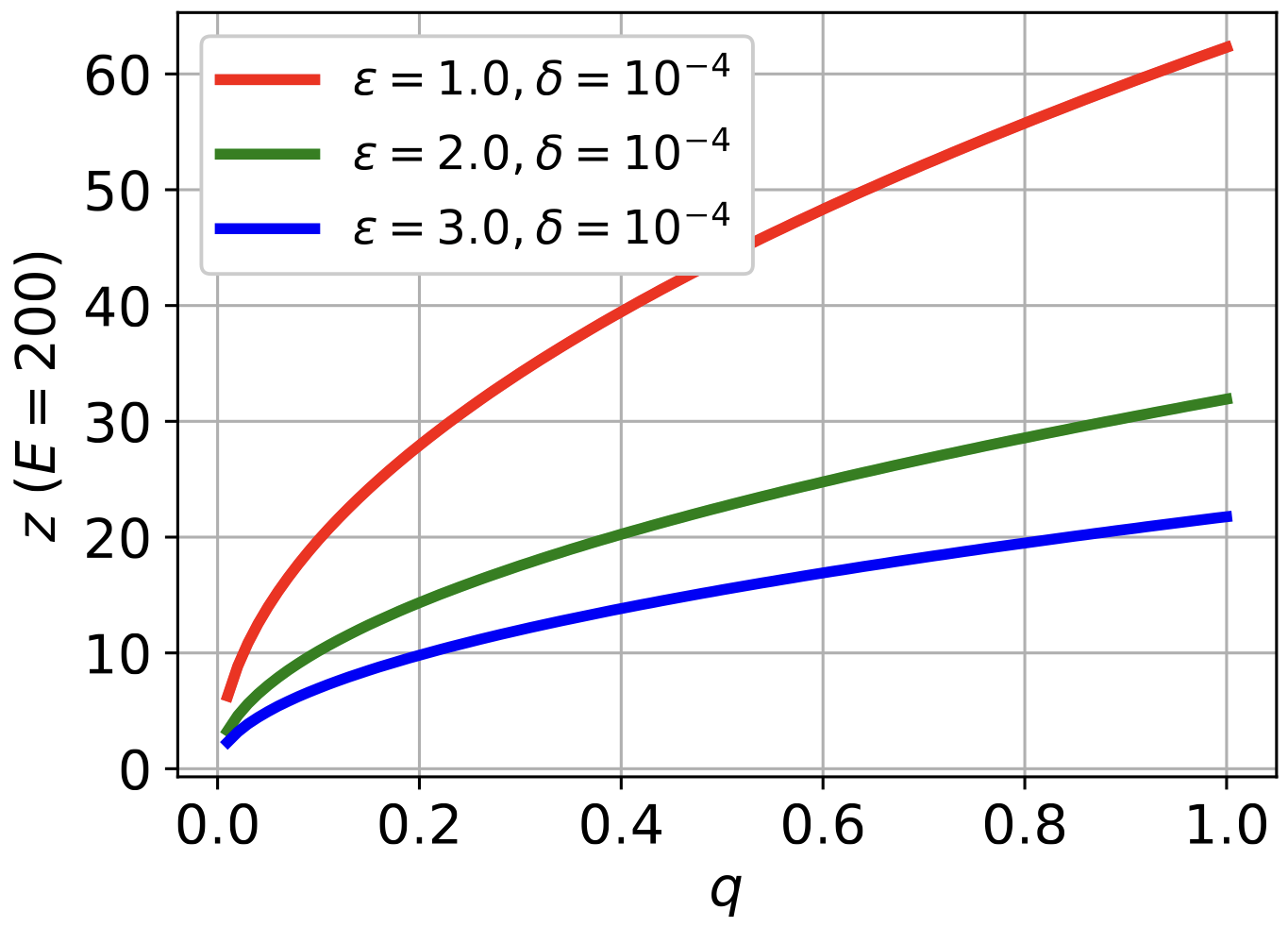}
    \caption{Plot of $z$ v.s. $q$ obtained from Moments Accountant \cite{Abadi2016} in a centralized setting with $E=200$. As observed, $z$ increases sub-linearly with $q$ (or equivalently with $b$) and decreases with dataset size.}
    \label{fig:zvsq}
\end{figure*}

    

\clearpage
\newpage

\section{Derivations}
\par{\textbf{Computation of \texorpdfstring{$\sigma_{i, \Tilde{g}}^2$}{Lg}, when gradient clipping is effective for all samples:}}\label{app:variance_derivation}
We know that the two sources of randomness (i.e., minibatch sampling and Gaussian noise) are independent, thus the variance is additive. Assuming that $E[\bar{g}_{ij}(\thetav)]$ is the same for all $j$ and is $G_i(\thetav)$, we have:

\begin{align} 
\label{var_g}
    \sigma_{i, \Tilde{g}}^2 := \texttt{Var}(\Tilde{g}_i(\thetav)) &= \texttt{Var}\bigg(\frac{1}{b_i}\sum_{j \in \mathcal{B}_i^t} \Bar{g}_{ij}(\thetav)\bigg) + \frac{p \sigma_{i, \texttt{\DP}}^2}{b_i^2} \nonumber\\
    & = \frac{1}{b_i^2}\bigg(\mathbb E \bigg[\bigg\|\sum_{j \in \mathcal{B}_i^t}\Bar{g}_{ij}(\thetav)\bigg\|^2\bigg] - \bigg\|\mathbb E \bigg[\sum_{j \in \mathcal{B}_i^t}\Bar{g}_{ij}(\thetav)\bigg]\bigg\|^2\bigg) + \frac{pc^2 z^2(\epsilon_i, \delta_i, q_i, K_i, E)}{b_i^2} \nonumber \\ 
    & = \frac{1}{b_i^2}\bigg(\mathbb E \bigg[\bigg\|\sum_{j \in \mathcal{B}_i^t}\Bar{g}_{ij}(\thetav)\bigg\|^2\bigg] - \bigg\|\sum_{j \in \mathcal{B}_i^t}G_i(\thetav) \bigg\|^2\bigg) + \frac{pc^2 z^2(\epsilon_i, \delta_i, q_i, K_i, E)}{b_i^2} \nonumber \\
    & = \frac{1}{b_i^2}\bigg(\mathbb E \bigg[\bigg\|\sum_{j \in \mathcal{B}_i^t}\Bar{g}_{ij}(\thetav)\bigg\|^2\bigg] - b_i^2\big\| G_i(\thetav) \big\|^2\bigg) + \frac{p c^2 z^2(\epsilon_i, \delta_i, q_i, K_i, E)}{b_i^2}
\end{align}

We also have:

\begin{align}
\label{eq:mean_norm2}
\mathbb E \bigg[\bigg\|\sum_{j \in \mathcal{B}_i^t}\Bar{g}_{ij}(\thetav)\bigg\|^2\bigg] &= \sum_{j \in \mathcal{B}_i^t} \mathbb E \bigg[\big\|\Bar{g}_{ij}(\thetav)\big\|^2\bigg] + \sum_{m \neq n \in \mathcal{B}_i^t} 2 \mathbb E \bigg[[\Bar{g}_{im}(\thetav)]^\top [\Bar{g}_{in}(\thetav)]\bigg] \nonumber \\
&= \sum_{j \in \mathcal{B}_i^t} \mathbb E \bigg[\big\|\Bar{g}_{ij}(\thetav)\big\|^2\bigg] + \sum_{m \neq n \in \mathcal{B}_i^t} 2 \mathbb E \bigg[\Bar{g}_{im}(\thetav)\bigg ]^\top \mathbb E \bigg[\Bar{g}_{in}(\thetav)\bigg] \nonumber \\
&= b_i c^2 + 2 \binom{b_i}{2} \big\| G_i(\thetav)\big\|^2,
\end{align}
where the last equation has used \Cref{expectation_gtilde} and that we clip the norm of sample gradients $\Bar{g}_{ij}(\thetav)$ with an ``effective" clipping threshold $c$. We can now plug eq. \ref{eq:mean_norm2} into the parenthesis in eq. \ref{var_g} and rewrite it as:

\begin{align}
    \sigma_{i, \Tilde{g}}^2 := \texttt{Var}(\Tilde{g}_i(\thetav)) & = \frac{1}{b_i^2}\bigg(\mathbb E \bigg[\bigg\|\sum_{j \in \mathcal{B}_i^t}\Bar{g}_{ij}(\thetav)\bigg\|^2\bigg] - b_i^2\big\| G_i(\thetav) \big\|^2\bigg) + \frac{p c^2 z^2(\epsilon_i, \delta_i, q_i, K_i, E)}{b_i^2} \nonumber\\
    &=  \frac{1}{b_i^2}\bigg( b_ic^2 + \bigg(2 \binom{b_i}{2} -b_i^2\bigg) \big\| G_i(\thetav)\big\|^2 \bigg) + \frac{p c^2 z^2(\epsilon_i, \delta_i, q_i, K_i, E)}{b_i^2} \nonumber\\
    &=  \frac{1}{b_i^2}\bigg( b_i c^2 -b_i \big\| G_i(\thetav)\big\|^2 \bigg) + \frac{p c^2 z^2(\epsilon_i, \delta_i, q_i, K_i, E)}{b_i^2} \nonumber\\
    &=  \frac{c^2 - \big\| G_i(\thetav)\big\|^2}{b_i} + \frac{p c^2 z^2(\epsilon_i, \delta_i, q_i, K_i, E)}{b_i^2}  \nonumber\\
    &\approx \frac{p c^2 z^2(\epsilon_i, \delta_i, q_i, K_i, E)}{b_i^2}
\end{align}

\newpage
\section{Assumptions and lemmas}
In this section, we formalize our assumptions and some lemmas, which we will use in our proofs.




\begin{assumption}[\textbf{Lipschitz continuity, $\beta$-smoothness and bounded gradient variance}]\label{assump:lipschitz_smooth_bounded}
$\{f_i\}_{i=1}^n$ are $L_0$-Lipschitz continuous and $\beta$-smooth: $\forall ~ \thetav,\thetav'\in \mathbb R^p,i: \|f_i(\thetav) - f_i(\thetav')\| \leq L_{0}\|\thetav - \thetav'\|$ and $\|\nabla f_i(\thetav) - \nabla f_i(\thetav')\| \leq \beta\|\thetav-\thetav'\|$. Also, the stochastic gradient $g_i(\thetav)$ is an unbiased estimate of $\nabla f_i(\thetav)$ with bounded variance: $\forall \thetav \in \mathbb R^p: \mathbb{E}_{\mathcal{B}_i^t} [g_i(\thetav)] = \nabla f_i(\thetav), ~~ \mathop{\mathbb{E}_{\mathcal{B}_i^t}}\big[\|g_i(\thetav) - \nabla f_i(\thetav)\|^2\big] \leq \sigma_{i, g}^2$. We also assume that for every $i, j \in [n], f_i-f_j$ is $\sigma$-Lipschitz continuous: $\|\nabla f_i(\thetav)-\nabla f_j(\thetav)\|\leq\sigma.$ 
\end{assumption}

\begin{assumption}[\textbf{\bf bounded sample gradients}]\label{assump:bounded_sample_grad} There exists a clipping threshold $\mathcal{C}$ such that for all $i, j$:
\begin{align}
    \|g_{ij}(\thetav)\|_2 := \|\nabla \ell(h(x_{ij},\thetav), y_{ij})\|_2 \leq \mathcal{C}
\end{align} 
Note that this condition always holds if $\ell$ is Lipschitz continuous or if $h$ is bounded. 
\end{assumption}

\begin{lemma}[Relaxed triangle inequality]\label{lemma:relaxed_triangle}
Let $\{v_1, \ldots, v_n\}$ be $n$ vectors in $\mathbb R^d$. Then, the followings is true:

\begin{itemize}
    \item $\|v_i + v_j\|^2 \leq (1+a)\|v_i\|^2 + (1+\frac{1}{a})\|v_j\|^2 ~(\textit{for any~} a>0)$
    \item $\|\sum_i v_i\|^2 \leq n \sum_i \|v_i\|^2$
\end{itemize}

\begin{proof}
The proof for the first inequality is obtained from identity:
\begin{align}
    \|v_i + v_j\|^2 = (1+a)\|v_i\|^2 + (1+\frac{1}{a})\|v_j\|^2 - \|\sqrt{a}v_i + \frac{1}{\sqrt{a}}v_j\|^2
\end{align}
The proof for the second inequality is achieved by using the fact that $h(x)=\|x\|^2$ is convex:
\begin{align}
    \|\frac{1}{n}\sum_i v_i\|^2 \leq \frac{1}{n} \sum_i \|v_i\|^2
\end{align}
\end{proof}
\end{lemma}

\begin{lemma}\label{lemma:meanvar}
    Let $\{v_1, \ldots, v_n\}$ be $n$ random variables in $\mathbb R^d$, with $\mathbb E[v_i] = \mathcal{E}_i$ and $\mathbb E [\|v_i - \mathcal{E}_i\|^2] = \sigma_i^2$. Then, we have the following inequality:
    \begin{align}
        \mathbb E[\|\sum_{i=1}^n v_i\|^2] \leq \mathbb \|\sum_{i=1}^n \mathcal{E}_i\|^2 + n \sum_{i=1}^n \sigma_i^2.
    \end{align}

 \begin{proof}
     From the definition of variance, we have:
     \begin{align}
         \mathbb E[\|\sum_{i=1}^n v_i\|^2] &= \mathbb \|\sum_{i=1}^n \mathcal{E}_i\|^2 + \mathbb E[\|\sum_{i=1}^n (v_i - \mathcal{E}_i)\|^2]\\
         &\leq \mathbb \|\sum_{i=1}^n \mathcal{E}_i\|^2 + n \sum_{i=1}^n \mathbb E[\| v_i - \mathcal{E}_i\|^2]\\
         &= \mathbb \|\sum_{i=1}^n \mathcal{E}_i\|^2 + n \sum_{i=1}^n \sigma_i^2,\\
    \end{align}
    where the inequality is based on the \Cref{lemma:relaxed_triangle}.
 \end{proof}   
\end{lemma}

\begin{property}[\textbf{Parallel Composition \cite{Yu2019DifferentiallyPM}}]\label{prop:parallel_composition}
Assume each of the randomized mechanisms $M_i: \mathcal{D}_i \to \mathbb R$ for $i \in [n]$ satisfies $(\epsilon_i, \delta_i)$-\DP and their domains $\mathcal{D}_i$ are disjoint subsets. Any function $g$ of the form $g(M_1, \ldots, M_n)$ satisfies $(\max_i \epsilon_i, \max_i \delta_i)$-\DP. 

\end{property}






\newpage
\section{Proofs}\label{sec:appendix_proofs}


\localdp*
\begin{proof}
The proof for the first part follows the proof of \algname{DPSGD} algorithm \cite{Abadi2016}.
Also, in \algname{Robust-HDP}, each client $i$ runs \algname{DPSGD} locally to achieve $(\epsilon_i, \delta_i)$-\DP independently. Hence, it satisfies heterogeneous \DP with the set of preferences $\{(\epsilon_i, \delta_i)\}_{i=1}^n$. Also, the clients datasets $\{\mathcal{D}_i\}_{i=1}^n$ are
disjoint. Hence, as \algname{Robust-HDP} runs RPCA on the clients models updates, it satisfies $\big(\max(\{\epsilon_i\}_{i=1}^n), \max(\{\delta_i\}_{i=1}^n)\big)$-\DP, according to parallel composition property above.
\end{proof}

\Robusthdp*
\begin{proof}

From our assumption \ref{assump:lipschitz_smooth_bounded} and that we use cross-entropy loss, we can conclude that Assumption \ref{assump:bounded_sample_grad} also holds for some $\mathcal{C}$. In that case, we have:

\begin{align}
    \Tilde{g}_i(\thetav) & = \frac{\sum_{j \in \mathcal{B}_i^t} g_{ij}(\thetav)}{b_i} + \mathcal{N}(\mathbf{0}, \frac{\sigma_{i, \texttt{\DP}}^2}{b_i^2} \mathbb{I}_p) = g_i(\thetav) + \mathcal{N}(\mathbf{0}, \frac{\sigma_{i, \texttt{\DP}}^2}{b_i^2} \mathbb{I}_p)
\end{align}

Therefore:
\begin{align}\label{eq:newvariance}
    & \mathbb E[\Tilde{g}_i(\thetav)] = \mathbb E[g_i(\thetav)] = \nabla f_i(\thetav) \nonumber
    \\
    & \texttt{Var}(\Tilde{g}_i(\thetav)) = \texttt{Var}(g_i(\thetav)) + \frac{p \sigma_{i, \texttt{\DP}}^2}{b_i^2} \leq \sigma_{i, \Tilde{g}}^2 := \sigma_{i, g}^2 + \frac{p \sigma_{i, \texttt{\DP}}^2}{b_i^2}.
\end{align}
i.e., the assumption of having unbiased gradient with bounded variance still holds (with a larger bound $\sigma_{i, \Tilde{g}}^2$, due to adding \DP noise). Consistent with the previous notations, we assume that the set of participating clients in round $e$ are $\mathcal{S}^e$, and for every client $i\notin \mathcal{S}^e$, we set $w_i^e=0$. Using this, we can write the model parameter at the end of round $e$ as:

\begin{align}
    \thetav^{e+1} = \sum_{i=1}^n w_i^e \thetav_{i, E_i}^e,
\end{align}
where $\{E_i\}_{i=1}^n$ is the heterogeneous number of gradient steps of clients (depending on their dataset size and batch size). From $\thetav_{i,k}^e = \thetav_{i,k-1}^e - \eta_l \Tilde{g}_i(\thetav_{i,k-1}^e)$, we can rewrite the equation above as:

\begin{align}
    \thetav^{e+1} = \thetav^e - \eta_l \sum_{i\in\mathcal{S}^e} w_i^e \sum_{k=1}^{E_i} \Tilde{g}_i(\thetav_{i,k-1}^e) = \thetav^e - \eta_l \sum_{i=1}^n w_i^e \sum_{k=1}^{E_i} \Tilde{g}_i(\thetav_{i,k-1}^e) = \thetav^e - \eta_l \sum_{i=1}^n w_i^e \sum_{k=0}^{E_i-1} \Tilde{g}_i(\thetav_{i,k}^e)
\end{align}

Note that the second equality holds because we assumed above that if client $i$ is not participating in round $e$ (i.e., $i \notin \mathcal{S}^e$), we set $w_i^e=0$. From $\beta$-smoothness of $\{f_i\}_{i=1}^n$, and consequently $\beta$-smoothness of $f$, we have:

\begin{align}
     f(\thetav^{e+1}) &\leq f(\thetav^{e}) + \langle \nabla f(\thetav^{e}), \thetav^{e+1}-\thetav^{e}\rangle + \frac{\beta}{2} \| \thetav^{e+1}- \thetav^{e}\|^2 \nonumber\\
     & = f(\thetav^{e}) - \eta_l \big\langle \nabla f(\thetav^{e}), \sum_{i=1}^n w_i^e \sum_{k=0}^{E_i-1} \Tilde{g}_i(\thetav_{i,k}^e)\big\rangle + \frac{\beta \eta_l^2}{2} \big\|\sum_{i=1}^n w_i^e \sum_{k=0}^{E_i-1} \Tilde{g}_i(\thetav_{i,k}^e)\big\|^2
\end{align} 

Now, we use identity $\Tilde{g}_i(\thetav_{i,k}^e) = \nabla f(\thetav^e) + \Tilde{g}_i(\thetav_{i,k}^e) - \nabla f(\thetav^e)$ to rewrite the equation above as:

\begin{align}
     f(\thetav^{e+1}) &\leq f(\thetav^{e}) - \eta_l \big\langle \nabla f(\thetav^{e}), \sum_{i=1}^n w_i^e \sum_{k=0}^{E_i-1} \nabla f(\thetav^e)\big\rangle - \eta_l \big\langle \nabla f(\thetav^{e}), \sum_{i=1}^n w_i^e \sum_{k=0}^{E_i-1} \big(\Tilde{g}_i(\thetav_{i,k}^e)-\nabla f(\thetav^e)\big)\big\rangle \nonumber \\
     & + \frac{\beta \eta_l^2}{2} \big\|\sum_{i=1}^n w_i^e \sum_{k=0}^{E_i-1} \big(\Tilde{g}_i(\thetav_{i,k}^e) - \nabla f(\thetav^{e})\big) + \sum_{i=1}^n w_i^e E_i \nabla f(\thetav^{e})\big\|^2 \nonumber \\
\end{align}

Hence,
\begin{align}
     f(\thetav^{e+1}) &\leq f(\thetav^{e}) - \eta_l \big\langle \nabla f(\thetav^{e}), \sum_{i=1}^n w_i^e \sum_{k=0}^{E_i-1} \nabla f(\thetav^e)\big\rangle - \eta_l \big\langle \nabla f(\thetav^{e}), \sum_{i=1}^n w_i^e \sum_{k=0}^{E_i-1} \big(\Tilde{g}_i(\thetav_{i,k}^e)-\nabla f(\thetav^e)\big)\big\rangle \nonumber \\
     & + \frac{\beta \eta_l^2}{2} \big\|\sum_{i=1}^n w_i^e \sum_{k=0}^{E_i-1} \big(\Tilde{g}_i(\thetav_{i,k}^e) - \nabla f(\thetav^{e})\big)\|^2 + \frac{\beta \eta_l^2}{2} \underbrace{(\sum_{i=1}^n w_i^e E_i)^2}_{\Bar{E}_l^{e^2}}\big\|\nabla f(\thetav^{e})\big\|^2 \nonumber \\
     &+ \beta \eta_l^2 \big\langle \sum_{i=1}^n w_i^e E_i \nabla f(\thetav^{e}), \sum_{i=1}^n w_i^e \sum_{k=0}^{E_i-1} \big(\Tilde{g}_i(\thetav_{i,k}^e)-\nabla f(\thetav^e)\big)\big\rangle.
\end{align} 

\textbf{Note that we denote $\sum_{i=1}^n w_i^e E_i$ with $\Bar{E}_l^e$ from now on}. With doing some algebra we get to:

\begin{align}
     f(\thetav^{e+1}) &\leq f(\thetav^{e}) - \eta_l \Bar{E}_l^e (1 - \frac{\beta}{2}\eta_l \Bar{E}_l^e) \|\nabla f(\thetav^e)\|^2 \nonumber \\
     & -\eta_l (1-\beta \eta_l \Bar{E}_l^e) \big\langle \nabla f(\thetav^{e}), \sum_{i=1}^n w_i^e \sum_{k=0}^{E_i-1} \big(\Tilde{g}_i(\thetav_{i,k}^e)-\nabla f(\thetav^e)\big)\big\rangle \nonumber \\
     & + \frac{\beta \eta_l^2}{2} \bigg\|\sum_{i=1}^n w_i^e \sum_{k=0}^{E_i-1} \big(\Tilde{g}_i(\thetav_{i,k}^e) - \nabla f(\thetav^{e})\big)\bigg\|^2. 
\end{align} 

By taking expectation from both side (expectation is conditioned on $\thetav^e$) and using Cauchy-Schwarz inequality, we have:

\begin{align}
     \mathbb E \big[f(\thetav^{e+1})\big] &\leq \mathbb E \big[f(\thetav^{e})\big] - \eta_l \Bar{E}_l^e (1 - \frac{\beta \eta_l}{2} \Bar{E}_l^e) \mathbb E \big[\|\nabla f(\thetav^e)\|^2\big] \nonumber \\
     & +\eta_l (1-\beta \eta_l \Bar{E}_l^e) \mathbb E\bigg[ \|\nabla f(\thetav^{e})\|\times \bigg\| \sum_{i=1}^n w_i^e \sum_{k=0}^{E_i-1} \big(\nabla f_i(\thetav_{i,k}^e)-\nabla f(\thetav^e)\big)\bigg\|\bigg] \nonumber \\
     & + \frac{\beta \eta_l^2}{2} \mathbb E \bigg[\bigg\|\sum_{i=1}^n w_i^e \sum_{k=0}^{E_i-1} \big(\Tilde{g}_i(\thetav_{i,k}^e) - \nabla f(\thetav^{e})\big)\bigg\|^2\bigg]. 
\end{align} 

Now, we use the inequality $ab \leq \frac{1}{2}(a^2 + b^2)$ for the second line to get:

\begin{align}
     \mathbb E \big[f(\thetav^{e+1})\big] &\leq \mathbb E \big[f(\thetav^{e})\big] + \underbrace{\big(\frac{1}{2}\eta_l (1-\beta \eta_l \Bar{E}_l^e) - \eta_l \Bar{E}_l^e (1 - \frac{\beta \eta_l}{2} \Bar{E}_l^e)\big)}_{\leq -\eta_l \frac{11 \Bar{E}_l^e-6}{12}} \mathbb E \big[\|\nabla f(\thetav^e)\|^2\big] \nonumber \\
     & +\frac{1}{2}\eta_l (1-\beta \eta_l \Bar{E}_l^e) \mathbb E\bigg[\bigg\| \sum_{i=1}^n w_i^e \sum_{k=0}^{E_i-1} \big(\nabla f_i(\thetav_{i,k}^e)-\nabla f(\thetav^e)\big)\bigg\|^2\bigg] \nonumber \\
     & + \frac{\beta \eta_l^2}{2} \mathbb E \bigg[\bigg\|\sum_{i=1}^n w_i^e \sum_{k=0}^{E_i-1} \big(\Tilde{g}_i(\thetav_{i,k}^e) - \nabla f(\thetav^{e})\big)\bigg\|^2\bigg],
\end{align} 

where the constant inequality in the first line is achieved from our assumption that $\eta_l \leq \frac{1}{6 \beta E_i}$ (and consequently: $\eta_l \leq \frac{1}{6 \beta \Bar{E}_l^e}$):

\begin{align}\label{eq:linear_term_grad_norm}
\frac{1}{2}\eta_l (1 - {\beta} \eta_l \Bar{E}_l^e) - \eta_l \Bar{E}_l^e (1 - \frac{{\beta}\eta_l}{2}\Bar{E}_l^e ) &= -\eta_l \left(\Bar{E}_l^e - \frac{1}{2} - \frac{{\beta} \eta_l}{2}\Bar{E}_l^{e^2} + \frac{\beta \eta_l \Bar{E}_l^e}{2} \right) \nonumber \\
&\leq -\eta_l \left( \frac{11\Bar{E}_l^e - 6}{12} + \frac{{\beta} \eta_l \Bar{E}_l^e}{2}\right) \nonumber \\
&\leq -\eta_l \frac{11\Bar{E}_l^e - 6}{12}.
\end{align}

Therefore,

\begin{align}
     \mathbb E \big[f(\thetav^{e+1})\big] &\leq \mathbb E \big[f(\thetav^{e})\big] -\eta_l \frac{11 \Bar{E}_l^e-6}{12} \mathbb E \big[\|\nabla f(\thetav^e)\|^2\big] \nonumber \\
     & +\frac{1}{2}\eta_l (1-\beta \eta_l \Bar{E}_l^e) \mathbb E\bigg[\bigg\| \sum_{i=1}^n w_i^e \sum_{k=0}^{E_i-1} \big(\nabla f_i(\thetav_{i,k}^e)-\nabla f(\thetav^e)\big)\bigg\|^2\bigg] \nonumber \\
     & + \frac{\beta \eta_l^2}{2} \mathbb E \bigg[\bigg\|\sum_{i=1}^n w_i^e \sum_{k=0}^{E_i-1} \big(\Tilde{g}_i(\thetav_{i,k}^e) - \nabla f(\thetav^{e})\big)\bigg\|^2\bigg].
\end{align} 

Now, we use the relaxed triangle inequality $\|a+b\|^2 \leq 2(\|a\|^2+\|b\|^2)$
for the last line above:

\begin{align}\label{eq:mathcalAandB}
     \mathbb E \big[f(\thetav^{e+1})\big] &\leq \mathbb E \big[f(\thetav^{e})\big] -\eta_l \frac{11 \Bar{E}_l^e-6}{12} \mathbb E \big[\|\nabla f(\thetav^e)\|^2\big] \nonumber \\
     & +\frac{1}{2}\eta_l (1-\beta \eta_l \Bar{E}_l^e) \underbrace{\mathbb E\bigg[\bigg\| \sum_{i=1}^n w_i^e \sum_{k=0}^{E_i-1} \big(\nabla f_i(\thetav_{i,k}^e)-\nabla f(\thetav^e)\big)\bigg\|^2\bigg]}_{\mathcal{B}} \nonumber \\
     & + \beta \eta_l^2 \underbrace{\mathbb E \bigg[\bigg\|\sum_{i=1}^n w_i^e \sum_{k=0}^{E_i-1} \big(\Tilde{g}_i(\thetav_{i,k}^e) - \nabla f_i(\thetav_{i,k}^{e})\big)\bigg\|^2\bigg]}_{\mathcal{A}}  + \beta \eta_l^2 \underbrace{\mathbb E \bigg[\bigg\|\sum_{i=1}^n w_i^e \sum_{k=0}^{E_i-1} \big(\nabla f_i(\thetav_{i,k}^e) - \nabla f(\thetav^{e})\big)\bigg\|^2\bigg]}_{\mathcal{B}}
\end{align}

Now, we bound each of the terms $\mathcal{A}$ and $\mathcal{B}$ separately:

\begin{align}
    \mathcal{A} &\leq \mathbb E \bigg[\bigg(\sum_{i=1}^n w_i^e \sum_{k=0}^{E_i-1} \big\|\Tilde{g}_i(\thetav_{i,k}^e) - \nabla f_i(\thetav_{i,k}^{e})\big\|\bigg)^2\bigg] \leq \mathbb E \bigg[\sum_{i=1}^n (w_i^e)^2 \times \sum_{i=1}^n\bigg(\sum_{k=0}^{E_i-1} \big\|\Tilde{g}_i(\thetav_{i,k}^e) - \nabla f_i(\thetav_{i,k}^{e})\big\|\bigg)^2\bigg]\nonumber \\
    &= \mathbb E \bigg[\|\wv^e\|^2 \sum_{i=1}^n\bigg(\sum_{k=0}^{E_i-1} \big\|\Tilde{g}_i(\thetav_{i,k}^e) - \nabla f_i(\thetav_{i,k}^{e})\big\|\bigg)^2\bigg]= \mathbb E \bigg[\sum_{i=1}^n\bigg(\sum_{k=0}^{E_i-1} \big\|\Tilde{g}_i(\thetav_{i,k}^e) - \nabla f_i(\thetav_{i,k}^{e})\big\|\bigg)^2\bigg]\nonumber \\
    &\leq \sum_{i=1}^n E_i\sum_{k=0}^{E_i-1} \mathbb E \bigg[\bigg\|\Tilde{g}_i(\thetav_{i,k}^e) - \nabla f_i(\thetav_{i,k}^{e})\bigg\|^2\bigg] \leq \sum_{i=1}^n E_i^2 \sigma_{i, \Tilde{g}}^2,
\end{align}

where in the first and second inequalities, we used Cauchy-Schwarz inequality. In the last inequality, we used \Cref{eq:newvariance}. Similarly, we can bound $\mathcal{B}$:

\begin{align}
    \mathcal{B} & = \mathbb E \bigg[\bigg\|\sum_{i=1}^n w_i^e \sum_{k=0}^{E_i-1} \big(\nabla f_i(\thetav_{i,k}^e) - \nabla f(\thetav^{e})\big)\bigg\|^2\bigg] = \mathbb E \bigg[\bigg\|\sum_{i=1}^n w_i^e \sum_{k=0}^{E_i-1} \nabla f_i(\thetav_{i,k}^e) - \sum_{i=1}^n w_i^e \sum_{k=0}^{E_i-1} \nabla f(\thetav^e)\bigg\|^2\bigg] \nonumber \\
    &  = \mathbb E \bigg[\bigg\|\sum_{i=1}^n w_i^e \sum_{k=0}^{E_i-1} \nabla f_i(\thetav_{i,k}^e) - \underbrace{\big(\sum_{i=1}^n w_i^e E_i \big)}_{\Bar{E}_l^e} \nabla f(\thetav^e)\bigg\|^2\bigg] = \mathbb E \bigg[\bigg\|\bigg(\sum_{i=1}^n w_i^e \sum_{k=0}^{E_i-1} \nabla f_i(\thetav_{i,k}^e)\bigg) - \Bar{E}_l^e \nabla f(\thetav^e)\bigg\|^2\bigg].
\end{align}
\textbf{Let us also define $\Delta_i^e := w_i^e - \lambda_i$ for client $i$} to be the difference between the aggregation weight of client $i$ in round $e$ ($w_i^e$) and its corresponding aggregation weights in the global objective function $f(\thetav)$ ($\lambda_i$). With this definition and that $\nabla f(\thetav^e) = \sum_{i=1}^n \lambda_i \nabla f_i(\thetav^e)$, we have:

\begin{align}\label{eq:boundonB}
    \mathcal{B} & = \mathbb E \bigg[\bigg\|\bigg(\sum_{i=1}^n \Delta_i^e \sum_{k=0}^{E_i-1} \nabla f_i(\thetav_{i,k}^e)\bigg) + \bigg(\sum_{i=1}^n \lambda_i \sum_{k=0}^{E_i-1} \nabla f_i(\thetav_{i,k}^e)\bigg) - \bigg(\sum_{i=1}^n \lambda_i \Bar{E}_l^e \nabla f_i(\thetav^e)\bigg) \bigg\|^2\bigg] \nonumber \\ 
    & \leq \underbrace{2 \mathbb E \bigg[\bigg\|\sum_{i=1}^n \Delta_i^e \sum_{k=0}^{E_i-1} \nabla f_i(\thetav_{i,k}^e)\bigg \|^2 \bigg]}_{\mathcal{C}} + \underbrace{2\mathbb E \bigg[\bigg \|\bigg(\sum_{i=1}^n \lambda_i \sum_{k=0}^{E_i-1} \nabla f_i(\thetav_{i,k}^e)\bigg) - \bigg(\sum_{i=1}^n \lambda_i \Bar{E}_l^e \nabla f_i(\thetav^e)\bigg) \bigg\|^2\bigg]}_{\mathcal{D}}.
\end{align}

Now, we bound each of the terms $\mathcal{C}$ and $\mathcal{D}$, separately:
\begin{align}
    \mathcal{C} &= 2 \mathbb E \bigg[\bigg\|\sum_{i=1}^n \Delta_i^e \sum_{k=0}^{E_i-1} \nabla f_i(\thetav_{i,k}^e)\bigg \|^2 \bigg] \nonumber \\
    &\leq 4 \mathbb E \bigg[\bigg\|\sum_{i=1}^n \Delta_i^e \sum_{k=0}^{E_i-1} \big(\nabla f_i(\thetav_{i,k}^e)-\nabla f_i(\thetav^e)\big)\bigg \|^2 \bigg] + 4 \mathbb E \bigg[\bigg\|\sum_{i=1}^n E_i \Delta_i^e \nabla f_i(\thetav^e)\bigg \|^2 \bigg]\nonumber \\
    & 
    \leq  4 \mathbb E\bigg[(\sum_{i=1}^n E_i) \sum_{i=1}^n \sum_{k=0}^{E_i-1} |\Delta_i^e|^2 \|\nabla f_i(\thetav_{i,k}^e) -\nabla f_i(\thetav^e)\|^2\bigg] + 4 \mathbb E \bigg[ n \sum_{i=1}^n \bigg\| E_i \Delta_i^e \nabla f_i(\thetav^e)\bigg \|^2 \bigg]\nonumber \\
    & 
    \leq  4(\sum_{i=1}^n E_i) \beta^2 \sum_{i=1}^n \sum_{k=0}^{E_i-1} |\Delta_i^e|^2 \mathbb E[\|\thetav_{i,k}^e -\thetav^e\|^2] + 4nL_0^2\sum_{i=1}^n E_i^2 \mathbb E[|\Delta_i^e|^2] \nonumber \\
    &\leq  4 \beta^2(\sum_{i=1}^n E_i) \sum_{i=1}^n \sum_{k=0}^{E_i-1} \mathbb E[\|\thetav_{i,k}^e -\thetav^e\|^2] + 4nL_0^2\sum_{i=1}^n E_i^2 \Eb[|\Delta_i^e|^2],
\end{align}
where in the third line, we have used relaxed triangle inequality, and in the fourth line, we have used $\beta$-smoothness and $L_0$-Lipschitz continuity of $f_i$. Also, in the last line we used $|\Delta_i^e| \leq 1$. Similarly:
\begin{align}
    \mathcal{D} &= 2\mathbb E \bigg[\bigg \|\sum_{i=1}^n \lambda_i\bigg( \sum_{k=0}^{E_i-1} \nabla f_i(\thetav_{i,k}^e) - \Bar{E}_l^e \nabla f_i(\thetav^e)\bigg) \bigg\|^2\bigg] \nonumber \\
    & \leq 2 \|\lambdav\|^2\sum_{i=1}^n \mathbb E \bigg[\bigg \| \sum_{k=0}^{E_i-1} \nabla f_i(\thetav_{i,k}^e) - \Bar{E}_l^e \nabla f_i(\thetav^e) \bigg\|^2\bigg]\nonumber \\
    & \leq 2 \|\lambdav\|^2\sum_{i=1}^n \mathbb E \bigg[\bigg \| \sum_{k=0}^{E_i-1} \big(\nabla f_i(\thetav_{i,k}^e)- \nabla f_i(\thetav^e)\big) + \big(E_i - \Bar{E}_l^e\big) \nabla f_i(\thetav^e) \bigg\|^2\bigg] \nonumber \\
    & \leq 4 \|\lambdav\|^2\sum_{i=1}^n \mathbb E \bigg[\bigg \| \sum_{k=0}^{E_i-1} \nabla f_i(\thetav_{i,k}^e)- \nabla f_i(\thetav^e)\bigg\|^2 + \big(E_i - \Bar{E}_l^e\big)^2 \underbrace{\big\| \nabla f_i(\thetav^e) \big\|^2}_{\leq L_0^2}\bigg] \nonumber \\
    & \leq 4 \beta^2 \|\lambdav\|^2\sum_{i=1}^n E_i \sum_{k=0}^{E_i-1} \mathbb E \big[\big \|\thetav_{i,k}^e- \thetav^e\big\|^2 \big]+ 4L_0^2 \|\lambdav\|^2 \sum_{i=1}^n \Eb[(E_i - \Bar{E}_l^e)^2].
\end{align}

In the first inequality, we used convexity of the norm function, and Cauchy-Schwarz inequality. Hence, by plugging the bounds above on $\mathcal{C}$ and $\mathcal{D}$ into 
\Cref{eq:boundonB}, we get:

\begin{align}
    \mathcal{B} & \leq 4 \beta^2 (1+\sum_{i=1}^n E_i) \bigg(\sum_{i=1}^n E_i \sum_{k=0}^{E_i-1} \mathbb E \big[\big\|\thetav_{i,k}^e - \thetav^{e}\big\|^2\big]\bigg) + 4L_0^2 \bigg( n\sum_{i=1}^n E_i^2 \Eb [|\Delta_i^e|^2] + \|\lambdav\|^2 \sum_{i=1}^n \Eb[(E_i-\Bar{E}_l^e)^2]\bigg)
\end{align}

In the following, we first simplify \Cref{eq:mathcalAandB}, and then, we plugg the bounds above on $\mathcal{A}$ and $\mathcal{B}$ in it. We have:
\begin{align}
     \mathbb E \big[f(\thetav^{e+1})\big] &\leq \mathbb E \big[f(\thetav^{e})\big] -\eta_l \frac{11 \Bar{E}_l^e-6}{12} \mathbb E \big[\|\nabla f(\thetav^e)\|^2\big] \nonumber \\
     & + \beta \eta_l^2 \underbrace{\mathbb E \bigg[\bigg\|\sum_{i=1}^n w_i^e \sum_{k=0}^{E_i-1} \big(\Tilde{g}_i(\thetav_{i,k}^e) - \nabla f_i(\thetav_{i,k}^{e})\big)\bigg\|^2\bigg]}_{\mathcal{A}}  \nonumber \\
     & + \underbrace{(\beta \eta_l^2 + \frac{1}{2}\eta_l (1-\beta\eta_l\Bar{E}_l^e))}_{< \frac{2}{3}\eta_l} \underbrace{\mathbb E \bigg[\bigg\|\sum_{i=1}^n w_i^e \sum_{k=0}^{E_i-1} \big(\nabla f_i(\thetav_{i,k}^e) - \nabla f(\thetav^{e})\big)\bigg\|^2\bigg]}_{\mathcal{B}},
\end{align} 
where from the assumption $\eta_l \leq \frac{1}{6 \beta E_i}$, we get to $\frac{\beta \eta_l^2}{2} \leq \frac{\eta_l}{12}$. Hence:
\begin{align}
    \beta \eta_l^2 + \frac{1}{2}\eta_l (1-\beta\eta_l\Bar{E}_l^e) = \beta \eta_l^2(1-\frac{\Bar{E}_l^e}{2}) + \frac{\eta_l}{2} \leq \frac{\beta \eta_l^2}{2} + \frac{\eta_l}{2} \leq \frac{\eta_l}{12} + \frac{\eta_l}{2} < \frac{2\eta_l}{3}.
\end{align}

Therefore, by plugging in the bounds on $\mathcal{A}$ and $\mathcal{B}$, we have:
\begin{align}\label{eq:inequalitywithdrift}
     \mathbb E \big[f(\thetav^{e+1})\big] & \leq \mathbb E \big[f(\thetav^{e})\big] -\eta_l \frac{11 \Bar{E}_l^e-6}{12} \mathbb E \big[\|\nabla f(\thetav^e)\|^2\big] + \beta \eta_l^2  \sum_{i=1}^n E_i^2 \sigma_{i, \Tilde{g}}^2  \nonumber \\
     & + \bigg( \frac{8}{3} \beta^2 \eta_l (1+\sum_{i=1}^n E_i) \bigg(\sum_{i=1}^n E_i \sum_{k=0}^{E_i-1} \mathbb E \big[\big\|\thetav_{i,k}^e - \thetav^{e}\big\|^2\big]\bigg)\bigg)\nonumber \\
     & + \bigg(\frac{8}{3}L_0^2 \eta_l \bigg( n\sum_{i=1}^n E_i^2 \Eb[|\Delta_i^e|^2] + \|\lambdav\|^2 \sum_{i=1}^n \Eb[(E_i-\Bar{E}_l^e)^2]\bigg)\bigg).
\end{align}

We now have the following lemma to bound local drift of clients during each communication round $e$:

\begin{restatable}[\textbf{Bounded local drifts}]{lemma}{localdrift}\label{lemma:localdrift}
Suppose Assumption \ref{assump:lipschitz_smooth_bounded} holds. The local drift happening at client $i$ during communication round $e$ is bounded:
\begin{align}
    \xi_i^e:= \sum_{k=0}^{E_i-1} \mathbb E \big[\big\|\thetav_{i,k}^e - \thetav^{e}\big\|^2\big] \leq (\texttt{cte} - 2) E_i^2 \eta_l^2 \left(\sigma_{i, \Tilde{g}^2} + 6 E_i \sigma^2 +  6E_i \Eb [\|\nabla f(\thetav^e)\|^2] \right),
\end{align}
where $\texttt{cte}$ is the mathematical constant $e$.
\begin{proof}

From $\thetav_{i,0}^e = \thetav^e$, we only need to focus on $E_i \geq 2$. We have:
\begin{align}\label{eq:local_updates}
\mathbb E\|\thetav_{i, k}^e - \thetav^e \|^2 &= \mathbb E[\|\thetav_{i,k-1}^e - \eta_l \Tilde{g}_i(\thetav_{i,k-1}^e)- \thetav^e\|^2] \nonumber \\
& \leq \mathbb E[\|\thetav_{i,k-1}^e - \eta_l \nabla f_i(\thetav_{i,k-1}^e)- \thetav^e\|^2] + \eta_l^2 \sigma_{i, \Tilde{g}}^2
\end{align}
where the inequality comes from \Cref{lemma:meanvar}. The first term on the right side of the above inequality can be bounded as: 
\begin{align}\label{eq:local_diff}
\mathbb E[\|\thetav_{i,k-1}^e - \eta_l \nabla f_i(\thetav_{i,k-1}^e)- \thetav^e\|^2] \leq \left( 1 + \frac{1}{2E_i - 1}\right) \mathbb E[\|\thetav_{i,k-1}^e - \thetav^e\|^2] + 2E_i \eta_l^2 \mathbb E\| [\nabla f_i(\thetav_{i, k-1}^e)\|^2],
\end{align}
where we have used \Cref{lemma:relaxed_triangle}. Now, we bound the last term  in the above inequality. We have:
\begin{align}
\nabla f_i(\thetav_{i, k-1}^e) = (\nabla f_i(\thetav_{i, k-1}^e) - \nabla f_i(\thetav^e)) + (\nabla f_i(\thetav^e) - \nabla f(\thetav^e)) + \nabla f(\thetav^e),
\end{align}
By using relaxed triangle inequality (\Cref{lemma:relaxed_triangle}) and Assumption \ref{assump:lipschitz_smooth_bounded}, we get:
\begin{align}\label{eq:local_norm_square}
\|\nabla f_i(\thetav_{i, k-1}^e)\|^2 &= 3\|\nabla f_i(\thetav_{i, k-1}^e) - \nabla f_i(\thetav^e)\|^2 + 3\|\nabla f_i(\thetav^e) - \nabla f(\thetav^e)\|^2 + 3\|\nabla f(\thetav^e)\|^2\nonumber \\
& \leq 3\beta^2 \|\thetav_{i, k-1}^e - \thetav^e\|^2 + 3\sigma^2 + 3\|\nabla f(\thetav)\|^2.
\end{align}

Now, we can rewrite \Cref{eq:local_diff} and then \Cref{eq:local_updates}:
\begin{align}
\mathbb E\|\thetav_{i, k}^e - \thetav^e \|^2 &\leq \underbrace{\left(1 + \frac{1}{2E_i - 1} + 6E_i \beta^2 \eta_l^2\right)}_{\leq 1+\frac{1}{E_i}} \mathbb E[\|\thetav_{i,k-1}^e - \thetav^e\|^2] \nonumber \\
& + \eta_l^2(6 E_i \sigma^2+\sigma_{i, \Tilde{g}}^2) + 6 E_i \eta_l^2 \mathbb E\| \nabla f(\thetav^e)\|^2 \nonumber \\
& \leq  (1+\frac{1}{E_i}) \mathbb E[\|\thetav_{i,k-1}^e - \thetav^e\|^2] + \eta_l^2(6 E_i \sigma^2+\sigma_{i, \Tilde{g}}^2) + 6 E_i \eta_l^2 \mathbb E [\| \nabla f(\thetav^e)\|^2]
\end{align}

From the inequality above and that $\mathbb E\|\thetav_{i, 0}^e - \thetav^e \|^2 = 0$, we have:
\begin{align}
    &\mathbb E\|\thetav_{i, 1}^e - \thetav^e \|^2 \leq \gamma \nonumber\\
    &\mathbb E\|\thetav_{i, 2}^e - \thetav^e \|^2 \leq (1+\frac{1}{E_i})\gamma + \gamma \nonumber\\
    &\mathbb E\|\thetav_{i, 3}^e - \thetav^e \|^2 \leq (1+\frac{1}{E_i})^2\gamma + (1+\frac{1}{E_i})\gamma + \gamma \nonumber\\
    &\ldots \nonumber\\
    &\mathbb E\|\thetav_{i, k}^e - \thetav^e \|^2 \leq (1+\frac{1}{E_i})^{(k-1)}\gamma + \ldots + (1+\frac{1}{E_i})^2\gamma + (1+\frac{1}{E_i})\gamma + \gamma,
\end{align}
where $\gamma = \eta_l^2(6 E_i \sigma^2+\sigma_{i, \Tilde{g}}^2) + 6 E_i \eta_l^2 \mathbb E [\| \nabla f(\thetav^e)\|^2]$. 
By using $1 + q + \dots + q^{n-1} = \frac{q^n - 1}{q-1}$, we get:
\begin{align}
    \mathbb E\|\thetav_{i, k}^e - \thetav^e \|^2 \leq E_i\bigg(\big(1+\frac{1}{E_i}\big)^k - 1\bigg)\big(\eta_l^2(6 E_i \sigma^2+\sigma_{i, \Tilde{g}}^2) + 6 E_i \eta_l^2 \mathbb E [\| \nabla f(\thetav^e)\|^2]\big).
\end{align}
Therefore, we have:

\begin{align}
\sum_{k=0}^{E_i -1} \mathbb E\|\thetav_{i, k}^e - \thetav^e \|^2 &\leq E_i^2 \bigg(\underbrace{\big(1+\frac{1}{E_i}\big)^{E_i}}_{\leq \texttt{cte}} - 2\bigg)\big(\eta_l^2(6 E_i \sigma^2+\sigma_{i, \Tilde{g}}^2) + 6 E_i \eta_l^2 \mathbb E[\| \nabla f(\thetav^e)\|^2]\big) \nonumber \\
& \leq (\texttt{cte}-2) E_i^2 \eta_l^2 \big(6 E_i \sigma^2+\sigma_{i, \Tilde{g}}^2 + 6 E_i \mathbb E[\| \nabla f(\thetav^e)\|^2]\big), 
\end{align}
where $E_i \geq 2$ and $\texttt{cte}$ above is the mathematical constant $e$.
\end{proof}
\end{restatable}

We can now plug the bound on local drifts into \Cref{eq:inequalitywithdrift} and get:
\begin{align}\label{eq:inequalitywithpsi}
     \mathbb E \big[f(\thetav^{e+1})\big] & \leq \mathbb E \big[f(\thetav^{e})\big] - \eta_l \bigg (\underbrace{\frac{11 \Bar{E}_l^e-6}{12} - 12 \beta^2\eta_l^2 (1+\sum_{i=1}^n E_i)\big(\sum_{i=1}^n E_i^4\big)}_{\geq \frac{11 \Bar{E}_l^e-7}{12}} \bigg) \mathbb E \big[\|\nabla f(\thetav^e)\|^2\big] \nonumber \\
     & + 6 \beta^2 \eta_l^3(1+\sum_{i=1}^n E_i) \bigg( 2 \sum_{i=1}^n E_i^4 \sigma^2 + \frac{1}{3} \sum_{i=1}^n E_i^3 \sigma_{i, \Tilde{g}}^2 \bigg) + \beta \eta_l^2 \sum_{i=1}^n E_i^2 \sigma_{i, \Tilde{g}}^2\nonumber\\
     & + \frac{8}{3}L_0^2 \eta_l \bigg(n \sum_{i=1}^n E_i^2 \Eb[(w_i^e - \lambda_i)^2] + \|\lambdav\|^2 \sum_{i=1}^n \Eb[(E_i-\Bar{E}_l^e)^2] \bigg),
\end{align} 
where we have used the second condition on $\eta_l$ in the first line to bound the multiplicative factor.

Hence, we have:
\begin{align}
     &\mathbb E \big[f(\thetav^{e+1})\big] \leq \nonumber \\
     &\mathbb E \big[f(\thetav^{e})\big] - \eta_l \bigg (\frac{11 \Bar{E}_l^e-7}{12}\bigg) \mathbb E \big[\|\nabla f(\thetav^e)\|^2\big] \nonumber \\
     & + \eta_l \underbrace{ \bigg(6 \beta^2 \eta_l^2(1+\sum_{i=1}^n E_i) \bigg( 2 \sum_{i=1}^n E_i^4 \sigma^2 + \frac{1}{3} \sum_{i=1}^n E_i^3 \sigma_{i, \Tilde{g}}^2 \bigg) + \beta \eta_l  \sum_{i=1}^n E_i^2 \sigma_{i, \Tilde{g}}^2 \bigg)}_{\Psi_{\sigma}} \nonumber\\
     & + \eta_l\underbrace{\frac{8 L_0^2}{3} \bigg( n\sum_{i=1}^n E_i^2 \Eb[(w_i^e - \lambda_i)^2]  +  \|\lambdav\|^2 \sum_{i=1}^n \Eb[(E_i-\Bar{E}_l^e)^2]\bigg)}_{\Psi_{p}}.\label{eq:psi_ineq}
\end{align}
Remind that $\Bar{E}_l^e=\sum_{i=1}^n w_i^e E_i$ is a weighted average of clients' number of local gradient steps. From above, we have:
\begin{align}\label{eq:finalbound}
\eta_l (\frac{11\Bar{E}_l^e - 7}{12}) \Eb \|\nabla f(\thetav^e)\|^2 \leq \Eb [f(\thetav^e) - f(\thetav^{e+1})] + (\Psi_{\sigma}+ \Psi_{\textit{p}})\eta_l.
\end{align}

We can now replace $\Bar{E}_l^e$, which is a weighted average of $\{E_i\}_{i=1}^n$ in round $e$, with $E_l^{\texttt{min}} = \min_{i} \{E_i\}_{i=1}^n$, and the inequality still holds:

\begin{align}
\eta_l (\frac{11E_l^{\texttt{min}} - 7}{12}) \Eb \|\nabla f(\thetav^e)\|^2 \leq \Eb [f(\thetav^e) - f(\thetav^{e+1})] + (\Psi_{\textit{p}} +\Psi_{\sigma})\eta_l.
\end{align}

By summing both sides of the above inequality over $e=0, \ldots, E-1$ and dividing by $E$, we get:
\begin{align}
    \min_{0\leq e \leq E-1} \mathbb E[\|\nabla f(\thetav^e)\|^2] &\leq \frac{12}{11E_l^{\texttt{min}} - 7} \bigg(\frac{f(\thetav^0)-f^*}{E \eta_l}+  \Psi_{\sigma} + \Psi_{\textit{p}}\bigg),
\end{align}
which completes the proof.
\end{proof}

\clearpage
\newpage
\section{Detailed results}

\subsection{Test accuracy comparison}
In \Cref{table:mnist} to \ref{table:cifar100}, we report the detailed test accuracy values for all algorithms, on all datasets and privacy distributions we study in this work. The results show that Robust-HDP is consistently outperforming the state-of-the-art algorithms across various datasets.  
\begin{table}[hbt!]
\centering
\caption{Comparison of different algorithms (on MNIST, $E=200$). FedAvg achieves $98.6\%$.}
\label{table:mnist}
\begin{tabular}{l|*{8}{c}c}\toprule
\diagbox{alg}{distr}
&\makebox[2.5em]{\footnotesize Dist1}
&\makebox[2.5em]{\footnotesize Dist2}
&\makebox[2.5em]{\footnotesize Dist3}
&\makebox[2.5em]{\footnotesize Dist4}
&\makebox[2.5em]{\footnotesize Dist5}
&\makebox[2.5em]{\footnotesize Dist6}
&\makebox[2.5em]{\footnotesize Dist7}
&\makebox[2.5em]{\footnotesize Dist8}
&\makebox[2.5em]{\footnotesize Dist9}\\
\midrule\midrule
\algname{WeiAvg} \cite{Liu2021ProjectedFA} & \tt \bf 90.08& 88.29& 89.74& 88.20& 84.94 & 81.40& 84.43& 78.71&81.38\\\midrule
\footnotesize \algname{PFA}\cite{Liu2021ProjectedFA} & 88.24 & 87.93 & 88.35 & 88.32 &  85.65 & 82.16 & 83.71& 80.25 & 78.51 \\\midrule
\algname{DPFedAvg} \cite{DPSCAFFOLD2022} & 84.24& 82.84& 83.50& 80.43& 83.02& 75.69& \bf 85.71& 70.58& 80.49\\\midrule
\algname{minimum} $\epsilon$ \cite{Liu2021ProjectedFA}& 77.80& 74.86& 74.86& 71.75& 68.42& 34.32& 77.62& 56.10& 68.44\\\midrule \midrule 
\algname{Robust-HDP} & 
89.83 &
\bf90.71& 
\bf 89.83& 
\bf89.38& 
\bf87.52&
\bf84.60&
84.03&
\bf81.19& 
\bf81.52\\
\bottomrule
\end{tabular}
\vspace{-2mm}
\end{table}

\begin{table}[hbt!]
\centering
\caption{Comparison of different algorithms (on FMNIST, $E=200$). FedAvg achieves $90.28\%$.}
\begin{tabular}{l|*{8}{c}c}\toprule
\diagbox{alg}{distr}
&\makebox[2.5em]{\footnotesize Dist1}
&\makebox[2.5em]{\footnotesize Dist2}
&\makebox[2.5em]{\footnotesize Dist3}
&\makebox[2.5em]{\footnotesize Dist4}
&\makebox[2.5em]{\footnotesize Dist5}
&\makebox[2.5em]{\footnotesize Dist6}
&\makebox[2.5em]{\footnotesize Dist7}
&\makebox[2.5em]{\footnotesize Dist8}
&\makebox[2.5em]{\footnotesize Dist9}\\
\midrule \midrule
\algname{WeiAvg} \cite{Liu2021ProjectedFA} & \bf 77.65 & \bf 78.30 & \bf 75.92 & \bf 77.10& 72.38 & 64.15& 66.80 &66.86 & 64.79\\\midrule
\algname{PFA}\cite{Liu2021ProjectedFA} & 75.03& 74.85& 72.90& \bf 77.10& 72.44& 71.08& 66.30& 67.89& 64.69\\\midrule
\algname{DPFedAvg} \cite{DPSCAFFOLD2022} & 74.12 &71.68 &71.97 &68.10 &70.20 &62.46 & 64.15& 65.87& 65.50\\\midrule
\algname{minimum} $\epsilon$ \cite{Liu2021ProjectedFA}&  73.15 & 64.26 & 64.26 & 62.60& 64.35 & 28.66 & 65.13 & 58.44 & 66.36\\\midrule \midrule
\algname{Robust-HDP} & 75.13 & 76.25 & 75.04 & 76.19 & \bf 73.80 &\bf 71.30 & \bf 66.85& \bf 68.32& \bf 66.96\\\midrule

\end{tabular}

\label{table:fmnist}
\end{table}

\begin{table}[hbt!]
\centering
\caption{Comparison of different algorithms (on CIFAR10, $E=200$). FedAvg achieves $73.55\%$.}
\begin{tabular}{l|*{8}{c}c}\toprule
\diagbox{alg}{distr}
&\makebox[2.5em]{\footnotesize Dist1}
&\makebox[2.5em]{\footnotesize Dist2}
&\makebox[2.5em]{\footnotesize Dist3}
&\makebox[2.5em]{\footnotesize Dist4}
&\makebox[2.5em]{\footnotesize Dist5}
&\makebox[2.5em]{\footnotesize Dist6}
&\makebox[2.5em]{\footnotesize Dist7}
&\makebox[2.5em]{\footnotesize Dist8}
&\makebox[2.5em]{\footnotesize Dist9}\\
\midrule \midrule
\algname{WeiAvg} \cite{Liu2021ProjectedFA} & 31.18& 31.18& 29.65& 27.74& 24.25 & 19.91 & 21.93 & 18.91& 20.64\\\midrule
 \algname{PFA}\cite{Liu2021ProjectedFA} & 26.91 & \bf 32.68 & 25.19& 29.21 & 21.63 & 18.93 & 20.63 & 16.27 & 15.75 \\\midrule
\algname{DPFedAvg} \cite{DPSCAFFOLD2022} &  31.51 & 21.51 & 22.28 & 20.50& 21.25& 15.19& 18.27& 16.45 & 18.63\\\midrule
\algname{minimum} $\epsilon$ \cite{Liu2021ProjectedFA}& 26.20 & 16.71 & 16.45 & 15.86 & 14.23& 10.51 & 13.35& 13.32& 14.11\\\midrule \midrule
\algname{Robust-HDP} & \bf 31.97 & 31.70 & \bf 32.0& \bf 30.60 & \bf 24.86 & \bf 23.61 & \bf 24.10& \bf 19.02 & \bf 22.05\\\midrule

\end{tabular}
\vspace{-2mm}

\label{table:cifar10}
\end{table}

\begin{table}[hbt!]
\centering
\caption{Comparison of different algorithms (on CIFAR100, $E=200$). FedAvg achieves $61.80\%$.}
\begin{tabular}{l|*{8}{c}c}\toprule
\diagbox{alg}{distr}
&\makebox[2.5em]{\footnotesize Dist1}
&\makebox[2.5em]{\footnotesize Dist2}
&\makebox[2.5em]{\footnotesize Dist3}
&\makebox[2.5em]{\footnotesize Dist4}
&\makebox[2.5em]{\footnotesize Dist5}
&\makebox[2.5em]{\footnotesize Dist6}
&\makebox[2.5em]{\footnotesize Dist7}
&\makebox[2.5em]{\footnotesize Dist8}
&\makebox[2.5em]{\footnotesize Dist9}\\
\midrule \midrule
\algname{WeiAvg} \cite{Liu2021ProjectedFA} & \bf 35.91 & \bf 36.23 & 32.61 & 30.92 & 29.42 & 27.37 & 27.26 & 27.03 & 26.57\\\midrule
 \algname{PFA}\cite{Liu2021ProjectedFA} & 34.21& 35.86& 30.12& 29.45& 26.95& 31.27& 24.35& 21.29 &18.05 \\\midrule
\algname{DPFedAvg} \cite{DPSCAFFOLD2022} &  31.34 & 31.30 & 25.01 & 26.74 & 24.96 & 21.27 & 21.72 & 17.36 & 17.71\\\midrule \midrule
\footnotesize \algname{minimum} $\epsilon$ \cite{Liu2021ProjectedFA}& 27.61 & 27.25 & 24.91 & 25.02 & 25.07 & 21.21 & 21.46& 17.42 & 17.50\\\midrule
\algname{Robust-HDP} & 33.35 & 33.38 &  \bf 33.46 & \bf 35.03 & \bf 31.58 & \bf 31.33 & \bf 30.27 & \bf 30.71 & \bf 29.21\\\midrule

\end{tabular}
\vspace{-2mm}

\label{table:cifar100}
\end{table}

\subsection{ablation study on privacy level and number of clients}
The results in \cref{table:ablation_privacy_preference} and \cref{table:ablation_num_clients} show the detailed results for the ablation study on privacy level and number of clients, reported in \cref{fig:ablation_privacy_preference} (left and middle figures, respectively). The values are the mean and standard deviation of average test accuracy across clients over three different runs.

\begin{table}[hbt!]
\centering
\caption{Detailed results for ablation study on the privacy level of clients in \cref{fig:ablation_privacy_preference}, left.}
\begin{tabular}{l|*{4}{c}}\toprule
\diagbox{alg}{distr}
&\makebox[2.5em]{\footnotesize Dist2}
&\makebox[2.5em]{\footnotesize Dist4}
&\makebox[2.5em]{\footnotesize Dist6}
&\makebox[2.5em]{\footnotesize Dist8}\\

\midrule \midrule
\algname{WeiAvg} \cite{Liu2021ProjectedFA} &
88.29$_{\pm0.67}$ & 
88.20$_{\pm0.52}$ &
81.60$_{\pm0.58}$ &
78.71$_{\pm0.67}$ \\\midrule

\algname{PFA}\cite{Liu2021ProjectedFA} & 
88.32$_{\pm0.85}$ &
86.91$_{\pm0.97}$ &
82.16$_{\pm0.99}$ &
79.92$_{\pm0.86}$ \\\midrule

\algname{DPFedAvg} \cite{DPSCAFFOLD2022} &
82.84$_{\pm0.65}$ &
80.43$_{\pm1.72}$ &
74.02$_{\pm1.54}$ &
70.58$_{\pm1.43}$
\\\midrule\midrule

\algname{Robust-HDP} &
\bf90.71$_{\pm0.65}$ &
\bf89.38$_{\pm0.76}$ &
\bf85.13$_{\pm0.68}$ &
\bf81.19$_{\pm0.97}$
\\\midrule

\end{tabular}
\label{table:ablation_privacy_preference}
\end{table}

\begin{table}[hbt!]
\centering
\caption{Detailed results for ablation study on number of clients in \cref{fig:ablation_privacy_preference}, middle.}
\begin{tabular}{l|*{3}{c}}\toprule
\diagbox{alg}{distr}
&\makebox[2.5em]{\footnotesize $n=20$}
&\makebox[2.5em]{\footnotesize $n=40$}
&\makebox[2.5em]{\footnotesize $n=60$}\\
\midrule \midrule

\algname{WeiAvg} \cite{Liu2021ProjectedFA} &
81.60$_{\pm0.58}$ &
75.20$_{\pm0.67}$ &
63.12$_{\pm0.78}$
\\\midrule

 \algname{PFA}\cite{Liu2021ProjectedFA} &
 82.16$_{\pm0.99}$ &
 73.15$_{\pm1.02}$ &
 66.0$_{\pm0.98}$
 \\\midrule
 
\algname{DPFedAvg} \cite{DPSCAFFOLD2022} &
74.02$_{\pm1.54}$ &
62.98$_{\pm1.85}$ &
58.49$_{\pm1.67} $
\\\midrule\midrule

\algname{Robust-HDP} &
\bf 85.13$_{\pm0.68}$ &
\bf 76.85$_{\pm0.75}$ &
\bf 72.77$_{\pm0.78}$
\\\midrule

\end{tabular}
\label{table:ablation_num_clients}
\end{table}

\subsection{Precision of \algname{Robust-HDP}}\label{app:RHDP_precision}
In this section, we investigate the precision of \algname{Robust-HDP} in estimating $\{\sigma_i^2\}_{i=1}^n$ and $\{w_i^*\}_{i=1}^n$. We also check the performance of \algname{RPCA} algorithm used by \algname{Robust-HDP}. \Cref{fig:rank}
shows the eigen values of the matrices $\mathbf{M}$ and $\mathbf{L}$ on MNIST at the end of the first global communication round for when clients' privacy parameters are sampled from Dist3 (inducing less \DP noise) and Dist9 (inducing more \DP noise) from \Cref{table:mixture_dists}. We can clearly observe that most of the eigen values of $\mathbf{L}$ returned by \algname{RPCA} are close to $0$, especially for Dist3, i.e., \algname{RPCA} has returned a low-rank matrix as the underlying low-rank matrix in $\mathbf{M}$ for both Dist3 and Dist9.

\begin{figure*}[ht]
\centering
    \includegraphics[width=0.4\columnwidth]{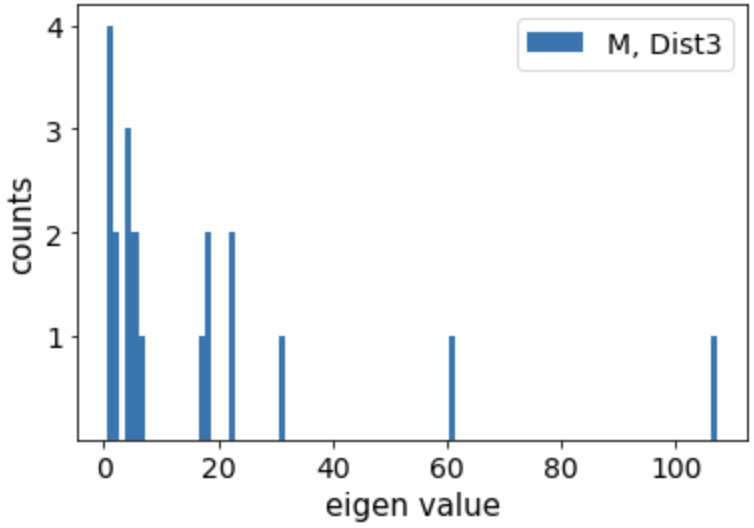}
    \includegraphics[width=0.4\columnwidth]{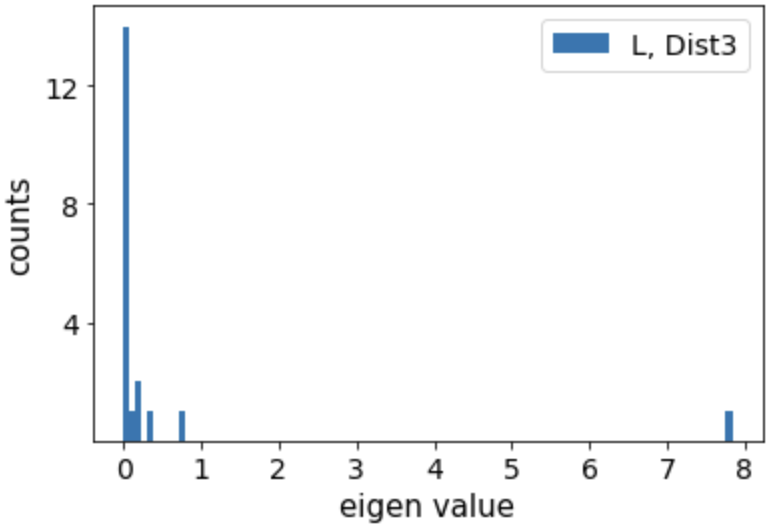}

    \includegraphics[width=0.4\columnwidth]{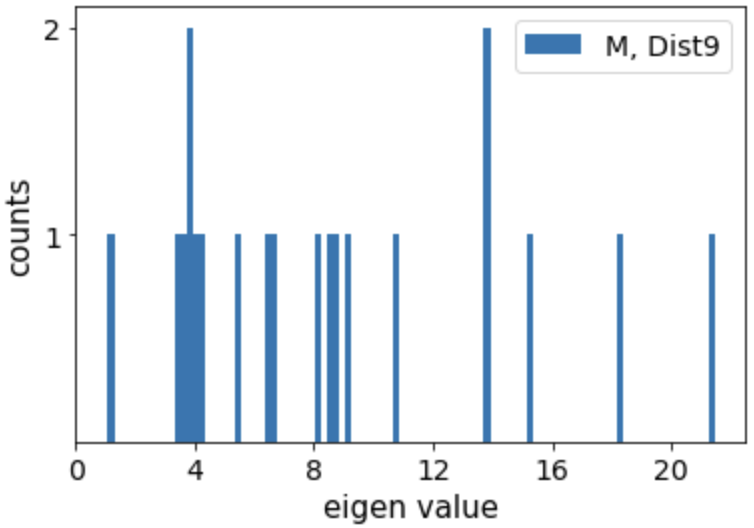}
    \includegraphics[width=0.4\columnwidth]{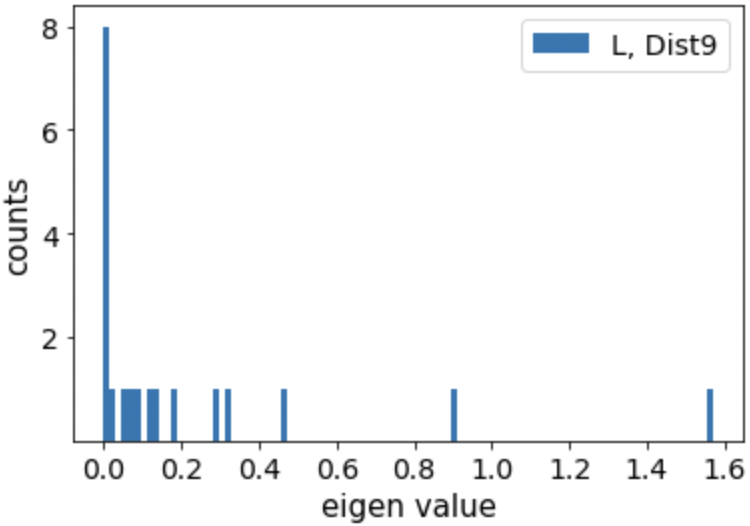}

    \caption{Comparison of the eigen values of matrices $\mathbf{M}$ (left) and $\mathbf{L}$ (right) on MNIST dataset. The concentration of eigen values in the right figures around 0 shows that the matrix $\mathbf{L}$  returned by Robust PCA, is indeed low rank, while $\mathbf{M}$  is not, due to the noise existing in clients model updates. The results for these experiments were reported in \cref{table:mnist} (\algname{Robust-HDP}, Dist3 and Dist9).}
    \label{fig:rank}
\end{figure*}

In \Cref{fig:weights_comparison_details}, we have shown the noise variance estimates $\{\hat{\sigma}_i^2\}_{i=1}^n$ and the aggregation weights $\{w_i\}_{i=1}^n$ returned by \algname{Robust-HDP}, and compared them with their true (optimum) values. We have also shown the weights assigned by other baseline algorithms. Having both privacy and batch size heterogeneity, \algname{Robust-HDP} assigns larger weights to clients with smaller $\epsilon$ and larger batch size (e.g., client 10, which has the largest batch size, has the largest assigned aggregation weight from \algname{Robust-HDP}). The weight assignment of \algname{Robust-HDP} is based on the noise estimates $\{\hat{\sigma}_i^2\}_{i=1}^n$: the larger the $\hat{\sigma}_i^2$, the smaller the assigned weight $w_i$. Also, as observed, the weight assignment of \algname{Robust-HDP} is very close to the optimum wights $\{w_i^*\}_{i=1}^n$. In contrast, WeiAvg and PFA assign weights just based on the privacy parameters $\epsilon_i$ of clients, which is suboptimal. Similarly, \algname{DPFedAvg} assings weights just based on the train set size of clients, which we assumed are uniform for the experiments in the main body of the paper and \cref{fig:weights_comparison_details}. 
We have done similar comparisons in the next section (\Cref{app:additional_exps}) for other heterogeneity scenarios.

\begin{figure*}[ht]
\centering
    \includegraphics[width=0.45\columnwidth]{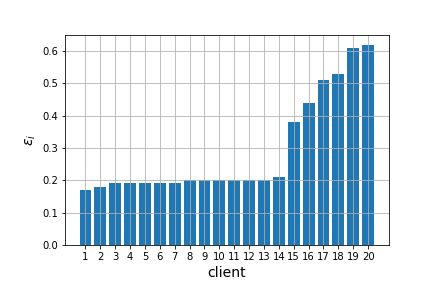}
    \includegraphics[width=0.45\columnwidth]{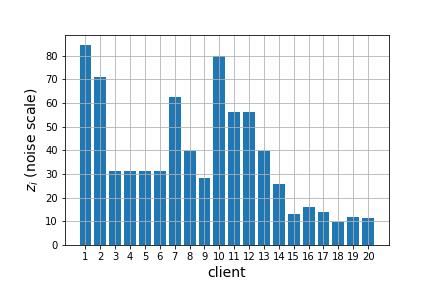}
    \includegraphics[width=0.41\columnwidth]{sigma2_mnist_dist8.png}
    ~~~~~~\includegraphics[width=0.45\columnwidth]{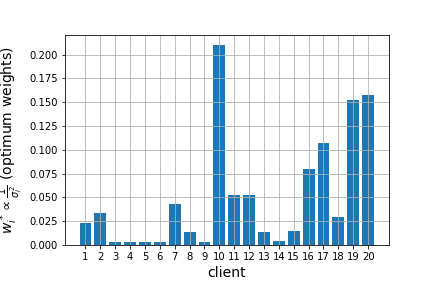}
    \includegraphics[width=0.45\columnwidth]{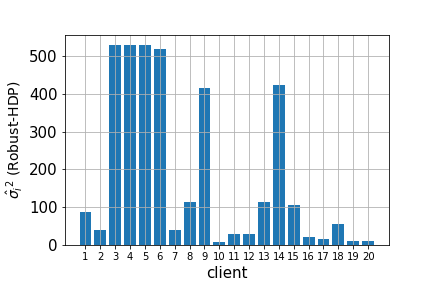}
    \includegraphics[width=0.45\columnwidth]{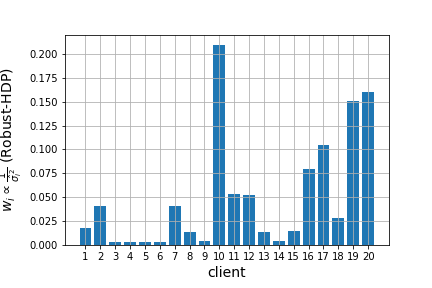}
    \includegraphics[width=0.45\columnwidth]{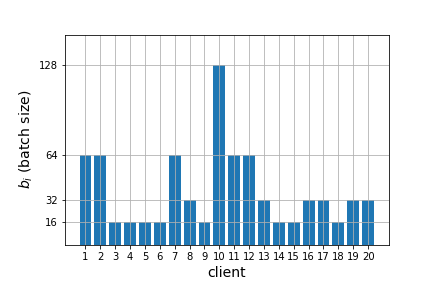}
    \includegraphics[width=0.45\columnwidth]{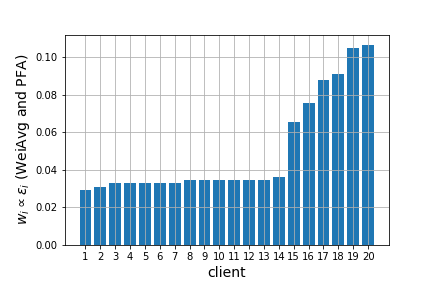}

    \caption{Comparison of weight assignments for Dist8 and MNIST with the data split in \Cref{table:split_uniform}. The assigned weights by baseline algorithms show that their weight assignment strategies are not based on the noise variance in clients model updates, hence suboptimal. The results for this experiment were reported in \cref{table:mnist} (\algname{Robust-HDP}, Dist8).}
    \label{fig:weights_comparison_details}
\end{figure*}

\clearpage
\newpage
\section{Additional Experiments}\label{app:additional_exps}
So far, we assumed heterogeneous batch sizes $\{b_i\}_{i=1}^n$, heterogeneous privacy parameters $\{\epsilon_i\}_{i=1}^n$ and uniform dataset sizes $\{N_i\}_{i=1}^n$. Now, we report and discuss some extra experimental results in this section. We consider three cases: 

\begin{itemize}
    \item uniform batch sizes $\{b_i=b\}$, heterogeneous privacy parameters $\{\epsilon_i\}$ and dataset sizes $\{N_i\}$
    
    \item uniform privacy parameters $\{\epsilon_i = \epsilon\}$, heterogeneous batch sizes $\{b_i\}$ and dataset sizes $\{N_i\}$

    \item uniform batch sizes $\{b_i=b\}$ and uniform privacy parameters $\{\epsilon_i = \epsilon\}$, heterogeneous dataset sizes $\{N_i\}$ (corresponding to regular homogeneous \DPFL setting, which is well-studied in the literature as a separate topic)
    
\end{itemize}
We run experiments on CIFAR10, as it uses a large size model and is more challenging. Unless otherwise stated, we use Dirichlet allocation  \citep{wang2019federated} to get label distribution heterogeneity for the experiments in this section. For all samples in each class $k$, denoted as the set $\Sc_k$, we split $\Sc_k = \Sc_{k,1}\cup \Sc_{k,2} \cup \dots \cup \Sc_{k, n}$ into $n$ clients ($n=20$) according to a symmetric Dirichlet distribution ${\rm Dir}(1)$. Then we gather the samples for client $j$ as $\Sc_{1, j} \cup \Sc_{2, j}\cup \dots \cup \Sc_{C, j}$, if we have $C$ classes in total. This results in different dataset sizes ($N_i$) for different clients. After splitting the data across the clients, we fix it and run the following experiments.

\subsection{uniform batch sizes \texorpdfstring{$\{b_i=b\}$}{Lg}, heterogeneous privacy parameters \texorpdfstring{$\{\epsilon_i\}$}{Lg} and heterogeneous dataset sizes \texorpdfstring{$\{N_i\}$}{Lg}}

Despite the haterogeneity that may exist in the memory budgets and physical batch sizes of clients, they may use gradient accumulation (see \Cref{sec:grad_acc}) to implement \algname{DPSGD} with the same logical batch size. However, such a synchronization can happen only when the untrusted server asks clients to all use a specific logical batch size. Otherwise, if every client decides about its batch size locally, the same batch size heterogeneity that we considered in the main body of the paper will happen again. In the case of such a batch size synchronization by the server, there will be some discrepancy between  the upload times of clients' model updates (as some need to use gradient accumulation with smaller physical batch sizes), which should be tolerated by the server. Having these points in mind, in this subsection, we assume such a batch size synchronization exists and we fix the logical batch size of all clients to the same value of $b=32$ by using gradient accumulation. We also sample their privacy preference parameters $\{\epsilon_i\}$ from \Cref{table:mixture_dists}. In this case, our analysis in \Cref{sec: batchsize_analysis} can be rewritten as follows (as before, we use the same $\delta_i=\delta$ and $K_i=K$ for all clients):

\paragraph{1. Effective clipping threshold:}
when the clipping is indeed effective for all samples, the variance of the noisy stochastic gradient in \Cref{eq:noisy_sg} can be computed as:

\begin{align}
    \mathbb E[\Tilde{g}_i(\thetav)] = \frac{1}{b}\sum_{j \in \mathcal{B}_i^t} \mathbb E[\bar{g}_{ij}(\thetav)] = \frac{1}{b}\sum_{j \in \mathcal{B}_i^t} G_i(\thetav) = G_i(\thetav),
\label{eq:appendix_expectation_gtilde_1}
\end{align} 
\begin{align}
    \sigma_{i, \Tilde{g}}^2 := \texttt{Var}(\Tilde{g}_i(\thetav)) =  \frac{c^2 - \big\| G_i(\thetav)\big\|^2}{b} + \frac{p c^2 z^2(\epsilon_i, \delta, \frac{b}{N_i}, K, E)}{b^2} \approx \frac{p c^2 z^2(\epsilon_i, \delta, \frac{b}{N_i}, K, E)}{b^2},
\label{eq:appendix_var_g_effective_1}
\end{align}

\paragraph{2. Ineffective clipping threshold:}
when the clipping is ineffective for all samples, we have:
\begin{align}
    & \mathbb E[\Tilde{g}_i(\thetav)] = \mathbb E[g_i(\thetav)] = \nabla f_i(\thetav),
    \\
    & \sigma_{i, \Tilde{g}}^2 = \texttt{Var}(\Tilde{g}_i(\thetav)) = \texttt{Var}(g_i(\thetav)) + \frac{p \sigma_{i, \texttt{\DP}}^2}{b^2} \leq \sigma_{i, g}^2 + \frac{p c^2 z^2(\epsilon_i, \delta, \frac{b}{N_i}, K, E)}{b^2},
\label{eq:appendix_var_g_ineffective_1}.
\end{align}

Finally:

\begin{align}
    \sigma_i^2 := \texttt{Var}(\Delta \Tilde{\thetav}_i^e|\thetav^e)
    & = K \cdot \lceil \frac{N_i}{b} \rceil \cdot \eta_l^2 \cdot \sigma_{i, \Tilde{g}}^2.
\end{align}

We observe that, the amount of noise in model updates ($\sigma_i^2$) varies across clients depending on their privacy parameter $\epsilon_i$ \textbf{and dataset size $N_i$}. Also, as observed in \Cref{fig:var_epsilon_b}, noise variance $\sigma_i^2$ does not change linearly with $\epsilon_i$. These altogether show that aggregation strategy $w_i \propto \epsilon_i$ is suboptimal. In contrast, Robust-HDP takes both of the sources of heterogeneity into account by assigning aggregation weights based on an estimation of $\{\sigma_i^2\}$ directly. With these settings, we got the results in \Cref{table:cifar10_uniform_b} on CIFAR10, which shows superiority of Robust-HDP in this heterogeneity scenario.

\begin{table}[t]
\centering
\caption{Comparison of different algorithms on CIFAR10 with uniform batch sizes $\{b_i=b\}$, heterogeneous privacy parameters $\{\epsilon_i\}$ and heterogeneous dataset sizes $\{N_i\}$. FedAvg ($E=200$) achieves $77.58\%$. We have dropped the "minimum $\epsilon$" algorithm due to its very low performance.}
\begin{tabular}{l|*{8}{c}c}\toprule
\diagbox{alg}{distr}
&\makebox[2.5em]{\footnotesize Dist1}
&\makebox[2.5em]{\footnotesize Dist2}
&\makebox[2.5em]{\footnotesize Dist3}
&\makebox[2.5em]{\footnotesize Dist4}
&\makebox[2.5em]{\footnotesize Dist5}
&\makebox[2.5em]{\footnotesize Dist6}
&\makebox[2.5em]{\footnotesize Dist7}
&\makebox[2.5em]{\footnotesize Dist8}
&\makebox[2.5em]{\footnotesize Dist9}\\
\midrule \midrule
\algname{WeiAvg} \cite{Liu2021ProjectedFA} & 31.99& 31.18& 29.59& 25.97& 24.66 & 16.61 & 23.13 & 17.01 & 14.34\\\midrule
 \algname{PFA}\cite{Liu2021ProjectedFA} & 32.12& 32.2 & 30.11& 28.48& 25.09& 16.85& \bf 23.34& 17.11& 15.12\\\midrule
\algname{DPFedAvg} \cite{DPSCAFFOLD2022} & 33.89 & 24.16 & 24.62 & 17.50 & 22.71& 16.76 & 19.52 & 13.81& \bf 16.27 \\\midrule
\algname{Robust-HDP} & \bf 34.94 & \bf 33.78 & \bf 31.34& \bf 32.50 & \bf 26.05 & \bf 17.98 & 23.13 & \bf 17.97 & 15.79\\\midrule

\end{tabular}
\label{table:cifar10_uniform_b}
\end{table}

\subsection{Heterogeneous batch sizes \texorpdfstring{$\{b_i\}$}{Lg}, uniform privacy parameters \texorpdfstring{$\{\epsilon_i = \epsilon\}$}{Lg} and heterogeneous dataset sizes \texorpdfstring{$\{N_i\}$}{Lg}}

In this section, we assume the same values for privacy parameters ($\epsilon_i=\epsilon$), but different batch and dataset sizes. Therefore, we have:

\paragraph{1. Effective clipping threshold:}

\begin{align}
    \mathbb E[\Tilde{g}_i(\thetav)] = \frac{1}{b}\sum_{j \in \mathcal{B}_i^t} \mathbb E[\bar{g}_{ij}(\thetav)] = \frac{1}{b}\sum_{j \in \mathcal{B}_i^t} G_i(\thetav) = G_i(\thetav),
\label{eq:appendix_expectation_gtilde_2}
\end{align} 
\begin{align}
    \sigma_{i, \Tilde{g}}^2 := \texttt{Var}(\Tilde{g}_i(\thetav)) =  \frac{c^2 - \big\| G_i(\thetav)\big\|^2}{b} + \frac{p c^2 z^2(\epsilon, \delta, \frac{b_i}{N_i}, K, E)}{b^2} \approx \frac{p c^2 z^2(\epsilon, \delta, \frac{b_i}{N_i}, K, E)}{b_i^2}.
\label{eq:appendix_var_g_effective_2}
\end{align}

\paragraph{2. Ineffective clipping threshold:}

\begin{align}
    & \mathbb E[\Tilde{g}_i(\thetav)] = \mathbb E[g_i(\thetav)] = \nabla f_i(\thetav),
    \\
    & \sigma_{i, \Tilde{g}}^2 = \texttt{Var}(\Tilde{g}_i(\thetav)) = \texttt{Var}(g_i(\thetav)) + \frac{p \sigma_{i, \texttt{\DP}}^2}{b^2} \leq \sigma_{i, g}^2 + \frac{p c^2 z^2(\epsilon, \delta, \frac{b_i}{N_i}, K, E)}{b_i^2},
\label{eq:appendix_var_g_ineffective_2}
\end{align}

and

\begin{align}
    \sigma_i^2 := \texttt{Var}(\Delta \Tilde{\thetav}_i^e|\thetav^e)
    & = K \times \lceil \frac{1}{q_i} \rceil \cdot \eta_l^2 \cdot \sigma_{i, \Tilde{g}}^2 \approx K \cdot \frac{N_i}{b_i} \cdot \eta_l^2 \cdot \sigma_{i, \Tilde{g}}^2.
\end{align}

Hence, $\sigma_i^2$ varies across clients as a function of both $b_i$ and $N_i$ and heavily depends on $b_i$ ($b_i$ appears with power 3). Despite this heterogeneity in the set $\{\sigma_i^2\}_{i=1}^n$, WeiAvg assigns the same aggregation weights to all clients, due to their privacy parameters being equal, which is clearly inefficient. In contrast, Robust-HDP estimates the values in $\{\sigma_i^2\}_{i=1}^n$ directly and assigns larger weights to clients with larger batch sizes. With these settings and the Dirichlet data allocation mentioned above, we got the results in \Cref{table:cifar10_uniform_eps_hetb}, which shows superiority of Robust-HDP in this case as well. We have used the mean values of the distributions Dist1, Dist3, Dist5, Dist7 and Dist9 from \Cref{table:mixture_dists} for $\epsilon$, i.e., $\epsilon \in \{\tt 2.6, \tt 2.0, \tt 1.1, \tt 0.6, \tt 0.35\}$. Also, as before, we have fixd $\delta_i$ to $\tt 1e-4$.

\begin{table}[t]
\centering
\caption{Comparison of different algorithms on CIFAR10 with heterogeneous batch sizes $\{b_i\}$, uniform privacy parameters $\{\epsilon_i = \epsilon\}$ and heterogeneous dataset sizes $\{N_i\}$. FedAvg ($E=200$) achieves $73.55\%$. "minimum $\epsilon$" algorithm is equivalent to DPFedAvg in this case.}
\begin{tabular}{l|*{4}{c}c}\toprule
\diagbox{alg}{distr}
&\makebox[2.5em]{\footnotesize $\epsilon = 2.6$}
&\makebox[2.5em]{\footnotesize $\epsilon = 2.0$}
&\makebox[2.5em]{\footnotesize $\epsilon = 1.1$}
&\makebox[2.5em]{\footnotesize $\epsilon = 0.6$}
&\makebox[2.5em]{\footnotesize $\epsilon = 0.35$}
\\
\midrule \midrule
\algname{WeiAvg} and \algname{PFA}\cite{Liu2021ProjectedFA} & 35.86 & 33.50 & 29.21 & \bf 24.49 & 18.40 \\\midrule
\algname{DPFedAvg \cite{DPSCAFFOLD2022}} & 37.00& 32.89 & 29.32& 23.06 & 19.14 \\\midrule
\algname{Robust-HDP} & \bf 37.45 & \bf 34.93& \bf 29.78& 23.15& \bf 19.54\\\midrule
\end{tabular}
\label{table:cifar10_uniform_eps_hetb}
\end{table}

\subsection{uniform batch sizes \texorpdfstring{$\{b_i=b\}$}{Lg}, uniform privacy parameters \texorpdfstring{$\{\epsilon_i = \epsilon\}$}{Lg} and heterogeneous dataset sizes \texorpdfstring{$\{N_i\}$}{Lg}}\label{app:uuh}
In this section, other than using the same values for clients batch sizes ($b_i=b$), we fix the privacy parameter of all clients to the same value $\epsilon$ (i.e., we have homogeneous DPFL, for which \algname{DPFedAvg} has been proposed). Therefore, we have:

\paragraph{1. Effective clipping threshold:}

\begin{align}
    \mathbb E[\Tilde{g}_i(\thetav)] = \frac{1}{b}\sum_{j \in \mathcal{B}_i^t} \mathbb E[\bar{g}_{ij}(\thetav)] = \frac{1}{b}\sum_{j \in \mathcal{B}_i^t} G_i(\thetav) = G_i(\thetav),
\label{eq:appendix_expectation_gtilde_3}
\end{align} 
\begin{align}
    \sigma_{i, \Tilde{g}}^2 := \texttt{Var}(\Tilde{g}_i(\thetav)) =  \frac{c^2 - \big\| G_i(\thetav)\big\|^2}{b} + \frac{p c^2 z^2(\epsilon, \delta, \frac{b}{N_i}, K, E)}{b^2} \approx \frac{p c^2 z^2(\epsilon, \delta, \frac{b}{N_i}, K, E)}{b^2}.
\label{eq:appendix_var_g_effective_3}
\end{align}

\paragraph{2. Ineffective clipping threshold:}

\begin{align}
    & \mathbb E[\Tilde{g}_i(\thetav)] = \mathbb E[g_i(\thetav)] = \nabla f_i(\thetav),
    \\
    & \sigma_{i, \Tilde{g}}^2 = \texttt{Var}(\Tilde{g}_i(\thetav)) = \texttt{Var}(g_i(\thetav)) + \frac{p \sigma_{i, \texttt{\DP}}^2}{b^2} \leq \sigma_{i, g}^2 + \frac{p c^2 z^2(\epsilon, \delta, \frac{b}{N_i}, K, E)}{b^2},
\label{eq:appendix_var_g_ineffective_3}
\end{align}

and

\begin{align}
    \sigma_i^2 := \texttt{Var}(\Delta \Tilde{\thetav}_i^e|\thetav^e)
    & = K \cdot \lceil \frac{1}{q_i} \rceil \cdot \eta_l^2 \cdot \sigma_{i, \Tilde{g}}^2 \approx K \cdot \frac{N_i}{b} \cdot \eta_l^2 \cdot \sigma_{i, \Tilde{g}}^2.
\end{align}

Hence $\sigma_i^2$ varies across clients as a function of only $N_i$. In the next paragraph, we show that this variation with $N_i$ is small. This means that when clients hold the same privacy parameter and also use the same batch size, the amount of noise in their model updates sent to the server are almost the same, i.e., $\sigma_i^2 \approx \sigma_j^2, i \neq j$. \textbf{Hence, in this case the problem in \Cref{eq:w_opt} has solution $w_i \approx \frac{1}{n}$}. In the following, we show what is the difference between the solutions provided by different algorithms for this case.

\subsubsection{Performance parity in DPFL systems} Before proceeding to the experimental results, we draw your attention to the weight assignments by Robust-HDP in this setting, where both privacy parameters and batch sizes are uniform. Robust-HDP aims at approximating $\{\sigma_i^2\}$ and:

\begin{align}\label{eq:wi_Ni}
    w_i^*\propto \frac{1}{\sigma_i^2} \approx \frac{b}{K \eta_l^2} \cdot \frac{1}{N_i \sigma_{i, \Tilde{g}}^2} \approx \frac{b^3}{K p c^2\eta_l^2} \cdot \frac{1}{N_i z^2(\epsilon, \delta, \frac{b}{N_i}, K, E)} = \frac{b^3}{K p c^2\eta_l^2} \cdot \frac{1}{H(N_i, b, \epsilon, \delta, K, E)}
\end{align}
where we have used \Cref{eq:var_g_effective} (with $b$ and $\epsilon$) and $H(N_i, b, \epsilon, \delta, K, E) := N_i z^2(\epsilon, \delta, \frac{b}{N_i}, K, E)$. Now note that $z$ decreases with $N_i$ sublinearly (see \Cref{fig:zvsq}. Remember that $q_i=\frac{b_i}{N_i}$). We have plotted the behavior of the function $H(N_i, b, \epsilon, \delta, K, E)$ as a function of $N_i$ in \Cref{fig:HvsN}. Hence, \emph{when $N_i$ decreases, $w_i^*$ increases slowly}. This means that Robust-HDP tries to minimize the noise in the aggregated model parameter (problem \ref{eq:w_opt}) and also assigns slightly larger weights to the clients with smaller datasets. Similarly, WeiAvg assigns uniform weights to all clients. In contrast, the solution provided by DPFedAvg focuses more on clients with larger train sets ($w_i \propto N_i$). Considering the point that $\{\sigma_i^2\}$ is almost uniform, this way it exploits clients with larger train sets during training. In the following, we discuss how this is related to performance fairness across clients.

\begin{figure*}[hbt!]
\centering
    \includegraphics[width=0.45\columnwidth]{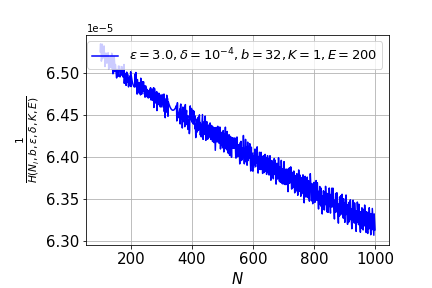}

    \caption{The behavior of function $1/H(N, b, \epsilon, \delta, K, E)$ as a function of $N$ (dataset size) for an instance value of $\epsilon=3$ and batch size $b=32$. Oscillations in the curves are due to finding $z$ empirically.}
    \label{fig:HvsN}
\end{figure*}

 There have been multiple works in the literature , showing that \DP has adverse effects on fairness in ML systems making it impossible to achieve both fairness and \DP simultaneously \cite{10.1145/3314183.3323847, Fioretto_2022, matzken2023tradeoffs}. The work in \cite{bagdasaryan2019differential} showed that accuracy of DP models drops much more for the underrepresented classes and subgroups, which yields to fairness issues. Interestingly, our Robust-HDP takes care of clients with minority data (i.e., those with small $N_i$) by assigning slightly larger weights to them at aggregation time, as shown in \Cref{eq:wi_Ni} and \Cref{fig:HvsN}. Similarly, WeiAvg assigns uniform weights to clients. Hence, when both batch size and privacy parameters are uniform across clients (i.e., homogeneous DPFL), we expect the weight assignments of Robust-HDP and WeiAvg to yield to a higher performance fairness across clients, while we expect DPFedAvg - which was designed for homogeneous DPFL - to improve system utility. Our experimental results further clarifies this. We use the mean values of the distributions Dist1, Dist3, Dist5, Dist7 and Dist9 from \Cref{table:mixture_dists} for $\epsilon$, i.e., $\epsilon \in \{\tt 2.6, \tt 2.0, \tt 1.1, \tt 0.6, \tt 0.35\}$. Also, as before, we fix $\delta_i$ to $\tt 1e-4$. With these settings and using the Dirichlet data allocation mentioned before, we got the results in \Cref{table:cifar10_alluniform} and \Cref{table:cifar10_alluniform_fairness}.

\begin{table}[t]
\centering
\caption{Comparison of different algorithms on CIFAR10 with uniform batch sizes $\{b_i=32\}$, uniform privacy parameters $\{\epsilon_i = \epsilon\}$ and heterogeneous dataset sizes $\{N_i\}$. FedAvg ($E=200$) achieves $73.55\%$. "minimum $\epsilon$" algorithm is equivalent to DPFedAvg in this case.}
\begin{tabular}{l|*{4}{c}c}\toprule
\diagbox{alg}{distr}
&\makebox[2.5em]{\footnotesize $\epsilon = 2.6$}
&\makebox[2.5em]{\footnotesize $\epsilon = 2.0$}
&\makebox[2.5em]{\footnotesize $\epsilon = 1.1$}
&\makebox[2.5em]{\footnotesize $\epsilon = 0.6$}
&\makebox[2.5em]{\footnotesize $\epsilon = 0.35$}
\\
\midrule \midrule
\algname{WeiAvg} and \algname{PFA}\cite{Liu2021ProjectedFA} & 37.24 & \bf 34.90 & 27.80 & 23.22 & 19.01\\\midrule
\algname{DPFedAvg \cite{DPSCAFFOLD2022}} & \bf 38.52 & 32.78 & \bf 30.42 & \bf 23.50& 20.26 \\\midrule
\algname{Robust-HDP} & 37.68 & 32.54 & 27.51 & 22.58 & \bf 20.52\\\midrule

\end{tabular}
\label{table:cifar10_alluniform}
\end{table}

\begin{table}[t]
\centering
\caption{Comparison of different algorithms on CIFAR10 with uniform batch sizes $\{b_i=32\}$, uniform privacy parameters $\{\epsilon_i = \epsilon\}$ and heterogeneous dataset sizes $\{N_i\}$ in terms of the std of clients' test accuracies and test accuracy of the client with the smallest train set (in parentheses). "minimum $\epsilon$" algorithm is equivalent to DPFedAvg in this case.}
\begin{tabular}{l|*{4}{c}c}\toprule
\diagbox{alg}{distr}
&\makebox[2.5em]{\footnotesize $\epsilon = 2.6$}
&\makebox[2.5em]{\footnotesize $\epsilon = 2.0$}
&\makebox[2.5em]{\footnotesize $\epsilon = 1.1$}
&\makebox[2.5em]{\footnotesize $\epsilon = 0.6$}
&\makebox[2.5em]{\footnotesize $\epsilon = 0.35$}
\\
\midrule \midrule
\algname{WeiAvg} and \algname{PFA}\cite{Liu2021ProjectedFA} &  \textbf{4.24} (32.95) & \textbf{4.11} (30.68) & 4.95 (25.01) & \textbf{3.89} (\textbf{27.84})& 5.70 (17.04)\\\midrule
\algname{DPFedAvg \cite{DPSCAFFOLD2022}} & 4.92 (\textbf{34.69}) & 4.71 (29.54) & 4.95 (25.41) & 6.78 (14.77) & 6.86 (11.36)\\\midrule
\algname{Robust-HDP} & 4.77 (34.09) & 4.59 (\textbf{34.91}) & \textbf{4.66} (\textbf{26.13}) & 6.01 (15.34) & \textbf{3.89} (\textbf{18.97})\\\midrule

\end{tabular}
\label{table:cifar10_alluniform_fairness}
\end{table}

\subsection{Conclusion: when to use Robust-HDP?}

We now summarize our understandings from the theories and experimental results in previous sections to conclude when to use Robust-HDP in \DPFL settings. From our experimental results, heterogeneity in either of the privacy parameters $\{\epsilon_i\}_{i=1}^n$ and batch sizes $\{b_i\}_{i=1}^n$, results in a considerable heterogeneity in noise variances $\{\sigma_i^2\}_{i=1}^n$. Hence, using \algname{Robust-HDP} in this cases will be beneficial in the system overall utility. However, if both $\{\epsilon_i\}_{i=1}^n$ and $\{b_i\}_{i=1}^n$ are homogeneous, and the only potential heterogeneity is in $\{N_i\}_{i=1}^n$ (i.e., homogeneous \DPFL), then using \algname{DPFedAvg} will be slightly better in terms of the system overall utility, as it assigns larger weights to the clients with larger dataset sizes. Despite this, using \algname{Robust-HDP} will alightly improve the performance of clients with smaller dataset sizes.

\subsection{Gradient accumulation}\label{sec:grad_acc}
When training large models with \algname{DPSGD}, increasing the batch size results in memory exploding during training or finetuning. This might happen even when we are not using \DP training. On the other hand, using a small batch size results in larger stochastic noise in batch gradients. Also, in the case of \DP training, using a small batch size results in fast increment of \DP noise (as explained in \ref{sec:noisy_updates} in details). Therefore, if the memory budget of devices allow, we prefer to avoid using small batch sizes. But what if there is a limited memory budget? A solution for virtually increasing batch size is ``gradient accumulation", which is very useful when the available physical memory of GPU is insufficient to accommodate the desired batch size. In gradient accumulation, gradients are computed for smaller batch sizes and summed over multiple batches, instead of updating model parameters after computing each batch gradient. When the accumulated gradients reach the target logical batch size, the model weights are updated with the accumulated batch gradients. The page in \url{https://opacus.ai/api/batch_memory_manager.html} shows the implementation of gradient accumulation for \DP training.

\section{Limitations and Future works}\label{app:future_dirs}

In this section, we investigate the potential limitations of our proposed \algname{Robust-HDP}, and look at the future directions for addressing them. As before, we assume full participation of clients for simplicity. Specifically, we are curious about what happens if the data distribution across clients is not completely \emph{i.i.d}, but rather is moderately/highly heterogeneous. We investigate \algname{Robust-HDP} in these two scenarios in \cref{app:moderatenoniid} and \cref{app:heavilynoniid}, respectively.

\subsection{\algname{Robust-HDP} with moderately heterogeneous data distribution}\label{app:moderatenoniid}

In order to evaluate \algname{Robust-HDP} when the data split is moderately heterogeneous, we run experiments on MNIST. In order to simulate a controlled higher data heterogeneity, we use the sharding data splitting method described in \Cref{appendix:datasets} and \Cref{tab:datasets}, \emph{and we let each client to hold data samples of at maximum 8 classes}, with 60 clients in total. We consider two cases:

\paragraph{All 60 clients use the same batch size \texttt{128}:} 
the results obtained for this case, i.e., heterogeneous data with uniform batch sizes 128, were reported in \cref{table:mnist_8labels_60clients_uniformbatch}. As we observed, \algname{Robust-HDP} still outperforms the baselines in most of the cases. However, compared to the results in \cref{table:mnist}, which were obtained when the data split was \emph{i.i.d}, its superiority has decreased. In order to get an understanding why this is the case, lets have a look at the aggregation weight assignments by different algorithms for this setting in \cref{fig:weights_comparison_details_60_noniid_uniformbatch128}. Remember that, we have assumed uniform batch size of 128 for all clients. Therefore, the only parameter that makes variation in $\{\sigma_i^2\}$ is the clients' privacy parameters $\{\epsilon_i\}$ being different. There are multiple points in \cref{fig:weights_comparison_details_60_noniid_uniformbatch128}. First, the accuracy of RPCA decomposition in estimating $\{\sigma_i^2\}$ has decreased (compare the difference between $\{\sigma_i^2\}$ and their estimates in \cref{fig:weights_comparison_details_60_noniid_uniformbatch128} with that in \cref{fig:weights_comparison_details} which was on a \emph{i.i.d} data split). Second, despite this, the aggregation weights returned by \algname{Robust-HDP} are very close to the optimum weights. This is the case because, as explained in \cref{sec:reliability_rpdp}, estimating the noise variances $\{\sigma_i^2\}$ up to a multiplicative factor suffices for \algname{Robust-HDP} to get to the optimum aggregation weights $\{w_i^*\}$. Lastly, compared to the aggregation weights returned by \algname{WeiAvg}, \algname{Robust-HDP} has smoothly assigned larger weights to the clients with larger privacy parameters $\{\epsilon_i\}$. Note that, as we have assumed uniform batch size for all clients, having a larger privacy parameter $\epsilon$ is equivalent to having a less noisy model update sent to the server. 

\emph{From the points mentioned above and the results in \cref{fig:weights_comparison_details_60_noniid_uniformbatch128}, we conclude that, despite the moderate data heterogeneity, \algname{Robust-HDP} is still successful in assigning the aggregation wights $\{w_i\}_{i=1}^n$ such that the noise the aggregated model update is minimized}. But considering the heterogeneity in clients' data, is this good for the accuracy of the model too? More specifically, with data heterogeneity, does assigning larger weights to the clients with less noisy model updates necessarily result in higher utility too? From the results in \cref{table:mnist_8labels_60clients_uniformbatch}, we observe that when the data is slightly heterogeneous and batch sizes are uniform, this is the case most of the times. However, as we will show next, this seems to be not the case when we also consider an additional heterogeneity in clients' batch sizes.

 \paragraph{The batch sizes of the 60 clients are randomly selected from $\{\texttt{16}, \texttt{32}, \texttt{64}, \texttt{128}\}$:}
 we recall \cref{eq:sigma_i^2}, which showed the considerable effect of batch size of a client on the noise variance in its model updates. When batch size decreases, its noise variance increases fast. Hence, unlike the previous case with uniform batch sizes, it is now both the batch sizes and privacy parameters of clients that determine the noise variance in their model updates.  The results in \cref{table:mnist_8labels_60clients_nonuniformbatch} are obtained in this case. Also, \cref{fig:weights_comparison_details_60_noniid_nonuniformbatch} compares the weight assignments by different algorithms.

As observed, \algname{Robust-HDP} no longer outperforms the baselines. To get a better understanding, lets have a look at the aggregation weight assignments by different algorithms for this setting in \cref{fig:weights_comparison_details_60_noniid_nonuniformbatch}.  
First, the accuracy of RPCA decomposition in estimating $\{\sigma_i^2\}$ has again decreased compared to that in  \cref{fig:weights_comparison_details}, which was on a \emph{i.i.d} data split. Second, despite this, the aggregation weights returned by \algname{Robust-HDP} are still close to the optimum weights. However, the plot of assigned weights by \algname{Robust-HDP} are more spiky than that in \Cref{fig:weights_comparison_details_60_noniid_uniformbatch128}: compared to the aggregation weights returned by \algname{WeiAvg}, \algname{Robust-HDP} has assigned larger weights to the clients with larger privacy parameters $\{\epsilon_i\}$ \emph{and larger batch sizes}. Batch size of clients has a larger effect on the aggregation weights assigned to them. For instance client 59, which has batch size 128 and the second largest privacy parameter, has been assigned aggregation weight close to 0.18, while the same client got aggregation weight close to 0.05 when all clients used the same batch size (\cref{fig:weights_comparison_details_60_noniid_uniformbatch128}). There are 6 clients whose aggregation weights sum to more than 0.5, i.e., these only 6 clients contribute to the aggregated model parameter more than the other 54 clients altogether. The reason behind this is that  \algname{Robust-HDP} aims at minimizing the noise level in the aggregated model update at the end of each round, and it has been successful in that. But the question is that, in this scenario with data heterogeneity, is this strategy beneficial for the utility of the final trained model too? Although, this strategy results in maximizing the trained model utility when the data split is \emph{i.i.d}, it is not the case when we have data heterogeneity and batch size heterogeneity simultaneously, and the results in \cref{table:mnist_8labels_60clients_nonuniformbatch} confirm this. This is a limitation for \algname{Robust-HDP}. However, we can provide a solution for it. Heterogeneity in batch sizes usually happens when clients have different memory budgets. Clients with low memory budgets can not use large batches, especially when training privately with \algname{DPSGD} \cite{Abadi2016}. As observed, when having data heterogeneity, this batch size heterogeneity deteriorates the performance of \algname{Robust-HDP}. Despite the heterogeneity that may exist in clients' memory budgets, they can use gradient accumulation explained in \cref{sec:grad_acc} to virtually increase their batch sizes to a uniform batch size (e.g., 128). In this case, we get back to the results in \cref{table:mnist_8labels_60clients_uniformbatch}, in which \algname{Robust-HDP} works well most of the times. The cost that we pay is that clients with limited physical memory sizes have to spend more time locally during each global round, and the server should wait longer for these clients before performing each aggregation.

\begin{figure*}[ht]
\centering
    \includegraphics[width=0.45\columnwidth]{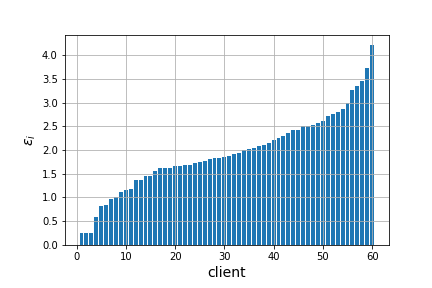}
    \includegraphics[width=0.45\columnwidth]{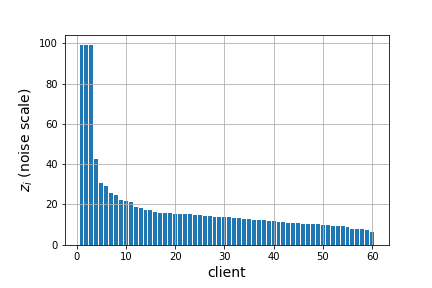}

    \includegraphics[width=0.45\columnwidth]{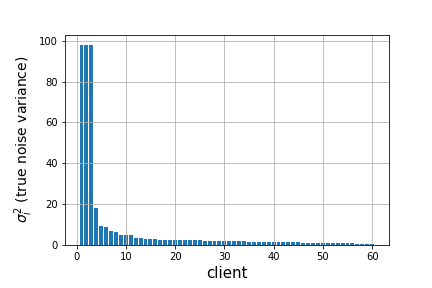}
    \includegraphics[width=0.45\columnwidth]{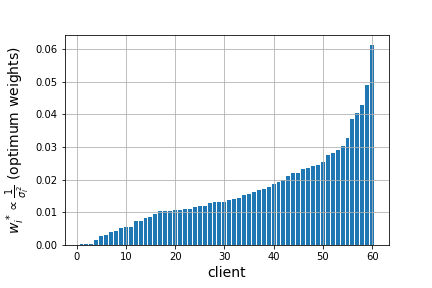}
    \includegraphics[width=0.45\columnwidth]{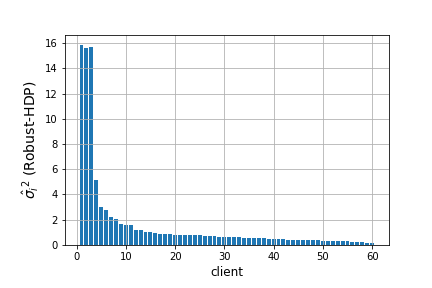}
    \includegraphics[width=0.45\columnwidth]{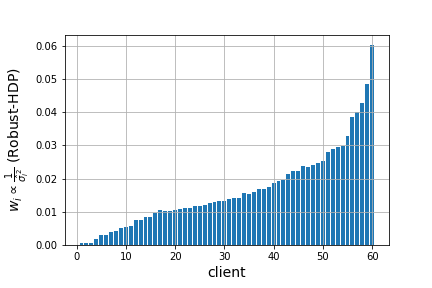}
    \includegraphics[width=0.45\columnwidth]{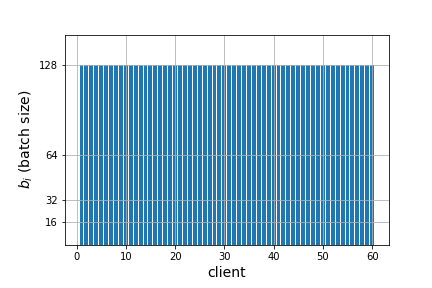}
    \includegraphics[width=0.45\columnwidth]{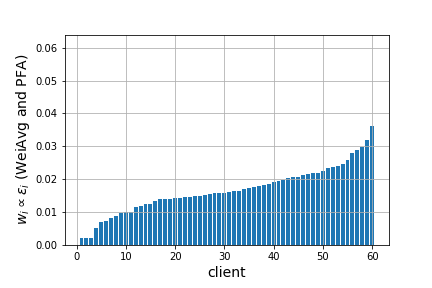}

    \caption{Comparison of weight assignments by different algorithms on MNIST, Dist1 and 60 clients with heterogeneous data distribution (maximum 8 labels per client) and uniform batch size ($b_i=128$). The weight assignments by \algname{Robust-HDP} are very close to the optimal weight assignment strategy, despite the heterogeneity in the data split. Also, the first 40 (last 20) clients, which have the smallest (largest) $\epsilon$ privacy parameters, get assigned smaller (larger) weights by \algname{Robust-HDP} than by \algname{PFA} and \algname{WeiAvg}, \textbf{showing the suboptimality of aggregation strategy of \algname{PFA} and \algname{WeiAvg}}. The results for this experiment were reported in \cref{table:mnist_8labels_60clients_uniformbatch} (\algname{Robust-HDP}, Dist1).}
    \label{fig:weights_comparison_details_60_noniid_uniformbatch128}
\end{figure*}

\begin{table}[t]
\centering
\caption{Comparison of different algorithms (on MNIST, $E=200$) with heterogeneous data split (maximum 8 labels per client) and 60 clients in the system with heterogeneous batch sizes.}
\label{table:mnist_8labels_60clients_nonuniformbatch}
\begin{tabular}{l|*{8}{c}c}\toprule
\diagbox{alg}{distr}
&\makebox[2.5em]{\footnotesize Dist1}
&\makebox[2.5em]{\footnotesize Dist2}
&\makebox[2.5em]{\footnotesize Dist3}
&\makebox[2.5em]{\footnotesize Dist4}
&\makebox[2.5em]{\footnotesize Dist5}
&\makebox[2.5em]{\footnotesize Dist6}
&\makebox[2.5em]{\footnotesize Dist7}
&\makebox[2.5em]{\footnotesize Dist8}
&\makebox[2.5em]{\footnotesize Dist9}\\
\midrule\midrule
\algname{WeiAvg} \cite{Liu2021ProjectedFA} & \bf 84.60 & 78.00 & \bf 83.70 &  \bf 78.42 & 77.3 & \bf 75.23 & 78.01 & 66.10. & 68.12\\\midrule
\footnotesize \algname{PFA}\cite{Liu2021ProjectedFA} & 79.07 & 75.72 & 83.23 & 76.54 & 78.06 & 64.00 & 79.10 & \bf 68.08 & 71.74\\\midrule
\algname{DPFedAvg} \cite{DPSCAFFOLD2022} & 83.41 & 74.95 & 82.69 & 70.85 & 77.45& 74.32 & 74.92& 62.07& 67.25\\\midrule
\algname{Robust-HDP} & 83.65 & \bf 79.38 & 82.88 & 77.13& \bf 83.53& 71.14& \bf 80.41 & 61.75 & \bf 71.76\\
\bottomrule
\end{tabular}
\end{table}

\begin{figure*}[ht]
\centering
    \includegraphics[width=0.45\columnwidth]{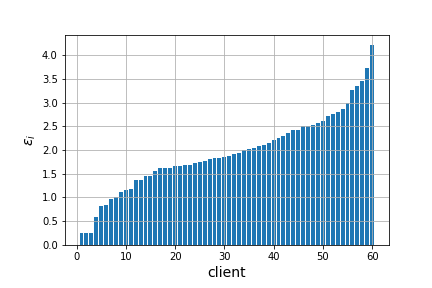}
    \includegraphics[width=0.45\columnwidth]{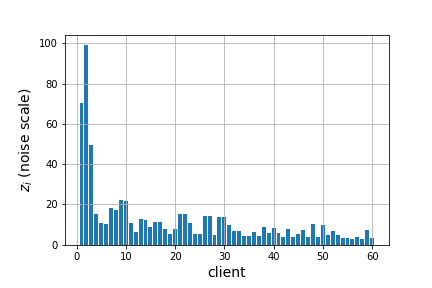}

    \includegraphics[width=0.45\columnwidth]{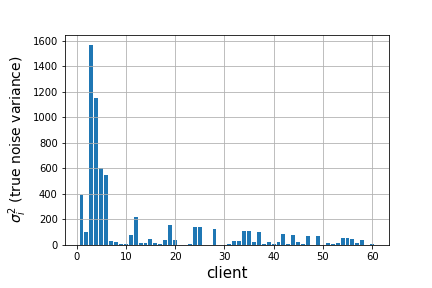}
    \includegraphics[width=0.45\columnwidth]{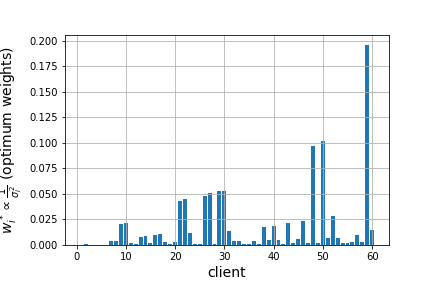}
    \includegraphics[width=0.45\columnwidth]{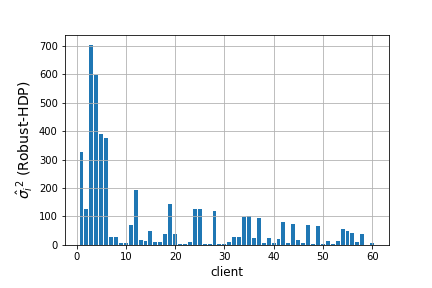}
    \includegraphics[width=0.45\columnwidth]{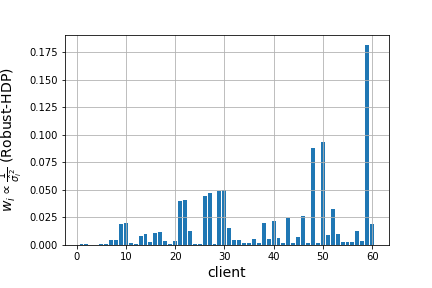}
    \includegraphics[width=0.45\columnwidth]{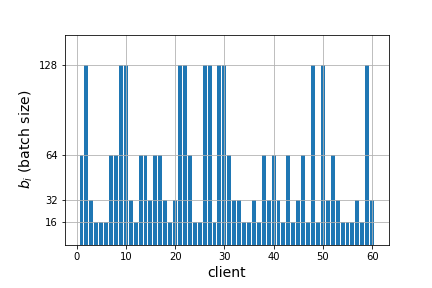}
    \includegraphics[width=0.45\columnwidth]{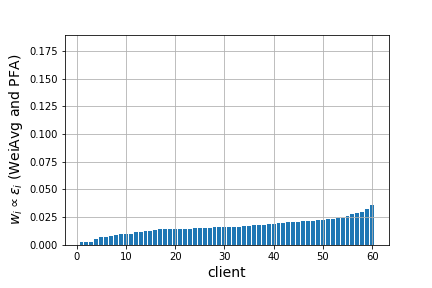}

   \caption{Comparison of weight assignments by different algorithms on MNIST, Dist1 and 60 clients with hetrogeneous data distribution (maximum 8 labels per client) and heterogeneous batch sizes. The weight assignments by \algname{Robust-HDP} are close to the optimal weight assignment strategy, despite the heterogeneity in the data split.  The results for this experiment were reported in \cref{table:mnist_8labels_60clients_nonuniformbatch} (\algname{Robust-HDP}, Dist1).}
    \label{fig:weights_comparison_details_60_noniid_nonuniformbatch}
\end{figure*}

\clearpage
\newpage
\subsection{\DPFL with highly heterogeneous data split across clients (future work)}\label{app:heavilynoniid}

Having studied \algname{Robust-HDP} in scenarios with \emph{i.i.d} and slightly heterogeneous data splits, we are curious about the scenarios with highly heterogeneous data splits. In non-private \FL systems, high data heterogeneity is usually addressed by personalized \FL \cite{Li2020DittoFA} and clustered \FL \cite{Sattler2019ClusteredFL}. In the former case, each client learns a model specifically for itself by fine-tuning the common model obtained from \FL on its local data. In the latter case, clients with similar data are first grouped into a cluster by the server, followed by federated training of a model for each cluster. In highly heterogeneous data distributions, clustered \FL is more common \cite{Sattler2019ClusteredFL, Werner2023ProvablyPA}.

On the other hand, we have \DPFL systems with highly heterogeneous data splits. In the existence of a trusted server, an idea was proposed by \citet{Chathoth2022cohortDP} for clustering clients with cohort-level privacy with privacy and data heterogeneity across cohorts, using $\epsilon$-\DP definition (\cref{def:epsilondeltadp} with $\delta=0$). When there is no trusted server, we can follow a similar direction of clustered \DPFL to address scenarios with highly heterogeneous data splits: clients are first clustered by the server such that the data distribution of clients in a cluster are more similar to each other, and then, a model is learned for each cluster. However, the \DP noise in clients' model updates makes clustering of clients harder. A recent work in \cite{malekmohammadi2024mitigating} has addressed this scenario by proposing an algorithm, which is robust to the \DP noise existing in clients' model updates, for clustering clients in \DPFL system.
\end{document}